\documentclass{article}
\usepackage{iclr2027_conference,times}

\usepackage{amsmath,amsfonts,bm}

\def\eqref#1{equation~\ref{#1}}
\def\Eqref#1{Equation~\ref{#1}}
\def\1{\bm{1}}

\DeclareMathAlphabet{\mathsfit}{\encodingdefault}{\sfdefault}{m}{sl}
\SetMathAlphabet{\mathsfit}{bold}{\encodingdefault}{\sfdefault}{bx}{n}

\newcommand{\E}{\mathbb{E}}

\newcommand{\R}{\mathbb{R}}

\usepackage{amsmath,amssymb,amsthm,mathtools,bm}
\usepackage{booktabs}
\usepackage{tabularx}
\usepackage{graphicx}
\usepackage{wrapfig}
\usepackage{flafter}
\usepackage{float}
\usepackage{caption}
\usepackage{xcolor}
\usepackage{mdframed}
\usepackage{hyperref}
\usepackage{url}

\definecolor{theoryblue}{HTML}{315E7D}
\definecolor{theoryorange}{HTML}{C56B35}
\definecolor{theoryink}{HTML}{26323B}
\definecolor{theorygray}{HTML}{F2F3F5}
\definecolor{theoryline}{HTML}{D9DDE2}

\hypersetup{
  pdftitle={The Composition Gap in Dataset Distillation},
  colorlinks=true,
  linkcolor={red!80!black},
  citecolor={blue!70!black},
  urlcolor={blue!70!black}
}

\newtheoremstyle{plainupright}{}{}{\upshape}{}{\bfseries}{.}{ }{}
\theoremstyle{plainupright}
\newtheorem{theorem}{Theorem}[section]
\newtheorem{lemma}[theorem]{Lemma}
\newtheorem{proposition}[theorem]{Proposition}
\newtheorem{corollary}[theorem]{Corollary}
\theoremstyle{definition}
\newtheorem{definition}[theorem]{Definition}
\newtheorem{example}[theorem]{Example}
\theoremstyle{remark}

\theoremstyle{definition}

\definecolor{boxgray}{HTML}{F2F2F2}
\mdfdefinestyle{resultbox}{
  linewidth=0pt, topline=false, bottomline=false, rightline=false, leftline=false,
  leftmargin=0pt, rightmargin=0pt,
  skipabove=2.5pt, skipbelow=2.5pt,
  innertopmargin=2pt, innerbottommargin=2pt,
  innerleftmargin=5pt, innerrightmargin=5pt,
  backgroundcolor=boxgray
}
\mdfdefinestyle{resultboxnb}{style=resultbox, nobreak=true}
\surroundwithmdframed[style=resultboxnb]{theorem}
\surroundwithmdframed[style=resultbox]{corollary}
\surroundwithmdframed[style=resultbox]{proposition}
\surroundwithmdframed[style=resultbox]{lemma}
\surroundwithmdframed[style=resultbox]{definition}
\surroundwithmdframed[style=resultbox]{example}
\newmdenv[style=resultbox]{interpretationbox}
\newmdenv[style=resultboxnb]{assumptionbox}

\usepackage{titletoc}
\titlecontents{section}[1.4em]{\small\bfseries\vspace{2pt}}{\contentslabel{1.4em}}{\hspace*{-1.4em}}{\titlerule*[0.6pc]{.}\contentspage}
\titlecontents{subsection}[3.2em]{\footnotesize}{\contentslabel{1.8em}}{\hspace*{-1.8em}}{\titlerule*[0.6pc]{.}\contentspage}
\usepackage{siunitx}
\usepackage{placeins}
\usepackage{needspace}
\usepackage{tikz}
\usetikzlibrary{arrows.meta,positioning}
\newcommand{\TrainMap}{\Phi}

\newcommand{\Distill}{\mathcal{A}}

\newcolumntype{L}[1]{>{\raggedright\arraybackslash}p{#1}}
\newcolumntype{Y}{>{\raggedright\arraybackslash}X}

\title{The Composition Gap in Dataset Distillation}

\author{
\begin{minipage}{\textwidth}
Guang Li,
Takahiro Ogawa
\& Miki Haseyama\\
{\normalfont\small
Hokkaido University \quad
\texttt{\{guang,ogawa,mhaseyama\}@lmd.ist.hokudai.ac.jp}}
\end{minipage}
}

\iclrfinalcopy

\begin{document}

\maketitle

\begin{abstract}
Dataset distillation compresses a training set into a small synthetic set, usually evaluated one at a time. In federated and data-governance settings, several parties distill their own data and a user trains on their union. We ask whether the union of separately distilled sets reproduces training on the union of the real data (\emph{composability}) and show that it can fail even when every source is distilled exactly and the total budget admits an exact joint distillate. Compressing a training trajectory into fewer steps transforms the source statistics nonlinearly, so averaging compressed sources differs from compressing their average. For quadratic objectives we derive the exact composition error for two-to-one step compression in terms of the source-Hessian variance and the linear terms of the losses, and on a smooth network at small step sizes this prediction captures the local endpoint discrepancy in magnitude and direction. For learned synthetic sets, however, the composed error decomposes exactly into this local discrepancy and an aggregate source residual. Under endpoint matching the residual exceeds the structural term by more than an order of magnitude, and under distribution matching the two terms partly cancel. Joint distillation also retains an accuracy advantage when both sets are distilled from the same dataset, where the local discrepancy is exactly zero. Training fidelity and downstream accuracy are therefore distinct requirements, neither established by evaluating each set on its own.
\end{abstract}

\section{Introduction}

Dataset distillation (DD), i.e., the compression of a large training set into a small synthetic set that supports effective model training \citep{wang2018dataset,zhao2021dataset,cazenavette2022dataset}, is usually evaluated one distillate at a time. In a distributed application, several data holders instead release independently distilled sets, and a downstream user combines them for training, as in a consortium of hospitals that share only distilled versions of their images \citep{li2022compressed,holland2024collaborative}. The union of the synthetic sets is meant to stand in for the union of the real data. This raises a natural question: \emph{Can independently distilled datasets be combined without loss of training fidelity?} Evaluating each set separately does not answer it (Figure \ref{fig:overview}).

\begin{figure}[t]
\centering
\definecolor{pblue}{HTML}{E3E9F3}\definecolor{pblued}{HTML}{4C72B0}
\definecolor{pcoral}{HTML}{F8E6DC}\definecolor{pcorald}{HTML}{DD8452}
\definecolor{pgray}{HTML}{EAEAF2}\definecolor{pgrayd}{HTML}{6B6B6B}
\begin{tikzpicture}[
  font=\small,
  src/.style={draw=pgrayd, fill=pgray, line width=0.5pt, rounded corners=1pt, inner sep=5pt, minimum height=20pt, minimum width=105pt, align=center},
  ind/.style={draw=pcorald, fill=pcoral, line width=0.5pt, rounded corners=1pt, inner sep=5pt, minimum height=20pt, minimum width=105pt, align=center},
  jnt/.style={draw=pblued, fill=pblue, line width=0.5pt, rounded corners=1pt, inner sep=5pt, minimum height=20pt, minimum width=105pt, align=center},
  arr/.style={-{Latex[length=2mm]}, line width=0.7pt},
  lab/.style={font=\scriptsize\itshape, inner sep=2pt, fill=white}]
\node[src] at (0,0) (S) {sources $D_1,\dots,D_m$\\ \scriptsize Hessians $H_1,\dots,H_m$};
\node[ind] at (10.25,0) (C) {compressed sources\\ \scriptsize $f_p(H_1),\dots,f_p(H_m)$};
\node[jnt] at (0,-1.9) (U) {real union\\ \scriptsize $\bar H=\sum_i\alpha_iH_i$};
\node[jnt] at (5.05,-1.9) (J) {compress the union\\ \scriptsize $f_p(\bar H)\ \rightarrow\ \Phi_{\rm joint}$};
\node[ind] at (10.25,-1.9) (I) {mix the compressed\\ \scriptsize $\sum_i\alpha_i f_p(H_i)\ \rightarrow\ \Phi_{\rm ind}$};
\draw[arr, pcorald] (S) -- node[lab, text=pcorald]{compress each source} (C);
\draw[arr, pcorald] (C) -- node[lab, text=pcorald]{mix} (I);
\draw[arr, pblued] (S) -- node[lab, text=pblued]{mix} (U);
\draw[arr, pblued] (U) -- node[font=\scriptsize\itshape, above=1pt, text=pblued, inner sep=1pt]{compress} (J);
\draw[{Latex[length=1.8mm]}-{Latex[length=1.8mm]}, line width=0.8pt, black!75] (J) -- node[font=\scriptsize\itshape, above=14pt, text=black!75, inner sep=1pt]{composition discrepancy} (I);
\end{tikzpicture}
\caption{Compression and mixing do not commute. Distilling each source and then mixing the distillates (orange) applies the finite-step compression map \(f_p\) to each source Hessian before averaging. Distilling the union (blue) averages first and compresses once. The two paths can yield different training maps, and their difference is the composition discrepancy. The diagram is a schematic of exact quadratic compression at a general ratio \(p\) for sources with a shared optimum; the local endpoint discrepancy \(\Delta\Phi\) measured in Section \ref{sec:experiments} is its two-step instance on real data.}
\label{fig:overview}
\end{figure}
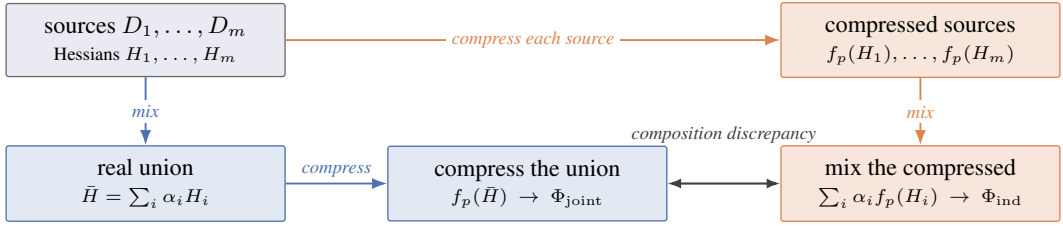

We answer this question through the training maps that gradient descent induces. A source distillate is faithful when a prescribed number of synthetic updates reproduces the endpoint of training on its real source, and composability asks whether the same fidelity holds after the sources and their distillates are combined with matching weights. Our central result is that exact source-wise fidelity does not guarantee composability, even when a joint solution exists within the same total budget.

The discrepancy arises from the order of compression and aggregation. Specifically, for quadratic regression the sufficient statistics defining the loss combine linearly across sources \citep{izzo2023theoretical,gupta2026algorithmic}. Matching two gradient steps with one synthetic step instead preserves a nonlinear transformation of the source curvature. Distilling sources separately applies this transformation before averaging, and distilling their union applies it after averaging (Figure \ref{fig:overview}), so the two orders can disagree even when the sources share a minimizer.

For two-to-one compression of quadratic objectives, we derive an exact endpoint discrepancy that scales with the squared step size (Theorem \ref{thm:structural}). It depends on the source-Hessian variance, which captures differences in both eigenvalues and eigenvectors, and on an affine term coupling the Hessians with the linear loss coefficients.

With distribution matching on CIFAR-10, SVHN and CIFAR-100 \citep{zhao2023distribution}, the excess endpoint error of the union is near zero under IID partitioning and positive under label skew, whereas the accuracy gain of joint distillation persists at every level of skew and even when both distillates are learned from the same dataset, so the curvature mechanism characterizes the training map and not the accuracy gain. Our contributions are summarized as follows:
\begin{itemize}
\item We show that exact source-wise fidelity does not guarantee composability, even when the total budget admits an exact joint distillate. For quadratic objectives the composition error has an exact form in the source-Hessian variance and an affine term that allows noncommuting source Hessians (Theorem \ref{thm:structural}), with a zero-error condition, an extension to other compression ratios, a certificate under approximate distillation (Corollary \ref{cor:detectability}) and a local statistic \(P_{\rm comp}\) computable on a network from source gradients and Hessian-vector products.
\item We decompose the observed error of a union of distillates exactly into the local endpoint discrepancy of the real sources and the aggregate source residual, and measure both on the network. The quadratic prediction is accurate at small step sizes, but for learned synthetic sets the source residual is the larger term: under endpoint matching it exceeds the structural term by more than an order of magnitude, and under distribution matching it is anti-aligned with the local discrepancy under label skew and partly cancels it.
\item We show that training fidelity and downstream accuracy are distinct requirements: the accuracy gain of joint distillation persists for two distillates of the same dataset, where the structural term is exactly zero, and reweighting the constituents after the fact narrows the excess endpoint error, removes at most about half of the squared composed error and gives no detectable accuracy change.
\end{itemize}

\section{Related Work}

Additional related work can be found in Appendix \ref{app:related-extended}.

\paragraph{Dataset distillation.}
DD constructs synthetic sets by unrolled training \citep{wang2018dataset,li2022awesome}, gradient matching \citep{zhao2021dataset,zhao2021siamese}, distribution matching \citep{zhao2023distribution}, kernel-ridge solutions \citep{nguyen2021kernel} or trajectory matching \citep{cazenavette2022dataset,li2024iadd}, among many refinements \citep{lei2024survey}. Theory characterizes what a distillate preserves: a solution at convergence \citep{izzo2023theoretical,nguyen2021kernel}, a loss over random regressors \citep{gupta2026algorithmic}, or a spectrum of the feature correlation matrix \citep{bo2026spectral}.

\paragraph{Distilled data across sources.}
Federated DD trains on synthetic data transmitted by clients \citep{zhou2020distilled,goetz2020synthetic,hu2022fedsynth,xiong2023feddm,liu2023meta,wang2024fedaf,holland2024collaborative,arazzi2025secure,shi2024hfldd}. We characterize the post-hoc composition of such sets under finite-step training. The closest analysis is the drift of local updates in federated optimization \citep{li2020fedavg}, whose endpoint-averaging identity our one-step case recovers, and poor composition of separately distilled sets has been observed empirically \citep{holland2024collaborative}. We add that the failure survives exact constituents and an equal-budget joint solution, and we characterize its size for noncommuting source Hessians and general compression ratios.

\section{Setup}
\label{sec:commutation}

\paragraph{Training maps and distillation.}
Let \(\TrainMap_D^{K,\eta}:\Theta\to\Theta\) denote the map induced by \(K\) full-batch gradient-descent steps of size \(\eta\) on the loss of a dataset \(D\). A DD algorithm returns \(S=\Distill(D)\) such that \(\TrainMap_S^{K,\eta}\) approximates the reference \(\TrainMap_D^{T,\eta}\) over initializations \(\theta_0\sim P_0\), typically with \(K<T\).

\begin{definition}[Endpoint error]\label{def:operator-error}
For two training maps \(\TrainMap_1,\TrainMap_2\), define \(\mathcal C_{P_0}(\TrainMap_1,\TrainMap_2):=\big(\E_{P_0}\|\TrainMap_1(\theta_0)-\TrainMap_2(\theta_0)\|_2^2\big)^{1/2}\) and \(\mathcal E_{P_0}(\TrainMap_1,\TrainMap_2):=\mathcal C_{P_0}(\TrainMap_1,\TrainMap_2)^2\).
\end{definition}

For sources \(D_1,D_2\) with distillates \(S_i=\Distill(D_i)\), the constituent error is \(\epsilon_i=\mathcal E_{P_0}(\TrainMap_{S_i}^{K,\eta},\TrainMap_{D_i}^{T,\eta})\). Write \(D_1\oplus_\alpha D_2\) for the weighted union with loss \(\alpha L_{D_1}+(1-\alpha)L_{D_2}\), where \(\alpha\) is the weight in the union loss rather than the share of stored samples.

\begin{definition}[Composition gap]\label{def:gap}
For synthetic weight \(\beta\) and total budget \(M\ge|S_1|+|S_2|\), with the infimum over weighted synthetic sets of at most \(M\) samples (Appendix \ref{app:framework}), define the independent composition error, the oracle joint error and the composition gap by
\begin{align*}
\mathcal E_{\rm ind}(\alpha,\beta)&:=\mathcal E_{P_0}\big(\TrainMap_{S_1\oplus_\beta S_2}^{K,\eta},\TrainMap_{D_1\oplus_\alpha D_2}^{T,\eta}\big),\\
\mathcal E_{\rm joint}^\star(M)&:=\inf_{|S|\le M}\mathcal E_{P_0}\big(\TrainMap_S^{K,\eta},\TrainMap_{D_1\oplus_\alpha D_2}^{T,\eta}\big),\\
\Gamma_{\rm joint}^\star&:=\mathcal E_{\rm ind}(\alpha,\alpha)-\mathcal E_{\rm joint}^\star(M)\ge0,
\end{align*}
and abbreviate \(\mathcal E_{\rm ind}:=\mathcal E_{\rm ind}(\alpha,\alpha)\).
\end{definition}

The gap is nonnegative since the composed set is an admissible joint candidate, and every exact joint candidate used below is realizable within the same budget (Appendix \ref{app:framework}). The empirical excess \(\widehat\Gamma_{\rm RMS}=\mathcal C_{\rm ind}-\mathcal C_{\rm joint}\) of Section \ref{sec:experiments} replaces the infimum by one joint distillation run and the difference of squared errors by the difference of RMS errors, and can therefore be negative.

\paragraph{Assumptions.}
Each theorem states which of the following assumptions it requires.

\begin{assumptionbox}
\textbf{(A1)} \emph{Quadratic, symmetric, positive branch:} full-batch fixed-step gradient descent on \(L_i(\theta)=\tfrac12\theta^\top H_i\theta-b_i^\top\theta+c_i\), with symmetric \(H_i\), \(\|H_i\|_2\le L\) and \(\eta L<1\).\par\noindent
\textbf{(A2)} \emph{Realizability:} \(H_i\succeq0\) and \(b_i\in\operatorname{range}(H_i)\), so that the transformed quadratics are realizable as finite squared-loss sample sets (Appendix \ref{app:framework}).\par\noindent
\textbf{(A3)} \emph{Common optimum, full coverage:} all sources share \(\theta^*\), and the second moment \(M_0=\E[(\theta_0-\theta^*)(\theta_0-\theta^*)^\top]\) is positive definite.\par\noindent
\textbf{(A4)} \emph{Slowly varying Hessian:} for a non-quadratic \(F\), \(\nabla^2F\) is \(\nu\)-Lipschitz on a ball containing the \(K\)-step trajectory.
\end{assumptionbox}

The two-to-one identity requires only symmetry. The positive-branch condition is required for the arbitrary-ratio and lower-bound results, (A2) ensures sample realizability, (A3) is invoked where stated and (A4) supports the local nonlinear analysis. Appendix \ref{app:optimizers} extends the two-to-one identity to heavy-ball momentum, where the compressed curvature becomes \((2+\beta)H-\eta H^2\) and the discrepancy is unchanged, and to the expectation over independent minibatches.

\paragraph{Universal composability requires affinity.}
A compression rule that composes for every source collection and every weighting must preserve convex averages. Under the positive-branch condition, compressing \(T\) real steps into \(K<T\) synthetic steps transforms a quadratic Hessian through \(H\mapsto f_p(H)=(I-(I-\eta H)^p)/\eta\) with \(p=T/K\), a matrix function \citep{higham2008functions} that satisfies \((I-\eta f_p(H))^K=(I-\eta H)^T\) and is not affine for \(p>1\).

\begin{figure}[t]
\centering
\includegraphics[width=\textwidth]{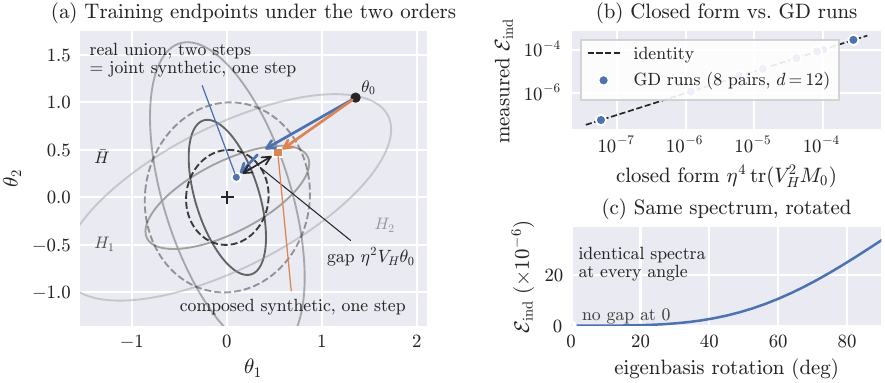}
\caption{Exact constituents compose with an error set by curvature disagreement (Theorem \ref{thm:structural}, quadratic models with a shared optimum at zero). (a) Two sources with different curvature ellipses. The real two-step and joint one-step endpoints coincide (blue), and the composed endpoint is displaced by \(\eta^2V_H\theta_0\) (orange). (b) Endpoint discrepancies of gradient-descent runs on random source pairs against the closed form. (c) Identical spectra with rotated eigenbases still yield a nonzero discrepancy, since \(V_H\) includes eigenvector alignment.}
\label{fig:mechanism}
\end{figure}

\begin{proposition}[Composability--affinity principle]
\label{thm:affinity}
Let \(\mathcal X\) be convex and let \(C:\mathcal X\to\mathcal Y\) be a continuous source-wise compression map into a normed space of synthetic statistics averaged under union. The following are equivalent: (i) every equal-mass pair composes, \(C((x+y)/2)=(C(x)+C(y))/2\); (ii) every finite weighted collection composes; (iii) \(C\) is affine. The same holds for training endpoints when the endpoint map is injective on the relevant convex hull of compressed statistics.
\end{proposition}

This is the midpoint form of Jensen's functional equation \citep{aczel1966lectures}. The proof can be found in Appendix \ref{app:framework}.

\section{The Composition Law}
\label{sec:law}

\subsection{One dimension}

Let \(L_a(\theta)=\tfrac a2(\theta-u)^2\) with \(0<\eta a<1\). After \(T\) steps the centered parameter is multiplied by \((1-\eta a)^T\). For \(p=T/K>1\), training for \(K\) steps with the synthetic curvature \(g_p(a):=(1-(1-\eta a)^p)/\eta\) reproduces this factor exactly, so each source admits a zero-error distillate.

\begin{example}[Exact constituents need not compose]
\label{ex:1d}
Take two sources with distinct curvatures \(a_1\ne a_2\) and common optimum \(u\), distill each exactly, and mix the distillates with the correct mass \(\alpha\in(0,1)\). Since \(g_p''(a)=-p(p-1)\eta(1-\eta a)^{p-2}<0\), strict Jensen's inequality gives
\begin{equation}
\alpha g_p(a_1)+(1-\alpha)g_p(a_2)<g_p\!\left(\alpha a_1+(1-\alpha)a_2\right).
\label{eq:scalar-jensen}
\end{equation}
The composed synthetic curvature is smaller than the exact joint curvature \(g_p(\bar a)\), which a single example realizes (Appendix Figure \ref{fig:clock}). Hence \(\epsilon_1=\epsilon_2=0\) and \(\Gamma_{\rm joint}^\star>0\) whenever \(\E_{P_0}[(\theta_0-u)^2]>0\) (Theorem \ref{thm:1d}, Appendix \ref{app:oned}).
\end{example}

For \(a_1=1\), \(a_2=3\), \(\eta=0.1\) and \(p=2\), the exact synthetic curvatures \(1.9\) and \(5.1\) average to \(3.5\), short of the joint value \(g_2(2)=3.6\). For \(p=2\) the shortfall is exactly \(\eta\alpha(1-\alpha)(a_1-a_2)^2\), the scalar form of the curvature variance introduced below (Figure \ref{fig:mechanism}a shows the two-dimensional case). Both distillates are exact and a joint distillate of the same budget exists, so the failure cannot be attributed to poorly optimized source distillates or to an insufficient synthetic budget.

\subsection{Exact law for two-to-one compression}

The scalar result extends to an exact matrix identity without a shared eigenbasis. Let source \(i\) have objective \(L_i(\theta)=\tfrac12\theta^\top H_i\theta-b_i^\top\theta+c_i\) with symmetric \(H_i\). For \(T=2\) and \(K=1\), the individually exact synthetic statistics are
\begin{equation}
\widetilde H_i=2H_i-\eta H_i^2,\qquad\widetilde b_i=(2I-\eta H_i)b_i.
\label{eq:two-to-one-stats}
\end{equation}
For source masses \(\alpha_i>0\) with \(\sum_i\alpha_i=1\), let \(\bar H:=\sum_i\alpha_iH_i\), \(\bar b:=\sum_i\alpha_ib_i\) and
\begin{equation}
V_H:=\sum_i\alpha_iH_i^2-\bar H^2,\qquad w:=\sum_i\alpha_iH_ib_i-\bar H\bar b.
\label{eq:vh-w}
\end{equation}
\(V_H\) is the variance of the source Hessians under the mass weights and \(w\) the corresponding covariance between the Hessians and the linear coefficients; both vanish when the sources coincide.

\begin{theorem}[Structural composition law]
\label{thm:structural}
Let \(\TrainMap_{\rm ind}\) be the one-step map of the union of the exact distillates, with statistics \((\sum_i\alpha_i\widetilde H_i,\sum_i\alpha_i\widetilde b_i)\), and \(\TrainMap_R\) the two-step map of the real union, with statistics \((\bar H,\bar b)\). For symmetric source Hessians, without any commutativity assumption,
\begin{equation}
V_H=\sum_i\alpha_i(H_i-\bar H)^2=\frac12\sum_{i,j}\alpha_i\alpha_j(H_i-H_j)^2\succeq0.
\label{eq:vh}
\end{equation}
The independently composed one-step map and the real-union two-step map satisfy
\begin{equation}
\TrainMap_{\rm ind}(\theta_0)-\TrainMap_R(\theta_0)=\eta^2(V_H\theta_0-w).
\label{eq:affine-law}
\end{equation}
For an initialization distribution with mean \(\mu_0\) and covariance \(\Sigma_0\), the composition error is
\begin{equation}
\mathcal E_{\rm ind}=\eta^4\Big[\underbrace{\operatorname{tr}(V_H^2\Sigma_0)}_{\text{curvature variance}}+\underbrace{\|V_H\mu_0-w\|_2^2}_{\text{mean affine discrepancy}}\Big].
\label{eq:affine-error}
\end{equation}
\end{theorem}

The full proof can be found in Appendix \ref{app:matrix}. In brief, expanding \eqref{eq:two-to-one-stats} shows that the composed statistics \(\sum_i\alpha_i(\widetilde H_i,\widetilde b_i)\) and the exact joint statistics \((2\bar H-\eta\bar H^2,(2I-\eta\bar H)\bar b)\) differ by exactly \((-\eta V_H,-\eta w)\). Substituting this difference into the one-step update yields \eqref{eq:affine-law}, and the expected squared norm gives \eqref{eq:affine-error}. The joint quadratic is realizable under (A2) and the positive branch (Lemma \ref{lem:realizability}).

Under a common optimum \(\theta^*\), \(w=V_H\theta^*\) and the error reduces to
\begin{equation}
\mathcal E_{\rm ind}=\eta^4\operatorname{tr}(V_H^2M_0),\qquad M_0:=\E[(\theta_0-\theta^*)(\theta_0-\theta^*)^\top].
\label{eq:common-optimum}
\end{equation}
Agreement on the optimum thus does not ensure agreement on the endpoint: the discrepancy \(\eta^2V_H(\theta_0-\theta^*)\) vanishes at the optimum and grows with the curvature differences along the directions the initialization occupies. \(V_H\) is the mass-weighted variance of the Hessians as matrices, so it records differences in eigenvalues and eigenvectors alike (Figure \ref{fig:mechanism}c), and \(M_0\) records the directions the initialization covers. The matrix law is not a restatement of the scalar Jensen gap: the Hessians need not commute, the affine term couples them with the linear coefficients, and the lower bound for other ratios (Theorem \ref{thm:defect}) holds without commutativity. The squared error grows as \(\eta^4\), the log--log learning-rate slope of four that Section \ref{sec:experiments} tests.

\begin{interpretationbox}
\textbf{Zero-error condition.} The composition error vanishes if and only if \(V_H\Sigma_0^{1/2}=0\) and \(V_H\mu_0=w\). Under a common optimum and \(M_0\succ0\), it vanishes if and only if all sources have the same Hessian.
\end{interpretationbox}

With \(M_0\succ0\) no direction of curvature disagreement is avoided by the initialization, which makes the second statement an equivalence. With \(\Sigma_0\succ0\) the first condition forces \(V_H=0\) and hence \(w=0\), so identical source Hessians are necessary and sufficient whatever the optima, and the affine condition is independent only for a rank-deficient initialization (Appendix \ref{app:matrix}).

\subsection{Other compression ratios}

Under a common optimum and the positive-branch condition, for any \(p=T/K>1\) and integer \(K\ge1\), the joint and the composed transformed Hessians differ by
\begin{equation}
J_p=\binom p2\eta V_H+R_p,\qquad \|R_p\|_F\le\eta^2Lc_p(\eta L)\sqrt{r_H}\,\|V_H\|_F,
\label{eq:jp}
\end{equation}
where \(c_p\) is the remainder series of Appendix \ref{app:general} and \(r_H=\operatorname{tr}(V_H)^2/\|V_H\|_F^2\) is the effective rank of \(V_H\ne0\). The \(K\)-step endpoint discrepancy admits a lower bound without commutativity (Theorem \ref{thm:defect}), positive for \(\eta Lc_p(\eta L)\sqrt{r_H}<\binom p2\) and unconditional for \(p=2\), so curvature variance governs the leading operator term for every ratio and step count.

\subsection{Approximate distillates}

Perturbing the composed statistics by \((\Delta H,\Delta b)\) changes the RMS error by at most a computable \(\Delta_{\rm aff}\) (Theorem \ref{thm:stability}, Appendix \ref{app:stability}). Writing \(\mathcal C=\sqrt{\mathcal E}\), \(\mathcal C_{\rm struct}\) for the composition error with exact constituents and \(\widehat{\mathcal C}_{\rm comp}\) for the error with the constituents actually produced, the bound yields a certificate.

\begin{corollary}[Sufficient detectability condition]
\label{cor:detectability}
Assume \(T=2\), \(K=1\), a common optimum \(\theta^*\), and that both the ideal and the actual source statistics fix \(\theta^*\). Let \(\mathcal U_{\rm src}=\sum_i\alpha_i\sqrt{\epsilon_i}\) and \(\mathcal C_{\rm struct}=\eta^2\|V_HM_0^{1/2}\|_F\). Then
\begin{equation}
\left|\widehat{\mathcal C}_{\rm comp}-\mathcal C_{\rm struct}\right|\le\mathcal U_{\rm src},\qquad\mathcal C_{\rm struct}>\mathcal U_{\rm src}\Longrightarrow\widehat{\mathcal C}_{\rm comp}\ge\mathcal C_{\rm struct}-\mathcal U_{\rm src}>0.
\label{eq:detectability-bound}
\end{equation}
\end{corollary}

The corollary compares two quantities that can be estimated separately: \(\mathcal C_{\rm struct}\) from the source statistics before any distillation, and \(\mathcal U_{\rm src}\) from the distillates actually produced. When the first exceeds the second, the composed union misses the real-union endpoint by at least their difference. On a network, where the common-optimum assumption is not established, the inequality is a diagnostic rather than a certificate. Constituent errors alone cannot upper-bound the composition error, since exact constituents with separating optima have unbounded composition error (Proposition \ref{prop:no-individual-control}).

\section{From Quadratic Models to Networks}
\label{sec:practical}

\begin{example}[Networks as local quadratic proxies]\label{ex:network}
For a neural loss \(F\) at a reference point \(\theta^\circ\), the frozen quadratic \(Q\) has symmetric Hessian \(H=\nabla^2F(\theta^\circ)\), which need not be positive semidefinite, so \(Q\) is a proxy and a realizable synthetic set only when (A2) holds.
\end{example}

For the objective \(F\) and the frozen quadratic \(Q\) of Example \ref{ex:network}, if \(\nabla^2F\) is \(\nu\)-Lipschitz on a ball of radius \(r\) centered at \(\theta^\circ\) that contains the \(K\)-step trajectory, and \(\|I-\eta H\|_2\le\rho\), then
\begin{equation}
\|\TrainMap_F^{K,\eta}(\theta_0)-\TrainMap_Q^{K,\eta}(\theta_0)\|_2\le\frac{\eta\nu r^2}{2}\sum_{t=0}^{K-1}\rho^t.
\label{eq:nonlinear-remainder}
\end{equation}
The nonlinear RMS composition error thus differs from its quadratic counterpart by at most \(\delta_S+\delta_R\), the sum of the synthetic and real training remainders (Theorem \ref{thm:nonlinear}, Appendix \ref{app:nonlinear}).

The structural term of the local model, that is, of the frozen quadratic at the initialization, can be evaluated at a shared parameter without a common optimum and without forming Hessians. Writing \(g_i=H_i\theta_0-b_i\) and \(\bar g=\sum_i\alpha_ig_i\), define
\begin{equation}
r_{\rm comp}(\theta_0):=\sum_i\alpha_iH_i(g_i-\bar g)=V_H\theta_0-w,\qquad P_{\rm comp}:=\eta^4\E_{P_0}\|r_{\rm comp}(\theta_0)\|_2^2,
\label{eq:pcomp}
\end{equation}
so that \(P_{\rm comp}=\mathcal E_{\rm ind}\) for ideal two-to-one quadratic compression and \(P_{\rm comp}=\eta^4\operatorname{tr}(V_H^2M_0)\) under a common optimum. Given the source gradients, each evaluation of \(r_{\rm comp}\) costs one Hessian-vector product per source \citep{pearlmutter1994fast} (Appendix \ref{app:e1-protocol}). The observed RMS error satisfies
\begin{equation}
\left|\mathcal C_{\rm obs}-\sqrt{P_{\rm comp}}\right|\le\Delta_{\rm aff}+\delta_S+\delta_R
\label{eq:obs-bound}
\end{equation}
(Corollary \ref{cor:proxy-gap}), where \(\Delta_{\rm aff}\) is the error of the distilled statistics and \(\delta_S,\delta_R\) are the nonlinear remainders of synthetic and real training. Thus \(\sqrt{P_{\rm comp}}\) approximates the observed RMS error when all three are small relative to it, and Section \ref{sec:experiments} measures the source matching errors and their contribution.

\section{Experimental Results}
\label{sec:experiments}

The experiments have three tasks (settings and additional results in Appendix \ref{app:experiments}): verifying the exact law on quadratic sources with learned synthetic sets, testing the local approximation on a network by comparing the local endpoint discrepancy \(\Delta\Phi(\theta_0)\) of the real sources, which involves no synthetic set, with the prediction \(\eta^2r_{\rm comp}(\theta_0)\), and decomposing the composition error of learned synthetic sets, by endpoint matching on the same network and by distribution matching on CIFAR-10, SVHN and CIFAR-100.

\paragraph{Learned quadratic distillates reproduce the law once each source is matched.}
\label{sec:e5}
For twenty random pairs of quadratic sources in \(d=12\) dimensions, \(n_{\rm syn}\) squared-loss samples per source are optimized on the two-step versus one-step endpoint discrepancy over random initializations, without access to the target statistics (Appendix \ref{app:e5}). Once the budget reaches the dimension, the learned Gram matrices recover \(2H-\eta H^2\) to machine precision, and the union of the learned sets then incurs exactly the composition error of Theorem \ref{thm:structural}, with \(V_H=\tfrac14(H_1-H_2)^2\), in every pair (Appendix Table \ref{tab:e5}). Below the dimension the constituents are inexact, and the composed error exceeds the closed form by a factor of \(4.8\) to \(11.7\) that grows as the budget shrinks.

\begin{figure}[t]
\centering
\includegraphics[width=\textwidth]{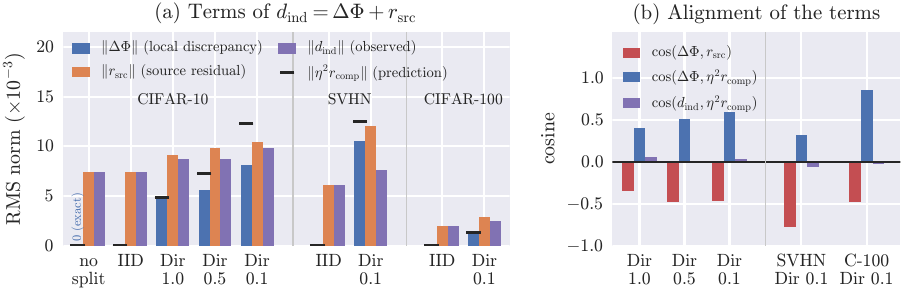}
\caption{The observed error of a union of distillates is the net of the local discrepancy and the source residual. Exact decomposition \(d_{\rm ind}=\Delta\Phi+r_{\rm src}\) on the distribution-matching sets of Section \ref{sec:operator}, with the step calibration described there (CIFAR-10, SVHN and CIFAR-100; level means over partitions, Appendix Table \ref{tab:e9}). (a) RMS norms of the local discrepancy \(\Delta\Phi\), the aggregate source residual \(r_{\rm src}\) and the observed error \(d_{\rm ind}\), with the quadratic prediction \(\eta^2r_{\rm comp}\) marked; the local discrepancy is exactly zero in the no-split control and \(0.013\)--\(0.018\times10^{-3}\) at IID, below the resolution of the axis. The residual is the larger term at every level, and \(\|d_{\rm ind}\|<\|r_{\rm src}\|\) under skew. (b) Cosines at the skewed levels: the two terms are anti-aligned, the prediction is partly aligned with \(\Delta\Phi\) at this step size (\(\eta L\approx1\)) and nearly orthogonal to \(d_{\rm ind}\).}
\label{fig:decomp}
\end{figure}

\begin{figure}[t]
\centering
\includegraphics[width=\textwidth]{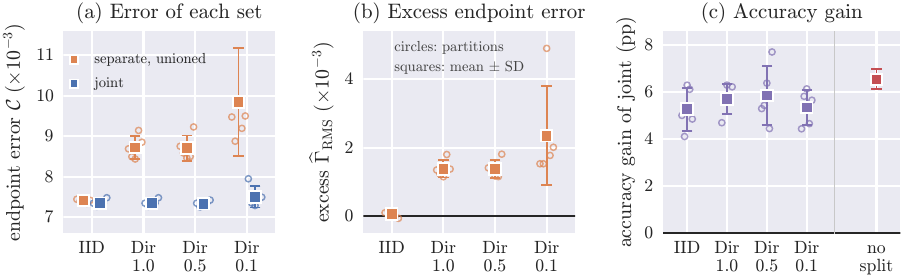}
\caption{The excess endpoint error grows with label skew, whereas the accuracy gain of joint distillation persists in the no-split control. Independent versus joint distillation on CIFAR-10 with distribution matching at ten images per class per source. Each skew level has five partitions (circles), and squares with bars give their mean and SD. Within a partition, endpoint errors (after the per-set step calibration of Section \ref{sec:operator}) are RMS over three initializations and accuracies average three evaluation seeds. (a) RMS endpoint error of the separately distilled union and of the joint set against the real-union two-step map. (b) Their difference, the excess endpoint error. (c) Accuracy gain of joint over separate-then-union distillation. The no-split control (red) has three replicates of the full dataset, and its bar is the SD over replicates.}
\label{fig:two}
\end{figure}

\paragraph{The quadratic prediction captures local network dynamics at small step sizes.}
\label{sec:e1}
Network experiments use two-source CIFAR-10 \citep{krizhevsky2009cifar} partitions, IID or class-wise Dirichlet with concentration \(1.0\), \(0.5\) or \(0.1\) (Dir-1.0 to Dir-0.1), and a three-block ConvNet (Appendix \ref{app:experiments}). The main endpoint measurements use two real steps (\(T=2\); Appendix \ref{app:e10} also examines \(T=4\), \(K=2\)). The local endpoint discrepancy \(\Delta\Phi(\theta_0)\) applies the construction of Theorem \ref{thm:structural} to the network, and its RMS is the local reference \(\mathcal C_{\rm loc}\), equal to \(\mathcal C_{\rm struct}=\sqrt{P_{\rm comp}}\) for quadratics:
\begin{equation}
\mathcal C_{\rm loc}=\bigl(\E_{\theta_0}\|\Delta\Phi(\theta_0)\|^2\bigr)^{1/2},\qquad \Delta\Phi(\theta_0)=\sum_i\alpha_i\Phi_{D_i}^{2,\eta}(\theta_0)-\Phi_{D_1\oplus_\alpha D_2}^{2,\eta}(\theta_0).
\label{eq:cloc}
\end{equation}
For a union of synthetic sets the discrepancy against the same target decomposes exactly as
\begin{equation}
d_{\rm ind}(\theta_0)=\Delta\Phi(\theta_0)+r_{\rm src}(\theta_0),\qquad r_{\rm src}(\theta_0)=\sum_i\alpha_i\bigl(\Phi^{1,\eta}_{S_i}(\theta_0)-\Phi^{2,\eta}_{D_i}(\theta_0)\bigr),
\label{eq:decomp}
\end{equation}
so three quantities appear below: \(\eta^2r_{\rm comp}(\theta_0)\) is the prediction from local gradients and curvature, \(\Delta\Phi(\theta_0)\) is the real-data discrepancy between separate training followed by averaging and training on the union, and \(d_{\rm ind}(\theta_0)\) is the error of the union of the distillates against the real-union target. The weighted source error \(\mathcal U_{\rm src}=\sum_i\alpha_i\|\Phi^{1,\eta}_{S_i}-\Phi^{2,\eta}_{D_i}\|_{L^2(P_0)}\) bounds \(\|r_{\rm src}\|_{L^2(P_0)}\) from above, and the residual falls below the bound when the source errors partly cancel.

We compare \(\Delta\Phi(\theta_0)\) with its prediction \(\eta^2r_{\rm comp}(\theta_0)\) at \(\eta_{\rm ref}=7.8\times10^{-4}\) and above. Within \(\eta L<1\), SiLU and tanh networks give squared-error exponents of \(3.71\pm0.08\) and \(3.86\pm0.06\) at Dir-0.1, against the quadratic value of four, and the calibration ratio is within about three percent of one at \(\eta_{\rm ref}\) (Appendix Figure \ref{fig:e1-full}). On the SiLU network at \(\eta_{\rm ref}\) the prediction also matches the direction of \(\Delta\Phi(\theta_0)\), with cosine \(1.000\) and a relative residual of \(2\) to \(4\) percent, and it degrades as the step size grows, overshooting \(\Delta\Phi\) by a factor of \(1.9\) at \(16\eta_{\rm ref}\) (Appendix \ref{app:e10}, Figure \ref{fig:e10}b).

\paragraph{Source residuals exceed the local discrepancy in the endpoint-matching experiments.}
\label{sec:e10}
Synthetic sets learned on the same network by matching one synthetic step to the two real steps at one, five or ten images per class per source retain a relative source error of \(0.5\) at one image per class and \(0.25\) to \(0.39\) at five and ten (Appendix \ref{app:e10}, Figure \ref{fig:e10}). At Dir-0.1 the aggregate source residual \(\|r_{\rm src}\|\) is \(24\), \(13\) and \(13\) times \(\mathcal C_{\rm struct}\) at one, five and ten images per class, the inequality of Corollary \ref{cor:detectability} is satisfied for none of \(96\) initializations at \(\eta_{\rm ref}\), and the composed discrepancy is nearly orthogonal to the prediction. For these sets the aggregate source residual \(r_{\rm src}\), not the structural term, sets the composed error.

\paragraph{The excess endpoint error of the union grows with label skew.}
\label{sec:operator}
Distribution matching (DM) produces synthetic sets as in practice, here on a ReLU network at \(\eta=6.25\times10^{-3}\) at its official schedule with ten images per represented class per source. The joint set is distilled from the union at the same total budget (Appendix \ref{app:e2}). DM sets are not trained for any step size, so each set is applied with the scalar step multiplier that best matches its own real two-step target at each evaluation initialization, the joint set being calibrated against the union target (Appendix \ref{app:repair}). Without it all errors rise and the excess shrinks but stays positive at every skewed level (Appendix \ref{app:e9}). Below, \(\Phi^{1,\eta}_{S}\) denotes this calibrated map and \(\mathcal C_{\rm ind}\), \(\mathcal C_{\rm joint}\) and \(\widehat\Gamma_{\rm RMS}=\mathcal C_{\rm ind}-\mathcal C_{\rm joint}\) the resulting endpoint errors (Definition \ref{def:operator-error}).

Joint distillation uses the same full training set at every level and its endpoint error stays within \(7.3\)--\(7.5\times10^{-3}\) (Figure \ref{fig:two}a). The error of the separately distilled union matches it under IID and rises to \(9.85\times10^{-3}\) at Dir-0.1. The excess is near zero under IID partitioning and positive and an order of magnitude larger on every label-skewed partition (Figure \ref{fig:two}b). SVHN and CIFAR-100 replicate this pattern, the CIFAR-100 excess rising from \(0.09\) at IID to \(0.56\times10^{-3}\) at Dir-0.1 and positive in nine of nine skewed cells (Appendix \ref{app:svhn}).

Both errors are comparable to \(\mathcal C_{\rm struct}\) and \(\mathcal U_{\rm src}\) is about three times \(\mathcal C_{\rm struct}\), so the inequality of Corollary \ref{cor:detectability} is not satisfied.

\paragraph{The observed discrepancy is the net of the local discrepancy and the source residual.}
By \eqref{eq:decomp} the excess is \(\widehat\Gamma_{\rm RMS}=\|\Delta\Phi+r_{\rm src}\|_{L^2(P_0)}-\|d_{\rm joint}\|_{L^2(P_0)}\) (Appendix \ref{app:e9}), and \(\|d_{\rm ind}\|<\|r_{\rm src}\|\) requires \(\langle\Delta\Phi,r_{\rm src}\rangle<-\tfrac12\|\Delta\Phi\|^2\), a net cancellation between the two terms. The DM sets show this cancellation (Figure \ref{fig:decomp}): the measured \(\|d_{\rm ind}\|\) is smaller than \(\|r_{\rm src}\|\) at every skewed level, and the cosine between \(\Delta\Phi\) and \(r_{\rm src}\) averages \(-0.35\) to \(-0.48\) and is negative for all \(45\) skewed partition--seed pairs. At this step size (\(\eta L\approx1\)) the quadratic model is a proxy rather than a calibrated predictor: \(\eta^2r_{\rm comp}\) matches \(\Delta\Phi\) in magnitude but only partly in direction and is nearly orthogonal to \(d_{\rm ind}\) (Figure \ref{fig:decomp}b). Reweighting the constituents narrows the excess without removing it: class weights fitted on held-out initializations remove \(15\) to \(25\) percent of the union's squared error at the skewed levels and cut the mean excess over the joint set to \(0.1\)--\(0.7\times10^{-3}\), still positive at every skewed level, and accuracy does not change detectably (Appendix \ref{app:reweight}).

\paragraph{The accuracy advantage persists without source heterogeneity.}
\label{sec:accuracy}
On CIFAR-10, joint distillation improves accuracy by \(5.3\) to \(5.9\) points at every heterogeneity level (Figure \ref{fig:two}c, Appendix Table \ref{tab:e2-gap}). With both independent distillates learned from the full training set, the local discrepancy is exactly zero, yet joint distillation still improves accuracy by \(6.6\pm0.4\) points (Appendix \ref{app:nosplit}). Between IID and Dir-0.1 the local reference \(\mathcal C_{\rm loc}\) grows five-hundredfold while the mean gain stays between \(5\) and \(6\) points: heterogeneity is not necessary for the gain. Matching the per-class counts of the joint set to those of the composed set at Dir-0.1 (Appendix \ref{app:balanced}) leaves the excess positive while the accuracy gain on the two affected partitions falls by \(3.5\) and \(1.9\) points, so much of that gain reflects the class allocation. Continuing distillation from the union of two distillates for \(2{,}000\) iterations recovers \(0.4\) points of the no-split gap, against \(1.8\) for a set initialized from real images over the same iterations (Appendix \ref{app:warm}). On the ConvNet the CIFAR-10 and CIFAR-100 gains do not increase with skew, whereas on SVHN \citep{netzer2011svhn} the gain increases with skew, and on ResNet-18, where the same sets gain \(1.8\) to \(3.4\) points, the Dir-0.1 gain exceeds the IID gain on every partition (Appendices \ref{app:svhn} and \ref{app:crossarch}).

\section{Conclusion and Limitations}

We showed that independently distilled datasets can each reproduce training on their own source exactly and still fail when combined, even when the total budget admits an exact joint distillate. For quadratic objectives the failure has an exact form, the difference between averaging compressed source statistics and compressing their average, and learned quadratic distillates reproduce it. On a smooth network at small step sizes the quadratic prediction captures the local endpoint discrepancy in magnitude and direction. For learned synthetic sets the source residual, not the structural term, sets the composed error. Under endpoint matching it exceeds the structural term by more than an order of magnitude. Under distribution matching it partly offsets the local discrepancy. The accuracy gain of joint distillation persists even for two distillates of the same dataset. Single-source fidelity is therefore insufficient to judge composite fidelity, and training fidelity and downstream accuracy are distinct requirements, neither of which is established by evaluating each distillate on its own.

The exact results concern fixed-step gradient descent on quadratic objectives (A1) and apply to a network only locally (A4); adaptive optimizers and long horizons lie outside them (Appendix \ref{app:limitations}). One distillation method was used at the official schedule, and the network experiments have not entered the regime in which the structural term dominates; whether a distillation method can reach it is open, as is whether multi-step trajectory matching couples short-horizon fidelity to accuracy. For a downstream user the practical consequence is that a union of independently produced distillates should be validated as a whole against a reference on the union, since the fidelity of each constituent certifies neither the endpoint nor the accuracy of their combination.

% \clearpage

\bibliography{refs}
\bibliographystyle{iclr2027_conference}

\clearpage
\appendix
\mdfapptodefinestyle{resultboxnb}{nobreak=false}
\begin{center}{\Large\bfseries Appendix}\end{center}
\vspace{4pt}
\startcontents[appendix]
\printcontents[appendix]{}{1}[1]{\section*{Contents}\vspace{-4pt}}
\vspace{6pt}

\section{Extended Related Work}
\label{app:related-extended}

\paragraph{Dataset distillation objectives.}
Beyond the objectives cited in the main text, feature-alignment and
contrastive objectives \citep{wang2022cafe,sajedi2023datadam,ma2026fd}, kernel and
random-feature regression \citep{nguyen2021infinite,zhou2022frepo},
difficulty- and representativeness-aware matching
\citep{guo2024difficulty,liu2023dream,li2025diff,ye2025igds}, and large-scale, generative and
sample-weighting formulations
\citep{kim2022efficient,du2023trajectory,yin2023squeeze,li2024generative,su2024generative,zou2025dataset,cai2026evlf}
have broadened the setting (benchmark: \citealp{cui2022dcbench,li2025dd}; surveys:
\citealp{yu2024review,lei2024survey}), and closed-form linear-probe
distillation for pre-trained backbones \citep{peng2026closedform} shares
our aim of making part of the problem exactly solvable. In all of these
settings the evaluation unit is a single synthetic set.

\paragraph{Sharing distilled data across institutions.}
Distilled sets have been proposed as an anonymized, compact substitute for
sharing clinical images across institutions
\citep{li2020soft,li2022compressed,li2023medical}, where the real union is never
available and the downstream user receives independently distilled sets.
Federated variants \citep[building on federated averaging;][]{mcmahan2017communication}
apply a similar approach to client data, either once
\citep{zhou2020distilled,song2023fedd3} or iteratively
\citep{goetz2020synthetic,hu2022fedsynth,xiong2023feddm,wang2024distdd},
with the server training on the union. \citet{liu2023meta} observe that
locally condensed sets remain heterogeneous and reweight them,
\citet{wang2024fedaf} condense collaboratively, and
\citet{holland2024collaborative} report that naive merging underperforms.
Two recent designs are closest to our setting. \citet{arazzi2025secure}
distill a single synthetic set collaboratively, coordinating clients during
condensation and applying a differential-privacy mechanism to the shared
updates; because the set is jointly optimized, no post-hoc union occurs.
\citet{shi2024hfldd} partition clients into clusters, transmit
independently distilled sets to a cluster head, and train on their union
at the head, with clusters arranged so that the union is approximately
label-balanced. This is a post-hoc union of independent distillates, but
their convergence analysis concerns the federated model rather than the
training map induced by the union. In continual learning, condensed
examples are accumulated as replay memories
\citep{rosasco2022distilled,deng2022remember,yang2023efficient}, and
\citet{yang2024what} and \citet{chen2025curriculum} report that mixing
distilled and real samples can degrade performance.

\paragraph{Merging and mergeability.}
Mergeable summaries and composable coresets provide task-specific
guarantees that survive union by construction
\citep{agarwal2013mergeable,indyk2014composable,feldman2020tiny}.
Standard source-wise distillation objectives do not provide such
guarantees on their own, because they compress a training procedure rather
than a statistic. Model-merging studies address mode connectivity and
weight averaging \citep{wortsman2022soups}, task
arithmetic
\citep{ilharco2023task}, permutation symmetry \citep{ainsworth2023rebasin},
and curvature-aware weighting \citep{matena2022fisher}. The last is the
closest analogue: Fisher-weighted averaging corrects, to first order, a
curvature mismatch between separately fine-tuned parameters. In our
setting the mismatch lies in the data, and Theorem \ref{thm:structural}
expresses it in closed form as the variance of the source Hessians.

\paragraph{Theory of distillation.}
Existing analyses ask what a distillate can represent at a converged
solution, for example exact recovery of the empirical-risk minimizer in
linear and kernel settings \citep{izzo2023theoretical,nguyen2021kernel}, or
in closed form for a frozen backbone \citep{peng2026closedform}. We instead
fix an optimizer and a finite horizon and ask whether compression commutes
with source mixing.

\section{Formal Framework and Finite-Sample Realizability}
\label{app:framework}

\begin{table}[h]
\caption{Notation.}
\label{tab:notation-appendix}
\centering
\small
\begin{tabular}{@{}ll@{}}
\toprule
Symbol & Meaning\\
\midrule
\multicolumn{2}{@{}l}{\emph{Maps}}\\
\(\TrainMap_D^{K,\eta}\) & \(K\)-step full-batch GD map at step size \(\eta\) on dataset \(D\)\\
\(\TrainMap_{\rm ind},\ \TrainMap_R\) & composed synthetic one-step map; real-union two-step map\\
\(f_p(H)\) & time-compression transform \((I-(I-\eta H)^p)/\eta\), \(p=T/K\)\\
\addlinespace
\multicolumn{2}{@{}l}{\emph{Errors}}\\
\(\mathcal E,\ \mathcal C=\sqrt{\mathcal E}\) & squared endpoint error over \(\theta_0\sim P_0\); its RMS form\\
\(\epsilon_i\) & squared endpoint error of constituent \(i\)\\
\(\mathcal E_{\rm ind},\ \mathcal C_{\rm ind}\) & error of the independently composed union\\
\(\Gamma^\star_{\rm joint}\) & oracle composition gap, \(\mathcal E_{\rm ind}-\mathcal E^\star_{\rm joint}(M)\)\\
\(\widehat\Gamma_{\rm RMS}\) & empirical excess endpoint error, \(\mathcal C_{\rm ind}-\mathcal C_{\rm joint}\)\\
\(\widehat\Gamma_{\rm mean}\) & mean over seeds of the per-seed endpoint-distance difference (Appendix \ref{app:experiments})\\
\addlinespace
\multicolumn{2}{@{}l}{\emph{Structural quantities}}\\
\(V_H,\ w\) & curvature variance and affine conflict, Eq.~(\ref{eq:vh-w})\\
\(\mathcal C_{\rm struct},\ \widehat{\mathcal C}_{\rm comp}\) & quadratic-model RMS gap with exact / actual constituents\\
\(\mathcal C_{\rm loc}\) & measured local reference (network, exact constituents)\\
\(P_{\rm comp}\) & HVP-computable estimate of \(\mathcal C_{\rm struct}^2\) for \(T=2,K=1\)\\
\(r_H\) & effective rank of \(V_H\)\\
\(M_0\) & second-moment matrix of the initialization about \(\theta^*\)\\
\addlinespace
\multicolumn{2}{@{}l}{\emph{Bounds}}\\
\(\mathcal U_{\rm src}\) & weighted constituent RMS, \(\sum_i\alpha_i\sqrt{\epsilon_i}\)\\
\(\Delta_{\rm aff},\ \delta_S,\ \delta_R\) & affine-stability and nonlinear remainders\\
\(\zeta\) & their sum, the constituent-plus-nonlinear uncertainty\\
\bottomrule
\end{tabular}
\end{table}

\subsection{Weighted unions and training maps}

A weighted dataset is a finite collection of examples with nonnegative weights
that sum to one. The loss of \(D_1\oplus_\alpha D_2\) is
\(\alpha L_{D_1}+(1-\alpha)L_{D_2}\). This definition separates the intended
source mass from the cardinality of a stored synthetic set. Naive concatenation
is the special case in which empirical sample counts determine the weights.

Throughout the theory, \(\TrainMap_D^{K,\eta}\) denotes full-batch gradient
descent with fixed learning rate \(\eta\), initialized from \(P_0\). All
endpoint comparisons use the same parameterization and initialization
distribution. The RMS quantity \(\mathcal C_{P_0}\) obeys the triangle inequality as an
\(L_2(P_0)\) norm. Its square \(\mathcal E_{P_0}\) is used for exact error
formulas.

The oracle joint error is no larger than the independent error when
\(M\ge\sum_i|S_i|\), since the independently composed set is itself feasible
in the oracle optimization. For the exact quadratic constructions, the next
realizability lemma (Lemma \ref{lem:realizability}) supplies a joint candidate within this budget rather than relying only
on the trivial feasibility argument.

\subsection{Proof of the composability--affinity principle}

\begin{proof}[Proof of Proposition \ref{thm:affinity}]
Condition (ii) immediately implies (i), while (iii) implies (ii). It remains
to show that (i) implies (iii). Repeated midpoint composition gives, for every
dyadic rational \(q\in[0,1]\),
\begin{equation}
C(qx+(1-q)y)=qC(x)+(1-q)C(y).
\end{equation}
Continuity of \(C\) extends the identity from dyadic rationals to every
\(q\in[0,1]\), so \(C\) is affine on \(\mathcal X\). If composition is
specified only through a downstream endpoint map \(E\), injectivity of \(E\)
on the relevant convex hull turns
\begin{equation}
E(C((x+y)/2))=E((C(x)+C(y))/2)
\end{equation}
back into the statistic-level midpoint identity, and the same argument
applies. Injectivity is a property of the induced operator map
\(\theta\mapsto\Phi^{K,\eta}(\theta)\) on \(\Theta\), or of its restriction to
an initialization distribution with \(M_0\succ0\), not of an endpoint at a
single \(\theta_0\); for quadratics it holds under (A3) together with the
positive-branch condition on the compressed statistics, whereas for
overparameterized networks the map from surrogate statistics to
trajectories is generally not injective, so the equivalence is an algebraic
property of the quadratic model rather than an impossibility result for
networks.
\end{proof}

\subsection{Finite-sample realizability}
\label{app:realizability-lemma}

\begin{lemma}[Quadratic objectives are finite synthetic datasets]
\label{lem:realizability}
Let
\begin{equation}
L(\theta)=\frac12\theta^\top H\theta-b^\top\theta+c,
\qquad H\succeq0,\quad b\in\operatorname{range}(H).
\end{equation}
Up to an additive constant, \(L\) is the sum of
\(r=\operatorname{rank}(H)\) scalar squared-loss examples. For
\(p=T/K>1\) and \(0\preceq\eta H\prec I\), the transformed statistics
\begin{equation}
\widetilde H=f_p(H),
\qquad
\widetilde b=f_p(H)H^\dagger b
\end{equation}
have the same range and rank as \(H\), reproduce \(T\) real steps in \(K\)
synthetic steps, and are realizable with the same budget.
For positive-mass sources \((H_i,b_i)\), the exact transformed real-union
objective needs at most \(\sum_i\operatorname{rank}(H_i)\) examples.
\end{lemma}

\begin{proof}
Write the positive eigendecomposition
\begin{equation}
H=\sum_{j=1}^r\lambda_ju_ju_j^\top,
\qquad b=\sum_{j=1}^rb_ju_j.
\end{equation}
For \(x_j=\sqrt{\lambda_j}u_j\) and
\(y_j=b_j/\sqrt{\lambda_j}\),
\begin{equation}
\frac12\sum_{j=1}^r(x_j^\top\theta-y_j)^2
=\frac12\theta^\top H\theta-b^\top\theta+
\frac12\sum_{j=1}^ry_j^2.
\end{equation}
The scalar map
\(\lambda\mapsto[1-(1-\eta\lambda)^p]/\eta\) is positive exactly when
\(\lambda>0\) on the stable interval, so \(\widetilde H\) has the same range
and rank. Since \(b\in\operatorname{range}(H)\),
\(\widetilde b=f_p(H)H^\dagger b\) belongs to that range. If
\(u=H^\dagger b\), both real and synthetic objectives have optimum \(u\) on
\(\operatorname{range}(H)\), while
\begin{equation}
(I-\eta\widetilde H)^K=(I-\eta H)^T.
\end{equation}
They therefore induce the same affine endpoint map.

For PSD matrices,
\begin{equation}
\ker\!\left(\sum_i\alpha_iH_i\right)=\bigcap_i\ker(H_i).
\end{equation}
Because \(b_i\in\operatorname{range}(H_i)=\ker(H_i)^\perp\), the union
linear statistic \(\bar b=\sum_i\alpha_ib_i\) belongs to
\(\operatorname{range}(\bar H)\). Finally,
\begin{equation}
\operatorname{rank}(\bar H)
\le\sum_i\operatorname{rank}(H_i),
\end{equation}
which proves the budget statement.
\end{proof}

\section{One-Dimensional Exact Separation}
\label{app:oned}

For \(L_a(\theta)=\frac a2(\theta-u)^2\), gradient descent gives
\begin{equation}
\TrainMap_a^{T,\eta}(\theta_0)-u=(1-\eta a)^T(\theta_0-u).
\end{equation}
The time-compression transform \(g_p(a):=(1-(1-\eta a)^p)/\eta\) of Section
\ref{sec:commutation} satisfies
\begin{equation}
(1-\eta g_p(a))^K=(1-\eta a)^T,
\end{equation}
so a single synthetic squared-loss example with
\(x=\sqrt{g_p(a)}\) and \(y=xu\) is individually exact.

\begin{figure}[h]
\centering
\includegraphics[width=0.9\linewidth]{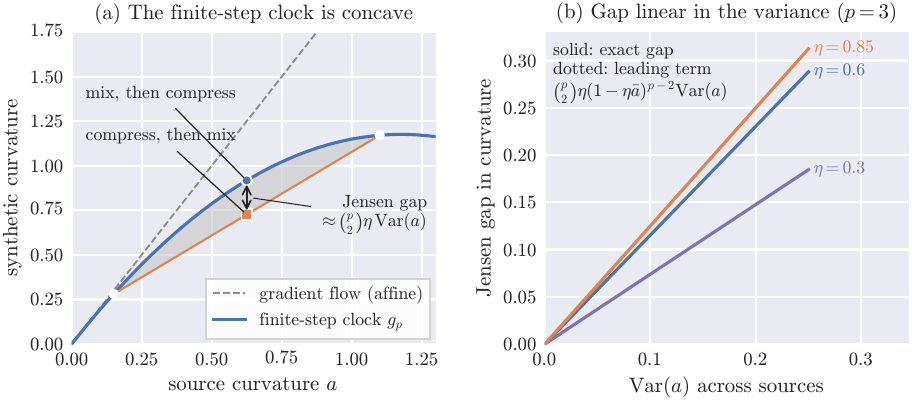}
\caption{Concavity of the finite-step clock. \(g_p\) is
concave in the source curvature, so the average of two compressed sources
(compress, then mix) lies below the compression of their average (mix,
then compress); under gradient flow (dashed) the two coincide.}
\label{fig:clock}
\end{figure}

\begin{theorem}[Exact constituents need not compose]
\label{thm:1d}
Let \(a_1\ne a_2\), \(\alpha\in(0,1)\), and distill each source exactly using
\(g_p\). Under the correct union mass \(\alpha\),
\begin{equation}
\alpha g_p(a_1)+(1-\alpha)g_p(a_2)
<g_p\!\left(\alpha a_1+(1-\alpha)a_2\right).
\end{equation}
Thus \(\epsilon_1=\epsilon_2=0\) but
\(\Gamma_{\rm joint}^\star=\mathcal E_{\rm ind}>0\) for every
non-degenerate initialization distribution around \(u\).
\end{theorem}

\begin{proof}[Proof of Theorem \ref{thm:1d}]
Direct differentiation yields
\begin{equation}
g_p''(a)=-p(p-1)\eta(1-\eta a)^{p-2}<0.
\end{equation}
Strict concavity gives
\begin{equation}
\bar s:=\alpha g_p(a_1)+(1-\alpha)g_p(a_2)
<g_p(\bar a)=:s_{\rm joint},
\quad \bar a:=\alpha a_1+(1-\alpha)a_2.
\end{equation}
The independent endpoint error is therefore
\begin{equation}
\mathcal E_{\rm ind}
=\E[(\theta_0-u)^2]
\left[(1-\eta\bar s)^K-(1-\eta\bar a)^T\right]^2>0.
\end{equation}
The exact joint curvature is \(s_{\rm joint}\), realizable by one squared-loss
example, so \(\mathcal E_{\rm joint}^\star=0\). Each constituent is exact by
construction, hence \(\epsilon_1=\epsilon_2=0\).
\end{proof}

Figure \ref{fig:reweighting} illustrates the single-mode repair and the
two-mode residual floor.

\begin{figure}[h]
\centering
\includegraphics[width=0.55\linewidth]{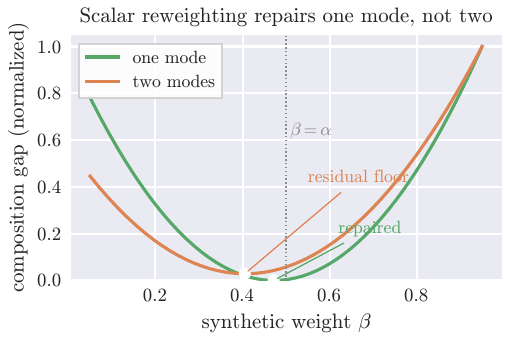}
\caption{Scalar reweighting repairs one curvature mode, not two.
Composition gap (normalized) as a function of the synthetic weight
\(\beta\): for a single mode there is a \(\beta^*\ne\alpha\) at which the
gap is zero; for two modes with different \(\beta^*_j\) the minimum over
\(\beta\) is a residual floor.}
\label{fig:reweighting}
\end{figure}

\begin{proposition}[Mode-wise reweighting compatibility]
\label{prop:mode-calibration}
For two scalar modes with \(a_{1j}<a_{2j}\), define
\begin{equation}
\beta_j^*=
\frac{g_p(a_{2j})-g_p(\alpha a_{1j}+(1-\alpha)a_{2j})}
{g_p(a_{2j})-g_p(a_{1j})}.
\label{eq:mode-calibration}
\end{equation}
Then \(\beta_j^*\in(0,1)\) is the unique source-1 synthetic weight that
repairs mode \(j\). A diagonal multi-mode problem is exactly repairable by one
global reweighting if and only if all nontrivial \(\beta_j^*\) coincide.
For \(p=2\), this weight has the closed form
\begin{equation}
\beta_j^*=\alpha\,
\frac{2-\eta(a_{2j}+\bar a_j)}
{2-\eta(a_{1j}+a_{2j})},
\qquad \bar a_j=\alpha a_{1j}+(1-\alpha)a_{2j}.
\label{eq:mode-calibration-p2}
\end{equation}
Consequently, in a nondegenerate mode \(\beta_j^*=\alpha\) if and only if
\(\eta=0\), and equality of the repair weights of two distinct modes is a
measure-zero, codimension-one condition on their four curvatures.
\end{proposition}

\begin{proof}
Strict monotonicity of \(g_p\) places
\(g_p(\alpha a_{1j}+(1-\alpha)a_{2j})\) strictly between the two endpoint
values, so solving the affine interpolation equation gives the unique weight
\eqref{eq:mode-calibration}. Simultaneous repair requires the same scalar
weight to solve every modal equation, which is equivalent to equality of the
\(\beta_j^*\). For \(p=2\), substituting
\(g_2(a)=a(2-\eta a)\) and factoring differences gives
\eqref{eq:mode-calibration-p2}. Its denominator is positive on the stable
branch. Since \(\bar a_j-a_{1j}=(1-\alpha)(a_{2j}-a_{1j})>0\), the equality
\(\beta_j^*=\alpha\) holds exactly when \(\eta=0\). Equating the two rational
expressions for modes \(j\) and \(j'\), then clearing their positive
denominators, gives one nonzero polynomial equation in
\((a_{1j},a_{2j},a_{1j'},a_{2j'})\). Its regular zero set has codimension one,
and the full zero set has Lebesgue measure zero. For general real \(p>1\),
\(g_p\) and \(\beta_j^*\) are real analytic on the stable branch. The equality
of two repair weights is again a nontrivial analytic equation: if the weight
were constant over all endpoint pairs, the coincident-endpoint limit would
force that constant to be \(\alpha\), contradicting strict concavity of
\(g_p\). Its zero set therefore also has Lebesgue measure zero.
\end{proof}

For any fixed real \(p>1\), the generalized binomial expansion gives
\begin{equation}
g_p(a)=pa-\binom{p}{2}\eta a^2+O(\eta^2a^3).
\end{equation}
The leading Jensen gap is therefore
\begin{equation}
g_p(\bar a)-\sum_i\alpha_i g_p(a_i)
=\binom{p}{2}\eta\operatorname{Var}_\alpha(a_i)+O(\eta^2),
\end{equation}
which vanishes in the gradient-flow limit while predicting the curvature
variance that appears in the matrix law.

\subsection{Quadratic endpoint matching (E5)}
\label{app:e5}

The identities above describe what an exact distillate must contain. E5
asks whether a distillation algorithm finds it. For twenty random pairs of
quadratic sources in \(d=12\) dimensions, a synthetic set of \(n_{\rm syn}\)
squared-loss samples per source is optimized by endpoint matching (gradient
descent on the two-step versus one-step endpoint discrepancy over random
\(\theta_0\), 6{,}000 steps, random initialization, no access to the
target statistics). Table \ref{tab:e5} reports, per budget, the relative
Frobenius error of the learned Gram matrix against the transformed
statistic \(2H-\eta H^2\), and the ratio of the measured composed defect to
the closed form \(\tfrac{\eta^2}{4}(H_1-H_2)^2\) of Theorem
\ref{thm:structural} (Appendix \ref{app:matrix}).

\begin{table}[h]
\caption{E5: endpoint matching recovers the transformed statistic
once the budget reaches the dimension.}
\label{tab:e5}
\centering
\small
\begin{tabular}{@{}rrrrrr@{}}
\toprule
\(n_{\rm syn}\) & rank of Gram & rel.\ Frobenius error & endpoint loss & measured/closed-form defect & log-log corr.\\
\midrule
4 & 4 & 48.0\% & \(4.8\times10^{-1}\) & 11.7 \(\pm\) 1.4 & 0.12\\
6 & 6 & 31.7\% & \(2.0\times10^{-1}\) & 7.6 \(\pm\) 0.9 & 0.16\\
8 & 8 & 19.9\% & \(8.9\times10^{-2}\) & 4.8 \(\pm\) 0.6 & 0.16\\
12 & 12 & \(<10^{-6}\) & \(1.3\times10^{-31}\) & 1.000 \(\pm\) 0.000 & 1.00\\
24 & 12 & \(<10^{-6}\) & \(1.6\times10^{-31}\) & 1.000 \(\pm\) 0.000 & 1.00\\
\bottomrule
\end{tabular}
\par\smallskip\raggedright\footnotesize
Twenty source pairs, \(d=12\), mean \(\pm\) SD over pairs. Log-log
correlation is between the measured defect and \(\|(H_1-H_2)^2\|_F\) across
pairs.
\end{table}

At \(n_{\rm syn}\ge d\) the algorithm recovers
\(2H-\eta H^2\) to machine precision from a random start and the
composed defect equals the closed form exactly. Below the dimension the
constituents are inexact and the composed error is dominated by
per-source approximation error, the budget-limited situation of the
CIFAR-10 runs (Appendix \ref{app:regime}).
In this experiment every budget-limited run gave a composed error above
the closed form, but that is not a general property of approximate
quadratic constituents: per-source approximation errors can cancel the
structural term. With \(H_1=1\), \(H_2=3\), \(\eta=0.1\), equal masses, and
a common optimum, the exact constituents have transformed curvatures
\(1.9\) and \(5.1\), whose average \(3.5\) differs from the joint value
\(3.6\). Approximate constituents with curvatures \(2.0\) and \(5.2\) each
carry an error, yet their average is exactly \(3.6\), so
\(\mathcal C_{\rm ind}=0<\mathcal C_{\rm struct}\), which is consistent with
Corollary \ref{cor:detectability} (whose positivity condition is sufficient, not necessary). Two
quantities should be kept apart here: \(\mathcal C_{\rm ind}=\|s+e_{\rm
ind}\|\) can fall below \(\|s\|\) only through cancellation between the
composed set's error and the structural term, whereas
\(\widehat\Gamma_{\rm RMS}=\|s+e_{\rm ind}\|-\|e_{\rm joint}\|\) also
depends on the joint set's error. The ratios near one in the network
budget sweep (Table \ref{tab:ipc}) concern the first quantity.

\section{Exact Matrix Composition Laws}
\label{app:matrix}

\subsection{Two-source identity}

For equal source masses and common optimum \(\theta^*\), let
\(\bar H=(H_1+H_2)/2\). The independent one-step contraction is
\begin{equation}
A_{\rm ind}
=I-\eta\left(2\bar H-\eta\frac{H_1^2+H_2^2}{2}\right),
\end{equation}
whereas the real-union two-step contraction is
\begin{equation}
A_R=(I-\eta\bar H)^2
=I-2\eta\bar H+\eta^2\bar H^2.
\end{equation}
Thus
\begin{equation}
A_{\rm ind}-A_R
=\eta^2\left(\frac{H_1^2+H_2^2}{2}-\bar H^2\right)
=\frac{\eta^2}{4}(H_1-H_2)^2.
\end{equation}
No commutativity is used in the last identity. For \(K=1\) the composed
one-step map is linear in the synthetic losses, so with exact constituents
\(\Phi_{\rm ind}(\theta_0)=\sum_i\alpha_i\Phi^{T,\eta}_{D_i}(\theta_0)\)
and the defect coincides with the drift between averaging local \(T\)-step
updates and taking \(T\) steps on the pooled data, the parameter-averaging
drift of federated optimization \citep{li2020fedavg}. The departure specific
to distillation appears through the transformed statistics and, for
\(K>1\), through the compounded contractions of Theorem \ref{thm:defect}.

\subsection{Multi-source identity}

For arbitrary positive masses,
\begin{align}
\sum_i\alpha_i(H_i-\bar H)^2
&=\sum_i\alpha_iH_i^2-\bar H^2,\\
\frac12\sum_{i,j}\alpha_i\alpha_j(H_i-H_j)^2
&=\sum_i\alpha_iH_i^2-\bar H^2.
\end{align}
Every summand in the first representation is of the form \(X^2\) for a
symmetric \(X\), proving \(V_H\succeq0\). Moreover, for a vector \(v\),
\begin{equation}
v^\top V_Hv=\sum_i\alpha_i\|(H_i-\bar H)v\|_2^2.
\end{equation}
It follows that \(V_H=0\) exactly when all positive-mass Hessians are equal.

\paragraph{Scope of positive semidefiniteness.}
The identities above require only symmetric \(H_i\): each
\((H_i-\bar H)^2\) remains positive semidefinite even if \(H_i\) is
indefinite. Positive semidefiniteness is instead used for finite-dataset
realizability, the positive contraction branch, and the rank-budget
interpretation. Consequently, an exact-loss Hessian can still define the local
algebraic proxy, while a generalized Gauss--Newton matrix supplies a PSD
surrogate aligned with the signed and realizability results.

\subsection{Curvature and affine statistics}

\begin{proof}[Proof of Theorem \ref{thm:structural}]
The independent synthetic union has statistics
\begin{equation}
H_{\rm ind}=2\bar H-\eta\sum_i\alpha_iH_i^2,
\qquad
b_{\rm ind}=2\bar b-\eta\sum_i\alpha_iH_ib_i.
\end{equation}
The exact joint one-step statistics are
\begin{equation}
H_J=2\bar H-\eta\bar H^2,
\qquad
b_J=(2I-\eta\bar H)\bar b.
\end{equation}
Hence \(H_{\rm ind}=H_J-\eta V_H\) and
\(b_{\rm ind}=b_J-\eta w\). A one-step affine GD map is
\((I-\eta H)\theta+\eta b\), which gives \eqref{eq:affine-law}. Write \(\theta_0=\mu_0+u_0\), with
\(\E[u_0]=0\) and \(\E[u_0u_0^\top]=\Sigma_0\). Squaring and taking
expectation yields \eqref{eq:affine-error}.
\end{proof}

If \(b_i=H_i\theta^*\) for all \(i\), then
\begin{equation}
w=\left(\sum_i\alpha_iH_i^2-\bar H^2\right)\theta^*
=V_H\theta^*,
\end{equation}
which gives \eqref{eq:common-optimum}.

\begin{corollary}[Distribution-aware zero-gap characterization]
\label{cor:zero-gap}
In the setting of Theorem \ref{thm:structural}, the composition error is zero
if and only if
\begin{equation}
V_H\Sigma_0^{1/2}=0\quad\text{and}\quad V_H\mu_0=w.
\label{eq:zero-gap}
\end{equation}
Under a common optimum this is equivalent to \(V_HM_0^{1/2}=0\). If
\(M_0\succ0\), exact composition holds iff all positive-mass source Hessians
are equal.
\end{corollary}

\begin{proof}[Proof of Corollary \ref{cor:zero-gap}]
The identity \eqref{eq:affine-error} is a sum of two nonnegative terms, with
\begin{equation}
\operatorname{tr}(V_H^2\Sigma_0)
=\|V_H\Sigma_0^{1/2}\|_F^2.
\end{equation}
It is therefore zero exactly under the two conditions in
\eqref{eq:zero-gap}. With a common optimum, \(w=V_H\theta^*\), and the same
sum is \(\|V_HM_0^{1/2}\|_F^2\). If \(M_0\succ0\), this vanishes exactly
when \(V_H=0\), which by the sum-of-squares identity in Appendix
\ref{app:matrix} holds exactly when every positive-mass \(H_i\) equals
\(\bar H\). The same conclusion holds without a common optimum whenever
\(\Sigma_0\succ0\): the first condition in \eqref{eq:zero-gap} then forces
\(V_H=0\), hence \(H_i=\bar H\) for all positive-mass sources, and
substituting into \(w\) gives \(w=\bar H\bar b-\bar H\bar b=0\), so the
second condition holds for every \(\mu_0\). Affine compatibility is an
independent requirement only when \(\Sigma_0\) is rank-deficient, for
instance at a fixed initialization \(\Sigma_0=0\), where \(V_H\mu_0=w\)
can hold with \(w\ne0\).
\end{proof}

\paragraph{Why union does not average minimizers.}
Dataset union acts linearly on \((H,b)\), not on \(u=H^{-1}b\). If two
Hessians are uniformly positive definite with minimum eigenvalue at least
\(m>0\), then
\begin{equation}
\|\widehat u-u^\circ\|_2
\le\frac1m\left(\|\Delta b\|_2+
\|\Delta H\|_2\|u^\circ\|_2\right).
\end{equation}
Without this margin, small affine-statistic error does not imply minimizer
closeness.

\paragraph{Relation to Jensen's inequality.}
\label{app:faq-jensen}
Strict concavity of the scalar clock \(g_p\) supplies the one-dimensional
case (Theorem \ref{thm:1d}) and nothing more. The matrix law (Theorem
\ref{thm:structural}) is an exact identity for noncommuting Hessians rather
than an inequality: it characterizes the endpoint defect, yields a
positive-semidefinite curvature discrepancy, separates a curvature-variance
term from an affine-conflict term, and identifies which components admit
scalar repair (Appendix \ref{app:repair}). The arbitrary-ratio
decomposition (Theorem \ref{thm:defect}) shows that curvature variance
remains the leading term under a common optimum, and the
strong-monotonicity lower bound (Lemma \ref{lem:power-lower}) rules out
endpoint cancellation across \(K\) steps without any Loewner ordering.
When the Hessians commute, spectral Jensen recovers a signed refinement
(Theorem \ref{thm:general}). In general no operator-concavity statement is
assumed or needed, and \(P_{\rm comp}\) follows from the identity, not
from the inequality.

\subsection{Momentum and minibatch stochasticity}
\label{app:optimizers}

The two-to-one identity of Theorem \ref{thm:structural} is stated for
full-batch gradient descent. The following lemma records that the same
discrepancy arises under heavy-ball momentum and, in expectation, under
independent minibatches. Adaptive preconditioners such as Adam, whose
effective step depends on the gradient history, are not covered.

\begin{lemma}[Two-to-one identity under momentum and minibatches]
\label{lem:optimizers}
Let the sources satisfy (A1) with masses \(\alpha_i\).
\begin{enumerate}
\item[(i)] Let the real training use heavy-ball momentum with a fixed
coefficient \(\beta\ge0\) and zero initial momentum,
\(\theta_{t+1}=\theta_t-\eta(H\theta_t-b)+\beta(\theta_t-\theta_{t-1})\)
with \(\theta_{-1}=\theta_0\), and let one synthetic step be a plain
gradient step. The synthetic statistics that reproduce two real steps on
source \(i\) exactly are
\(\widetilde H_i=(2+\beta)H_i-\eta H_i^2\) and
\(\widetilde b_i=((2+\beta)I-\eta H_i)b_i\), and the union of these exact
distillates satisfies
\(\TrainMap_{\rm ind}(\theta_0)-\TrainMap_R(\theta_0)=\eta^2(V_H\theta_0-w)\),
the same discrepancy as in \eqref{eq:affine-law}.
\item[(ii)] Let each real step use an independent unbiased minibatch,
with statistics \((H^{(t)},b^{(t)})\), \(\E[H^{(t)}]=H\),
\(\E[b^{(t)}]=b\), independent of \(\theta_0\) and of each other. Then
the expected two-step endpoint equals the full-batch endpoint,
\(\E[\theta_2\mid\theta_0]=(I-\eta H)^2\theta_0+\eta(2I-\eta H)b\), so the
statistics of Theorem \ref{thm:structural} reproduce it exactly and the
discrepancy of the composed one-step map from the expected real-union
endpoint is again \(\eta^2(V_H\theta_0-w)\). Against the stochastic
reference, the mean squared endpoint error of any deterministic one-step
candidate is its error against the expected endpoint plus the reference
variance \(\E\|\theta_2-\E[\theta_2\mid\theta_0]\|^2\), which does not
depend on the candidate and cancels in the composition gap
\(\Gamma^\star_{\rm joint}\).
\end{enumerate}
\end{lemma}

\begin{proof}
(i) With zero initial momentum the first step is
\(\theta_1=(I-\eta H)\theta_0+\eta b\) and
\(\theta_1-\theta_0=-\eta(H\theta_0-b)\). The second step gives
\(\theta_2=(I-\eta H)\theta_1+\eta b+\beta(\theta_1-\theta_0)
=[(I-\eta H)^2-\beta\eta H]\theta_0+\eta[(2+\beta)I-\eta H]b\).
A plain step with statistics \((\widetilde H,\widetilde b)\) maps
\(\theta_0\) to \((I-\eta\widetilde H)\theta_0+\eta\widetilde b\), and
equating the two maps for all \(\theta_0\) gives
\(\widetilde H=(2+\beta)H-\eta H^2\) and
\(\widetilde b=((2+\beta)I-\eta H)b\). Applying the same rule to the
union, whose statistics are \((\bar H,\bar b)\), the composed and joint
statistics differ by
\(\sum_i\alpha_i\widetilde H_i-\widetilde H_\cup=-\eta(\sum_i\alpha_iH_i^2-\bar H^2)=-\eta V_H\)
and
\(\sum_i\alpha_i\widetilde b_i-\widetilde b_\cup=-\eta(\sum_i\alpha_iH_ib_i-\bar H\bar b)=-\eta w\);
the terms carrying \(\beta\) are linear in \((H,b)\) and cancel. One
synthetic step with these differences yields
\(\eta^2(V_H\theta_0-w)\), as in the proof of Theorem \ref{thm:structural}.
(ii) Two minibatch steps give
\(\theta_2=(I-\eta H^{(2)})(I-\eta H^{(1)})\theta_0+\eta(I-\eta H^{(2)})b^{(1)}+\eta b^{(2)}\).
Taking the conditional expectation and using the independence of the two
batches gives the full-batch expression. For a deterministic candidate
endpoint \(\phi(\theta_0)\),
\(\E\|\phi-\theta_2\|^2=\|\phi-\E[\theta_2\mid\theta_0]\|^2+\E\|\theta_2-\E[\theta_2\mid\theta_0]\|^2\)
because the cross term has zero mean; the second term is the same for the
composed candidate and for every joint candidate, so it cancels in
\(\mathcal E_{\rm ind}-\mathcal E^\star_{\rm joint}\).
\end{proof}

\section{General Finite-Step Jensen Gap}
\label{app:general}

\subsection{General noncommutative arbitrary-ratio decomposition}

\begin{lemma}[Strong monotonicity of positive matrix powers]
\label{lem:power-lower}
If symmetric matrices \(X,Y\succeq qI\), with \(q>0\), then for every
integer \(K\ge1\),
\begin{equation}
\|X^K-Y^K\|_F\ge Kq^{K-1}\|X-Y\|_F.
\label{eq:power-lower}
\end{equation}
No commutativity between \(X\) and \(Y\) is required.
\end{lemma}

\begin{proof}
Let \(E=X-Y\) and \(Z_t=Y+tE\succeq qI\). The Fr\'echet derivative is
\begin{equation}
D(Z_t^K)[E]=\sum_{\ell=0}^{K-1}Z_t^\ell E Z_t^{K-1-\ell}.
\end{equation}
For each summand,
\begin{align}
\left\langle E,Z_t^\ell E Z_t^{K-1-\ell}\right\rangle_F
&=\left\|Z_t^{\ell/2}E
Z_t^{(K-1-\ell)/2}\right\|_F^2\\
&\ge q^{K-1}\|E\|_F^2.
\end{align}
Integrating the derivative along \(Z_t\) and applying Cauchy--Schwarz gives
\begin{equation}
\|X^K-Y^K\|_F\|E\|_F
\ge\langle X^K-Y^K,E\rangle_F
\ge Kq^{K-1}\|E\|_F^2,
\end{equation}
which proves \eqref{eq:power-lower}.
\end{proof}

\paragraph{Centered matrix-power bound.}
Let symmetric \(H_i\) satisfy \(\|H_i\|_2\le L\),
\(\bar H=\sum_i\alpha_iH_i\), and \(\Delta_i=H_i-\bar H\). For every
integer \(k\ge2\),
\begin{align}
\left\|\bar H^k-\sum_i\alpha_iH_i^k\right\|_F
&\le \binom{k}{2}L^{k-2}
\sum_i\alpha_i\|\Delta_i\|_2\|\Delta_i\|_F\nonumber\\
&\le \binom{k}{2}L^{k-2}\operatorname{tr}(V_H).
\label{eq:centered-power}
\end{align}

\begin{proof}
For \(F_k(X)=X^k\), the second Fr\'echet derivative is
\begin{equation}
D^2F_k(X)[E,E]
=2\sum_{0\le a<b\le k-1}
X^aEX^{b-a-1}EX^{k-1-b}.
\end{equation}
Along \(X_t=\bar H+t\Delta_i=(1-t)\bar H+tH_i\), one has
\(\|X_t\|_2\le L\). Each of the \(k(k-1)\) ordered insertion terms is
therefore bounded in Frobenius norm by
\(L^{k-2}\|\Delta_i\|_2\|\Delta_i\|_F\). Taylor's formula with integral
remainder gives
\begin{equation}
H_i^k=\bar H^k+DF_k(\bar H)[\Delta_i]
+\int_0^1(1-t)D^2F_k(X_t)[\Delta_i,\Delta_i]dt.
\end{equation}
Averaging cancels the first derivative since
\(\sum_i\alpha_i\Delta_i=0\). The integral contributes the factor
\(\int_0^1(1-t)dt=1/2\), leaving \(\binom{k}{2}\). Finally,
\begin{equation}
\sum_i\alpha_i\|\Delta_i\|_2\|\Delta_i\|_F
\le\sum_i\alpha_i\|\Delta_i\|_F^2
=\operatorname{tr}(V_H),
\end{equation}
which proves \eqref{eq:centered-power}.
\end{proof}

For \(p=T/K>1\), symmetric \(H\) with \(\|H\|_2\le L\), and \(\eta L<1\),
the spectral transform \(f_p(H):=(I-(I-\eta H)^p)/\eta\) satisfies
\((I-\eta f_p(H))^K=(I-\eta H)^T\). The following theorem separates the
universal leading term from higher-order, possibly noncommutative,
corrections for every real compression ratio.

\begin{theorem}[Arbitrary-ratio defect decomposition]
\label{thm:defect}
Let real \(p>1\), symmetric \(H_i\) with \(\|H_i\|_2\le L\),
\(\eta L<1\), and suppose the
sources share an optimum. Define
\begin{equation}
H_{\rm ind}=\sum_i\alpha_i f_p(H_i),\quad
H_J=f_p(\bar H),\quad J_p=H_J-H_{\rm ind}.
\end{equation}
Let \(r_H:=\operatorname{tr}(V_H)^2/\|V_H\|_F^2\) be the effective rank
(participation ratio) of \(V_H\ne0\), so that
\(\operatorname{tr}(V_H)=\sqrt{r_H}\,\|V_H\|_F\) for \(V_H\succeq0\) and
\(1\le r_H\le\operatorname{rank}(V_H)\), and, for \(0\le x<1\), define
\begin{equation}
c_p(x):=\sum_{k=3}^{\infty}\left|\binom pk\right|
\binom{k}{2}x^{k-3}.
\end{equation}
Without any commutativity assumption,
\begin{align}
J_p&=\binom{p}{2}\eta V_H+R_p,\label{eq:defect-decomp}\\
R_p&=\sum_{k=3}^{\infty}(-1)^{k+1}\binom{p}{k}\eta^{k-1}
\left(\bar H^k-\sum_i\alpha_iH_i^k\right),\nonumber\\
\|R_p\|_F&\le \eta^2L c_p(\eta L)\operatorname{tr}(V_H)\nonumber\\
&= \eta^2L c_p(\eta L)\sqrt{r_H}\|V_H\|_F.
\label{eq:defect-remainder}
\end{align}
Let \(B_{\rm ind}=I-\eta H_{\rm ind}\) and \(B_J=I-\eta H_J\), and write
\(A_{\rm ind}:=B_{\rm ind}^K\) and \(A_R:=B_J^K\) for the \(K\)-step
contractions. The endpoint defect is exactly
\begin{equation}
A_{\rm ind}-A_R
=\eta\sum_{\ell=0}^{K-1}B_{\rm ind}^{K-1-\ell}J_pB_J^\ell.
\label{eq:defect-endpoint}
\end{equation}
For \(K=1\),
\begin{equation}
\mathcal E_{\rm ind}=\eta^2\|J_pM_0^{1/2}\|_F^2,
\label{eq:defect-error}
\end{equation}
For arbitrary \(K\), if \(M_0\succeq\lambda_0I\) and
\(q:=(1-\eta L)^p\), then
\begin{equation}
\sqrt{\mathcal E_{\rm ind}}
\ge K\eta q^{K-1}\sqrt{\lambda_0}\|J_p\|_F
\ge K\eta^2 q^{K-1}\sqrt{\lambda_0}\|V_H\|_F
\left[\binom{p}{2}-\eta Lc_p(\eta L)\sqrt{r_H}\right]_+.
\label{eq:defect-lower}
\end{equation}
\end{theorem}

\begin{proof}[Proof of Theorem \ref{thm:defect}]
Because \(\|\eta H\|_2<1\), the generalized binomial series converges
absolutely in operator norm and gives
\begin{equation}
f_p(H)=\sum_{k=1}^{\infty}(-1)^{k+1}
\binom{p}{k}\eta^{k-1}H^k.
\end{equation}
For integer \(p\), the series terminates at \(k=p\). No product between
distinct source Hessians occurs, so no commutativity assumption is needed.
Subtracting the weighted source expansion from the expansion at \(\bar H\),
the \(k=1\) terms cancel. Since
\begin{equation}
\bar H^2-\sum_i\alpha_iH_i^2=-V_H,
\end{equation}
the \(k=2\) term is \(\binom{p}{2}\eta V_H\), and the remaining terms are
exactly \(R_p\) in \eqref{eq:defect-decomp}. Applying
\eqref{eq:centered-power} termwise gives
\begin{align}
\|R_p\|_F
&\le \eta^2L\operatorname{tr}(V_H)
\sum_{k=3}^{\infty}\left|\binom pk\right|
\binom{k}{2}(\eta L)^{k-3}\nonumber\\
&=\eta^2Lc_p(\eta L)\operatorname{tr}(V_H).
\end{align}
The series defining \(c_p\) converges for \(\eta L<1\). Since
\(V_H\succeq0\), \(\operatorname{tr}(V_H)=\sqrt{r_H}\,\|V_H\|_F\) by the
definition of the effective rank, proving \eqref{eq:defect-remainder}. (With
the algebraic rank \(\rho=\operatorname{rank}V_H\) one has instead
\(\operatorname{tr}(V_H)\le\sqrt{\rho}\|V_H\|_F\). Since \(r_H\le\rho\), the
effective-rank form is the sharper of the two.)

The exact joint contraction is
\(B_J=I-\eta f_p(\bar H)=(I-\eta\bar H)^p\), hence
\(B_J^K=(I-\eta\bar H)^T=A_R\). Also
\(B_{\rm ind}-B_J=\eta J_p\). The difference-of-powers identity gives
\begin{equation}
B_{\rm ind}^K-B_J^K
=\eta\sum_{\ell=0}^{K-1}
B_{\rm ind}^{K-1-\ell}J_pB_J^\ell,
\end{equation}
establishing \eqref{eq:defect-endpoint}. When \(K=1\), the centered endpoint
difference is \(\eta J_p(\theta_0-\theta^*)\), which yields
\eqref{eq:defect-error}.

For general \(K\), set
\(D=B_{\rm ind}^K-B_J^K\). Since \(D\) is symmetric and
\(M_0\succeq\lambda_0I\),
\begin{equation}
\sqrt{\mathcal E_{\rm ind}}
=\|DM_0^{1/2}\|_F
\ge\sqrt{\lambda_0}\|D\|_F.
\end{equation}
Moreover,
\begin{equation}
B_{\rm ind}=\sum_i\alpha_i(I-\eta H_i)^p\succeq qI,
\qquad B_J=(I-\eta\bar H)^p\succeq qI,
\quad q=(1-\eta L)^p.
\end{equation}
Lemma \ref{lem:power-lower}, together with
\(B_{\rm ind}-B_J=\eta J_p\), implies
\begin{equation}
\|D\|_F\ge K\eta q^{K-1}\|J_p\|_F.
\end{equation}
Finally, the reverse triangle inequality combined with
\eqref{eq:defect-remainder} gives
\begin{equation}
\|J_p\|_F\ge\eta\|V_H\|_F
\left[\binom p2-\eta Lc_p(\eta L)\sqrt{r_H}\right]_+,
\end{equation}
proving \eqref{eq:defect-lower}.
\end{proof}

Figure \ref{fig:separation-scan} visualizes the sufficient-separation bracket
\(s_{p,r}(x)=\binom p2-xc_p(x)\sqrt r\). Positivity certifies a nonzero lower
bound, whereas nonpositivity only means that this sufficient condition is uninformative. The exact \(p=2\) case is omitted because \(c_2\equiv0\) and the
bracket is one.
\begin{figure}[t]
\centering
\includegraphics[width=0.96\linewidth]{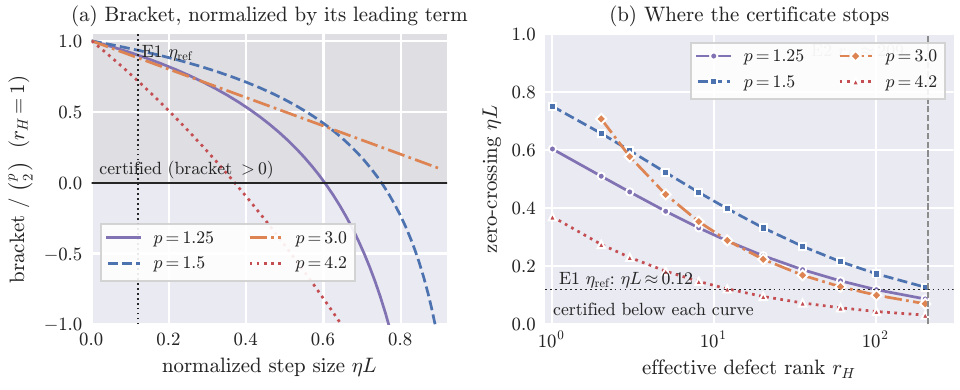}
\caption{Where the arbitrary-ratio lower bound is certified.
The zero-crossing of Theorem \ref{thm:defect}'s sufficient-separation
bracket as a function of effective defect rank, for several compression
ratios; the bound is certified below each curve, and larger \(p\), step
size, or rank can make it uninformative even though the exact composition gap
remains nonnegative. The dotted line marks
\(\eta L\approx0.12\), the value of the ReLU 2{,}000-example subset
diagnostic at \(\eta_{\rm ref}\) (the calibration cells of Figure
\ref{fig:e1-full} lie at \(\eta L=0.04\)--\(0.11\)); the
independent-versus-joint comparison operates at \(\eta L\approx1\),
outside the range, where only the exact \(p=2\) identity is invoked.}
\label{fig:separation-scan}
\end{figure}

\subsection{Proof of the flow-limit boundary}
\label{app:boundaries}

\begin{proposition}[Flow-limit boundary]
\label{prop:flow-limit}
Fix \(p>1\), bounded symmetric source Hessians, and a common optimum. For
exact quadratic constituents, the RMS composition gap is \(O(\eta^2)\) for
fixed \(K\), and \(O(\eta)\) under \(K\eta\to\tau/p\). Hence the finite-step
Jensen mechanism vanishes as \(\eta\to0\).
\end{proposition}

\begin{proof}[Proof of Proposition \ref{prop:flow-limit}]
For bounded symmetric \(H\) and fixed real \(p>1\), the generalized binomial
series is uniform for sufficiently small \(\eta\) and gives
\begin{equation}
f_p(H)=pH-\binom p2\eta H^2+O(\eta^2).
\end{equation}
Consequently,
\begin{equation}
J_p=f_p(\bar H)-\sum_i\alpha_i f_p(H_i)
=\binom p2\eta V_H+O(\eta^2),
\end{equation}
so \(\|J_p\|_2=O(\eta)\). Under a common optimum, the endpoint identity in
Theorem \ref{thm:defect} and bounded contraction factors imply
\begin{equation}
\mathcal C_{\rm struct}
\le \eta K\|J_p\|_2\sqrt{\operatorname{tr}(M_0)}.
\end{equation}
This is \(O(\eta^2)\) for fixed \(K\). If \(K\eta\to\tau/p\), the prefactor
\(\eta K\) remains bounded and the same expression is \(O(\eta)\). In both
regimes the finite-step composition mechanism vanishes.
\end{proof}

\subsection{Different optima via affine augmentation}

For source \(i\), introduce the homogeneous-coordinate one-step map and its
generator
\begin{equation}
\mathsf A_i=
\begin{bmatrix}I-\eta H_i&\eta b_i\\0&1\end{bmatrix}
=I_{d+1}+\eta\mathsf G_i,
\qquad
\mathsf G_i=
\begin{bmatrix}-H_i&b_i\\0&0\end{bmatrix}.
\end{equation}
Write \(\bar{\mathsf A}=\sum_i\alpha_i\mathsf A_i\) and
\(\bar{\mathsf G}=\sum_i\alpha_i\mathsf G_i\).

\begin{proposition}[Exact affine arbitrary-ratio law]
\label{prop:affine-ratio}
Let \(H_i\succeq0\), \(b_i\in\operatorname{range}(H_i)\), and
\(\eta\|H_i\|_2<1\). For any real \(p>1\), exact source-wise time compression replaces
\(\mathsf A_i\) by its principal power \(\mathsf A_i^p\). Let
\begin{equation}
\mathsf A_{\rm ind}=\sum_i\alpha_i\mathsf A_i^p,
\qquad
\mathsf A_J=\left(\sum_i\alpha_i\mathsf A_i\right)^p,
\qquad
\mathsf D_p=\mathsf A_{\rm ind}-\mathsf A_J.
\end{equation}
Then, without common optima or commutativity,
\begin{align}
\mathsf D_p
&=\sum_{k=2}^{\infty}\binom{p}{k}\eta^k
\left(\sum_i\alpha_i\mathsf G_i^k-\bar{\mathsf G}^{k}\right),
\label{eq:affine-ratio}\\
\mathsf A_{\rm ind}^{K}-\mathsf A_J^{K}
&=\sum_{\ell=0}^{K-1}\mathsf A_{\rm ind}^{K-1-\ell}
\mathsf D_p\mathsf A_J^\ell.
\label{eq:affine-endpoint}
\end{align}
For \(p=2\), the augmented defect matrix is exactly
\begin{equation}
\mathsf D_2
=\eta^2
\begin{bmatrix}V_H&-w\\0&0\end{bmatrix},
\end{equation}
recovering Theorem \ref{thm:structural}.
\end{proposition}

\begin{proof}
Because \(b_i\in\operatorname{range}(H_i)\), take
\(u_i=H_i^\dagger b_i\). Translation by \(u_i\) makes \(\mathsf A_i\)
similar to \(\operatorname{diag}(I-\eta H_i,1)\), so its principal real power
exists and is precisely the affine quadratic map in Lemma
\ref{lem:realizability}. The same argument applies to the weighted union
since \(\bar b\in\operatorname{range}(\bar H)\).
Weighted union averages the affine statistics, hence it averages the
one-step homogeneous-coordinate matrices. Compressing time before
mixing gives \(\sum_i\alpha_i\mathsf A_i^p\), whereas mixing before applying
the same principal power gives \(\bar{\mathsf A}^{p}\). Expanding
\((I+\eta\mathsf G_i)^p\) by the generalized binomial series makes the
constant and linear terms cancel. Convergence is blockwise: by induction on
\(k\),
\[
\mathsf G_i^k=\begin{bmatrix}(-H_i)^k & (-H_i)^{k-1}b_i\\ 0 & 0\end{bmatrix},
\qquad k\ge1,
\]
so \(\|\eta^k\mathsf G_i^k\|\le(\eta L)^k+\eta\|b_i\|(\eta L)^{k-1}\), and
with \(\eta L<1\) the generalized binomial series
\(\sum_k\binom pk\eta^k\mathsf G_i^k\) converges absolutely in every block.
This proves \eqref{eq:affine-ratio}.
The identity \eqref{eq:affine-endpoint} is the
difference-of-powers identity. For \(p=2\), direct block multiplication gives
\begin{equation}
\sum_i\alpha_i\mathsf G_i^2-\bar{\mathsf G}^{2}
=\begin{bmatrix}
\sum_i\alpha_iH_i^2-\bar H^2&
-\sum_i\alpha_iH_ib_i+\bar H\bar b\\0&0
\end{bmatrix},
\end{equation}
whose top blocks are \(V_H\) and \(-w\).
\end{proof}

The matrices \(\mathsf A_i\) are generally non-symmetric. Consequently,
Proposition \ref{prop:affine-ratio} supplies an exact identity for different
source optima, but Lemma \ref{lem:power-lower} does not furnish an
arbitrary-\(K\) lower bound in this augmented space.

\subsection{Commuting arbitrary-ratio lower bound}

\begin{theorem}[Multi-source finite-step Jensen gap]
\label{thm:general}
Let \(H_1,\ldots,H_m\) be pairwise commuting positive-semidefinite Hessians
that share an optimum. Define
\(H_{\rm ind}=\sum_i\alpha_if_p(H_i)\) and
\(H_{\rm joint}=f_p(\sum_i\alpha_iH_i)\). Then
\(H_{\rm ind}\preceq H_{\rm joint}\), strictly on every heterogeneous common
eigendirection. Let \(\bar h_j=\sum_i\alpha_ih_{ij}\),
\(v_j=\sum_i\alpha_i(h_{ij}-\bar h_j)^2\), and
\(m_f=\min_{x\in[0,L]}p(p-1)\eta(1-\eta x)^{p-2}\). If \(U\) is the common
eigenbasis and \(\sigma_j^2=(U^\top M_0U)_{jj}\), then
\begin{equation}
\mathcal E_{\rm ind}\ge
\sum_j\sigma_j^2\left[
\frac{K\eta m_f}{2}(1-\eta\bar h_j)^{p(K-1)}v_j
\right]^2.
\label{eq:general-lower}
\end{equation}
\end{theorem}

Since all source Hessians commute, they admit a common eigenbasis. On a common
mode, define the scalar function
\begin{equation}
f_p(x)=\frac{1-(1-\eta x)^p}{\eta}.
\end{equation}
Its second derivative is
\begin{equation}
f_p''(x)=-p(p-1)\eta(1-\eta x)^{p-2}\le-m_f<0.
\end{equation}

\begin{lemma}[Strong-concavity Jensen remainder]
For masses \(\alpha_i\ge0\) with \(\sum_i\alpha_i=1\), values \(x_i\in[0,L]\), mean
\(\bar x=\sum_i\alpha_ix_i\), and variance
\(v=\sum_i\alpha_i(x_i-\bar x)^2\),
\begin{equation}
f_p(\bar x)-\sum_i\alpha_if_p(x_i)\ge\frac{m_f}{2}v.
\end{equation}
\end{lemma}

\begin{proof}
The function \(q(x)=f_p(x)+\frac{m_f}{2}x^2\) is concave. Jensen's inequality
for \(q\), followed by expansion of the quadratic term, proves the result.
\end{proof}

\begin{proof}[Proof of Theorem \ref{thm:general}]
Applying scalar Jensen mode by mode gives
\(H_{\rm ind}\preceq H_{\rm joint}\) and the curvature difference
\begin{equation}
\delta_j:=f_p(\bar h_j)-\sum_i\alpha_if_p(h_{ij})
\ge\frac{m_f}{2}v_j.
\end{equation}
Let \(q_j=1-\eta f_p(\bar h_j)=(1-\eta\bar h_j)^p\). The independent
one-step contraction on mode \(j\) is \(q_j+\eta\delta_j\), whereas the joint
contraction is \(q_j\). Therefore
\begin{align}
(q_j+\eta\delta_j)^K-q_j^K
&\ge Kq_j^{K-1}\eta\delta_j\\
&\ge\frac{K\eta m_f}{2}
(1-\eta\bar h_j)^{p(K-1)}v_j.
\end{align}
Squaring this difference, weighting by the initialization mode variance
\(\sigma_j^2\), and summing over modes gives
\eqref{eq:general-lower}.
\end{proof}

The arbitrary-ratio signed theorem uses simultaneous diagonalization. The
arbitrary-ratio decomposition in Theorem \ref{thm:defect} does not, but its
higher-order remainder need not have a fixed Loewner sign. Extending the
global ordering to every real \(p>1\) and noncommuting sources would require
an operator-concavity statement that is not assumed here. This does not weaken
the Frobenius lower bound in \eqref{eq:defect-lower}: strong monotonicity of
positive matrix powers prevents endpoint cancellation without asserting a
Loewner ordering for \(J_p\).

\section{Approximate Distillation and Identifiability}
\label{app:stability}

\subsection{Difference-of-powers bound}

For square matrices \(X,Y\),
\begin{equation}
X^K-Y^K=\sum_{\ell=0}^{K-1}X^{K-1-\ell}(X-Y)Y^\ell.
\end{equation}
If \(\|X\|_2,\|Y\|_2\le\rho\), then
\begin{equation}
\|X^K-Y^K\|_2\le K\rho^{K-1}\|X-Y\|_2.
\label{eq:power-bound}
\end{equation}
This identity does not require commutativity.

\subsection{Proof of affine-statistic stability}

Define
\begin{align}
L_K&:=K\eta\|\Delta H\|_2,\\
C_K&:=K\eta\|\Delta b\|_2+
\frac{\eta^2K(K-1)}{2}B_b\|\Delta H\|_2,\\
\Delta_{\rm aff}&:=\left[L_K^2\operatorname{tr}(\Sigma_0)+
(L_K\|m_0\|_2+C_K)^2\right]^{1/2}.
\end{align}

\begin{theorem}[Affine-statistic stability]
\label{thm:stability}
Let \(\mathcal C_{\rm struct}\) be the RMS gap between the ideal composed
map and the real-union map, and let
\(\widehat{\mathcal C}_{\rm comp}\) be the corresponding RMS error of the
actual composed map. If the ideal and actual one-step affine contractions have
operator norm at most one and \(\|b^\circ\|_2\le B_b\), then
\[
\left|\widehat{\mathcal C}_{\rm comp}-\mathcal C_{\rm struct}\right|
\le \Delta_{\rm aff}.
\]
For contractions bounded by \(\rho\), the same conclusion holds with
\(L_K(\rho)\) and \(C_K(\rho)\) defined below.
\end{theorem}

\begin{proof}[Proof of Theorem \ref{thm:stability}]
The \(K\)-step affine map associated with \((H,b)\) is
\begin{equation}
\TrainMap_{H,b}^{K,\eta}(\theta_0)
=(I-\eta H)^K\theta_0+
\eta\sum_{t=0}^{K-1}(I-\eta H)^tb.
\end{equation}
\Eqref{eq:power-bound} gives the linear-map difference
\begin{equation}
\|\widehat B^K-(B^\circ)^K\|_2
\le K\eta\|\Delta H\|_2=L_K.
\end{equation}
For the affine intercept, telescope each power to obtain
\begin{align}
&\left\|\eta\sum_{t=0}^{K-1}\widehat B^t\widehat b-
\eta\sum_{t=0}^{K-1}(B^\circ)^tb^\circ\right\|_2\\
&\quad\le K\eta\|\Delta b\|_2+
\frac{\eta^2K(K-1)}{2}B_b\|\Delta H\|_2=C_K.
\end{align}
Writing \(\theta_0=m_0+u_0\) makes the centered cross term vanish, so
\begin{equation}
\E\|\widehat\TrainMap(\theta_0)-\TrainMap^\circ(\theta_0)\|_2^2
\le L_K^2\operatorname{tr}(\Sigma_0)+
(L_K\|m_0\|_2+C_K)^2.
\end{equation}
The reverse triangle inequality in \(L_2(P_0)\) proves
the theorem's bound.
\end{proof}

For contractions bounded by a general \(\rho\), replace the simplified
constants by
\begin{align}
L_K(\rho)&=K\eta\rho^{K-1}\|\Delta H\|_2,\\
C_K(\rho)&=\eta S_K(\rho)\|\Delta b\|_2+
\eta^2T_K(\rho)B_b\|\Delta H\|_2,\\
S_K(\rho)&=\sum_{t=0}^{K-1}\rho^t,\qquad
T_K(\rho)=\sum_{t=1}^{K-1}t\rho^{t-1}.
\end{align}

\subsection{From endpoint error to statistic error}

Let the ideal and actual one-step contraction matrices for source \(i\) be
\(Y_i=I-\eta H_i^\circ\) and \(X_i=I-\eta\widehat H_i\), and define
\begin{equation}
R_i=X_i^K-Y_i^K,\qquad
\epsilon_i=\operatorname{tr}(R_i^2M_0).
\end{equation}
Assume \(M_0\succeq\lambda_0I\), and that \(X_i,Y_i\) are symmetric positive
definite with spectra in \([q_{\min},\rho]\), where \(q_{\min}>0\). Lemma
\ref{lem:power-lower}, applied directly to \(X_i,Y_i\), yields
\begin{equation}
\|\widehat H_i-H_i^\circ\|_F
\le\frac{\|R_i\|_F}{K\eta q_{\min}^{K-1}}
\le\frac{\sqrt{\epsilon_i}}
{K\eta q_{\min}^{K-1}\sqrt{\lambda_0}}.
\end{equation}
Thus endpoint fidelity controls curvature statistics only under full-rank
initialization coverage and a fixed positive dynamics branch.

\begin{corollary}[Stability from identifiable endpoint fidelity]
\label{cor:endpoint-stability}
Suppose \(M_0\succeq\lambda_0I\), every ideal and actual one-step contraction
has spectrum in \([q_{\min},\rho]\), and the ideal and actual source statistics
share a fixed point \(\theta^*\). If the squared source endpoint errors are
\(\epsilon_i\), then
\begin{equation}
\left|\widehat{\mathcal C}_{\rm comp}-\mathcal C_{\rm struct}\right|
\le
\sqrt{\frac{\operatorname{tr}(M_0)}{\lambda_0}}
\left(\frac{\rho}{q_{\min}}\right)^{K-1}
\sum_i\alpha_i\sqrt{\epsilon_i}.
\label{eq:endpoint-stability}
\end{equation}
\end{corollary}

\begin{proof}[Proof of Corollary \ref{cor:endpoint-stability}]
Because both ideal and actual maps fix \(\theta^*\), their source endpoint
difference is \(R_i(\theta_0-\theta^*)\). The preceding display bounds each
one-step contraction difference. After weighted union,
\begin{equation}
\|\widehat B-B^\circ\|_F
\le\sum_i\alpha_i\|X_i-Y_i\|_F.
\end{equation}
Applying \eqref{eq:power-bound} and
\(\|AM_0^{1/2}\|_F\le\|A\|_2\sqrt{\operatorname{tr}(M_0)}\) yields
\begin{equation}
\mathcal C_{P_0}(\widehat B^K,(B^\circ)^K)
\le\sqrt{\frac{\operatorname{tr}(M_0)}{\lambda_0}}
\left(\frac{\rho}{q_{\min}}\right)^{K-1}
\sum_i\alpha_i\sqrt{\epsilon_i}.
\end{equation}
The reverse triangle inequality in \(L_2(P_0)\) proves
\eqref{eq:endpoint-stability}.
\end{proof}

\subsection{Proof of the sufficient detectability condition}

\begin{proof}[Proof of Corollary \ref{cor:detectability}]
Write \(E_i^H=\widehat H_i-H_i^\circ\). For \(K=1\), the source endpoint
difference around the common fixed point is
\(-\eta E_i^H(\theta_0-\theta^*)\), and therefore
\begin{equation}
\sqrt{\epsilon_i}=\eta\|E_i^HM_0^{1/2}\|_F.
\end{equation}
The independently composed statistic perturbation is
\(\sum_i\alpha_iE_i^H\). Hence
\begin{equation}
\mathcal C_{P_0}(\widehat\TrainMap_{\rm ind},\TrainMap_{\rm ind}^\circ)
=\eta\left\|\sum_i\alpha_iE_i^HM_0^{1/2}\right\|_F
\le\sum_i\alpha_i\sqrt{\epsilon_i}=\mathcal U_{\rm src}.
\end{equation}
The reverse triangle inequality against the real-union map proves
\eqref{eq:detectability-bound}. When
\(\mathcal C_{\rm struct}>\mathcal U_{\rm src}\), the same reverse triangle
inequality gives the displayed positive lower bound. Failure of this sufficient
condition does not imply non-detectability, and Proposition
\ref{prop:no-individual-control} rules out a converse based only on the
\(\epsilon_i\).
\end{proof}

\subsection{Why individual errors alone are insufficient}

\begin{proposition}[No control from individual error alone]
\label{prop:no-individual-control}
No function \(F:\mathbb R_+^m\to\mathbb R_+\) with \(F(0)=0\) can upper-bound
all composition errors using only constituent endpoint errors. Exact scalar
constituents can have unbounded composition error as their optima separate.
\end{proposition}

\begin{proof}[Proof of Proposition \ref{prop:no-individual-control}]
There is no function \(F:\R_+^m\to\R_+\), with
\(F(0,\ldots,0)=0\), that upper-bounds every composition error using only
individual endpoint errors. Take the one-dimensional construction with
\(\eta=1/2\), \(T=2\), \(K=1\), curvatures \(h_1=1/2\), \(h_2=3/2\), equal
masses, and minimizers \(u_1=a\), \(u_2=-a\). Individually exact synthetic
curvatures are \(7/8\) and \(15/8\), yet
\begin{equation}
\epsilon_1=\epsilon_2=0,
\qquad
\mathcal E_{\rm ind}=\frac{\sigma_0^2}{256}+\frac{a^2}{64}.
\end{equation}
Here \(\E[\theta_0]=0\). For a general initialization mean \(\mu_0\), the
right-hand side is
\(\sigma_0^2/256+(\mu_0/16+a/8)^2\).
The first term gives a positive structural gap at \(a=0\), and the second
makes the error unbounded while all constituent errors remain zero.
\end{proof}

\section{Local Nonlinear Remainder}
\label{app:nonlinear}

Let
\begin{equation}
g=\nabla F(\theta^\circ),\qquad H=\nabla^2F(\theta^\circ),
\end{equation}
and
\begin{equation}
Q(\theta)=F(\theta^\circ)+g^\top(\theta-\theta^\circ)+
\frac12(\theta-\theta^\circ)^\top H(\theta-\theta^\circ).
\end{equation}

\begin{theorem}[Hessian-Lipschitz local remainder]
\label{thm:nonlinear}
Suppose \(\nabla^2F\) is \(\nu\)-Lipschitz in operator norm on a
radius-\(r\) ball, the nonlinear \(K\)-step gradient-descent trajectory
remains in that ball, and \(\|I-\eta H\|_2\le\rho\). Then
\[
\|\TrainMap_F^{K,\eta}(\theta_0)-
\TrainMap_Q^{K,\eta}(\theta_0)\|_2
\le\frac{\eta\nu r^2}{2}\sum_{t=0}^{K-1}\rho^t.
\]
Consequently, the nonlinear real-versus-synthetic RMS gap differs from its
frozen-quadratic counterpart by at most the sum of the two remainders.
\end{theorem}

Throughout, the expansion point is the trajectory start,
\(\theta^\circ=\theta_0\). With gradients bounded by \(G\) along the path,
the \(K\)-step displacement satisfies \(r\le K\eta G=O(\eta)\), so the
remainder is \(O(\eta^3)\), which is the scaling assumed in Corollary
\ref{cor:eta-window}. For \(K=1\) the synthetic remainder vanishes
identically, since the one-step update on the synthetic loss equals the
one-step update on its quadratic expansion at \(\theta_0\), so
\(\delta_S=0\) and the uncertainty in Corollary \ref{cor:proxy-gap}
reduces to \(\Delta_{\rm aff}+\delta_R\).

\begin{proof}[Proof of Theorem \ref{thm:nonlinear}]
Hessian Lipschitzness and Taylor's theorem imply, throughout the radius-\(r\)
ball,
\begin{equation}
\|\nabla F(\theta)-g-H(\theta-\theta^\circ)\|_2
\le\frac{\nu}{2}\|\theta-\theta^\circ\|_2^2
\le\frac{\nu r^2}{2}.
\end{equation}
Let \(e_t\) be the difference between nonlinear and frozen-quadratic iterates
initialized at the same point. Their recursions give
\begin{equation}
\|e_{t+1}\|_2\le\rho\|e_t\|_2+\frac{\eta\nu r^2}{2},
\qquad e_0=0.
\end{equation}
Unrolling gives \eqref{eq:nonlinear-remainder}. For real and synthetic
maps, the reverse triangle inequality in \(L_2(P_0)\) and the ordinary
triangle inequality bound the difference between their nonlinear and
quadratic RMS gaps by \(\delta_R+\delta_S\).
\end{proof}

\subsection{Proof of the finite-step detectability window}
\begin{figure}[h]
\centering
\includegraphics[width=0.72\linewidth]{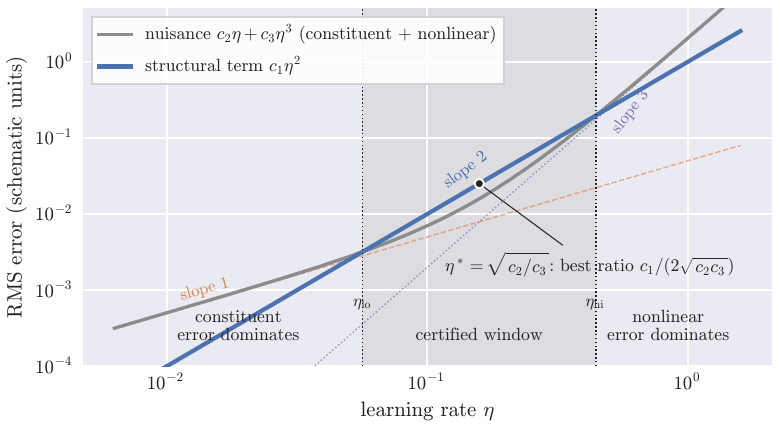}
\caption{The detectability window (schematic). The structural term
grows as \(\eta^2\), constituent error as \(\eta\), nonlinear error as
\(\eta^3\); the structural term is certifiably visible only where it exceeds
their sum, the shaded interval of Corollary \ref{cor:eta-window}.}
\label{fig:window}
\end{figure}

\begin{corollary}[Finite-step detectability window]
\label{cor:eta-window}
Suppose on a valid local interval that the observed RMS error satisfies
\(|\mathcal C_{\rm obs}-\mathcal C_{\rm struct}|\le\mathcal U_{\rm src}+\delta_S+\delta_R\),
as in Corollary \ref{cor:detectability} extended by the nonlinear
remainders of Theorem \ref{thm:nonlinear}, and that
\(\mathcal C_{\rm struct}=c_1\eta^2\),
\(\mathcal U_{\rm src}\le c_2\eta\), and
\(\delta_S+\delta_R\le c_3\eta^3\), where \(c_1,c_2,c_3>0\). Then
\(\mathfrak D_{\rm loc}(\eta)=c_1\eta/(c_2+c_3\eta^2)\) is maximized at
\(\eta^*=\sqrt{c_2/c_3}\), with value \(c_1/(2\sqrt{c_2c_3})\). A nonempty
certified interval exists iff \(c_1^2>4c_2c_3\), before intersection with the
stability interval.
\end{corollary}

\begin{proof}[Proof of Corollary \ref{cor:eta-window}]
Under the assumed bound, a positive structural signal is certified whenever
\(c_1\eta^2>c_2\eta+c_3\eta^3\) (Corollary \ref{cor:proxy-gap} gives the
same form with \(\Delta_{\rm aff}\) in place of \(\mathcal U_{\rm src}\)). Dividing by the positive factor \(\eta\)
gives \(c_3\eta^2-c_1\eta+c_2<0\), which has a nonempty solution interval iff
its discriminant is positive. The two roots are
\begin{equation}
\eta_\pm=\frac{c_1\pm\sqrt{c_1^2-4c_2c_3}}{2c_3},
\end{equation}
and the certified interval is \((\eta_-,\eta_+)\). Direct
differentiation of \(c_1\eta/(c_2+c_3\eta^2)\) gives the unique maximizer
\(\eta^*=\sqrt{c_2/c_3}\) and the stated maximum. The resulting interval must
still be intersected with the positive-branch and local-trajectory conditions.
\end{proof}

The theorem requires the trajectories to stay in the local ball. This
condition can be verified retrospectively for the local dynamics test, or enforced by
restricting the step budget and initialization radius. The theorem does not
cover a trajectory that crosses multiple curvature regimes.

\paragraph{Limits of the mechanism.}
\label{app:limits}
The mechanism vanishes in the flow limit (Proposition \ref{prop:flow-limit}),
can be masked when constituent error exceeds it (Corollary \ref{cor:detectability} then gives no conclusion), and is local, since
\eqref{eq:nonlinear-remainder} needs a slowly varying Hessian (Example
\ref{ex:network}). With terms scaling as \(c_1\eta^2,c_2\eta,c_3\eta^3\), a
certified interval in \(\eta\) exists iff \(c_1^2>4c_2c_3\) (Corollary
\ref{cor:eta-window}, Appendix Figure \ref{fig:window}). Which components a
scalar retiming can repair is treated in Appendix \ref{app:repair}.

\section{Calibration and Components That Scalar Retiming Cannot Repair}
\label{app:repair}

Propositions \ref{prop:clock} and \ref{prop:rank-barrier} below are
stated in the continuous gradient-flow limit, where retiming acts through
matrix exponentials. The rank bound of Proposition \ref{prop:rank-barrier}
carries over to discrete steps by the same argument with \((I-\eta H)^T\)
in place of \(e^{-TH}\), and its discrete form is given at the end of
Appendix \ref{app:rank-budget}. The exact scalar calibration of
Proposition \ref{prop:clock} is specific to the flow limit: in discrete
time a single step count generally cannot compensate a curvature scaling on
every mode at once, as equation \ref{eq:mode-retiming} shows.

\subsection{Mode-specific retiming}
\label{app:mode-retiming}

For commuting real and synthetic Hessians, let \(\lambda_j\) and \(\mu_j\)
denote the real and synthetic curvatures on common mode \(j\). When
\(0<1-\eta\lambda_j<1\) and \(0<1-\eta\mu_j<1\), the continuous effective
synthetic step count that exactly repairs that mode is
\begin{equation}
K_j^*=T\frac{\log(1-\eta\lambda_j)}
{\log(1-\eta\mu_j)}.
\label{eq:mode-retiming}
\end{equation}
A single scalar retiming repairs all task-relevant modes only if their
\(K_j^*\) values coincide. If the protocol requires an integer number of GD
steps, the common value must additionally be an integer or be approximated.

\subsection{Uniform source clocks}

\begin{proposition}[Exact calibration of uniform source clocks]
\label{prop:clock}
Let \(Z=\sum_i\alpha_i/c_i\). Choosing
\(\beta_i=(\alpha_i/c_i)/Z\) and \(\tau=TZ\) gives
\(e^{-\tau\sum_i\beta_ic_iH_i}=e^{-T\sum_i\alpha_iH_i}\). If the \(H_i\)
are linearly independent, exactness for fixed \(\beta_i\) requires
\(\beta_ic_i/\alpha_i\) to be constant over sources.
\end{proposition}

\begin{proof}
For the stated weights,
\(H_S=\sum_i\beta_ic_iH_i=Z^{-1}\sum_i\alpha_iH_i\). Hence
\(\tau H_S=T\sum_i\alpha_iH_i\), and the two matrix exponentials agree. If
the \(H_i\) are linearly independent, equality of their generators requires
\(\tau\beta_ic_i=T\alpha_i\) for every source. Thus
\(\beta_ic_i/\alpha_i\) is constant across sources.
\end{proof}

\subsection{Repairability taxonomy}

The formal clock and rank results separate defects that admit a scalar
calibration from those that require a mode-dependent or higher-rank repair
(Table \ref{tab:taxonomy}). Only the semantic-mass and clock interventions
are tested empirically in this paper. The other entries summarize
theoretical consequences; a nonzero mean affine discrepancy is not by
itself a criterion of non-repairability, since a common optimum with
\(\mu_0\ne\theta^*\) gives \(V_H\mu_0-w\ne0\) while a single scalar
still repairs the map whenever the linear part and the affine offset are
matched by the same factor.

\begin{table}[h]
\caption{Theory-derived repairability taxonomy.}
\label{tab:taxonomy}
\centering
\small
\begin{tabular}{@{}lccp{4.6cm}@{}}
\toprule
Component & Signature & Diagnostic & Minimum repair required\\
\midrule
Global clock & common ratio & endpoint speed & scalar retiming\\
Source clock & coefficients \(c_i\) & source HVP scale & per-source reweighting\\
Curvature variance & \(V_H\ne0\) & \(P_{\rm comp}\) & mode-dependent if the required retimings differ (\ref{app:mode-retiming})\\
Affine conflict & \(V_H\mu_0-w\ne0\) & mean shift & scalar retiming only if it matches the linear part and the affine offset simultaneously (the scalar example of Section \ref{sec:law} is repaired by the factor \(3.6/3.5\)); otherwise mode-dependent\\
Missing subspace & rank deficit & persistent residual & not repairable by retiming\\
\bottomrule
\end{tabular}
\end{table}

\subsection{Rank-budget proof}
\label{app:rank-budget}

\Needspace{6\baselineskip}
\begin{proposition}[Rank-budget non-repairability]
\label{prop:rank-barrier}
Let \(s_1\ge\cdots\ge s_d\) be singular values of
\(B=(I-e^{-TH_R})M_0^{1/2}\). Every rank-\(r\) synthetic Hessian and scalar
retiming \(\tau\) obey
\begin{equation}
\mathcal E(\tau)\ge\sum_{j=r+1}^d s_j(B)^2.
\end{equation}
\end{proposition}

\begin{proof}[Proof of Proposition \ref{prop:rank-barrier}]
Let \(Q\) be the orthogonal projector onto \(\ker(H_S)\). Since
\(\operatorname{rank}(H_S)\le r\), \(\operatorname{rank}(Q)\ge d-r\), and
\begin{equation}
Qe^{-\tau H_S}=Q.
\end{equation}
Projecting the endpoint difference gives
\begin{equation}
Q(e^{-\tau H_S}-e^{-TH_R})=Q(I-e^{-TH_R}).
\end{equation}
The norm of an orthogonal projection lower-bounds the full endpoint error, so
\begin{equation}
\mathcal E(\tau)\ge\|QB\|_F^2,
\qquad B=(I-e^{-TH_R})M_0^{1/2}.
\end{equation}
Among all projectors of rank at least \(d-r\), the Ky Fan minimum principle
gives
\begin{equation}
\|QB\|_F^2=\operatorname{tr}(QBB^\top)
\ge\sum_{j=r+1}^{d}s_j(B)^2,
\end{equation}
which proves Proposition \ref{prop:rank-barrier}. A scalar squared-loss example contributes
a rank-one positive-semidefinite outer product to its Hessian, proving the
sample-budget interpretation. In discrete time the projector satisfies
\(Q(I-\eta H_S)^K=Q\), so the same argument gives
\(\mathcal E\ge\sum_{j=r+1}^ds_j(B_{\rm disc})^2\) with
\(B_{\rm disc}=(I-(I-\eta H_R)^T)M_0^{1/2}\), which unifies the bound with
the finite-step setting of Theorems \ref{thm:structural} and
\ref{thm:defect}.
\end{proof}

\subsection{Proof of the observable-gap certificate}

\begin{corollary}[Two-to-one observable-gap certificate]
\label{cor:proxy-gap}
Assume \(T=2\) for each real source and union map, \(K=1\) for synthetic maps,
and the hypotheses of Theorems \ref{thm:stability} and
\ref{thm:nonlinear}. With \(P_{\rm comp}\) evaluated at the frozen-model
reference and nonlinear RMS bounds \(\delta_S,\delta_R\),
\begin{equation}
\left|\mathcal C_{\rm obs}-\sqrt{P_{\rm comp}}\right|
\le \zeta,\qquad \zeta:=\Delta_{\rm aff}+\delta_S+\delta_R.
\label{eq:proxy-certificate}
\end{equation}
For partitions \(a,b\), a proxy separation larger than \(\zeta_a+\zeta_b\)
certifies the corresponding observed ordering.
\end{corollary}

\begin{proof}[Proof of Corollary \ref{cor:proxy-gap}]
Theorem \ref{thm:structural} and \eqref{eq:pcomp} give the ideal quadratic RMS
gap \(\sqrt{P_{\rm comp}}\), without a common-optimum assumption.
Theorem \ref{thm:stability} changes this norm by at most
\(\Delta_{\rm aff}\). Replacing the actual and reference quadratic dynamics
by their nonlinear counterparts changes it by at most
\(\delta_S+\delta_R\), by two applications of the triangle inequality.
This proves \eqref{eq:proxy-certificate}. For partitions \(a,b\),
\begin{align}
\mathcal C_{\rm obs}^{(a)}&\ge\sqrt{P_{\rm comp}^{(a)}}-\zeta_a,\\
\mathcal C_{\rm obs}^{(b)}&\le\sqrt{P_{\rm comp}^{(b)}}+\zeta_b.
\end{align}
The displayed separation condition therefore certifies the ordering.
\end{proof}

\section{HVP Estimator and Concentration}
\label{app:hvp}

Under a common optimum, the structural statistic is
\begin{equation}
P_H:=\eta^4\operatorname{tr}(V_H^2M_0)
=\eta^4\E_{z\sim(0,M_0)}\|V_Hz\|_2^2.
\label{eq:predictor}
\end{equation}
For general \((p,K)\), its small-step leading coefficient is
\begin{equation}
P_H^{(p,K)}:=\left[K\binom p2\eta^2\right]^2
\operatorname{tr}(V_H^2M_0).
\label{eq:general-predictor}
\end{equation}
The positive-branch factor in Theorem \ref{thm:defect} calibrates a lower
bound and is not part of this leading coefficient. Two evaluations should be
kept apart. The experiments evaluate \(r_{\rm comp}(\theta_0)=V_H\theta_0-w\)
directly for each initialization draw with one HVP per source from the
source gradients (Appendix \ref{app:e1-protocol}), which includes the affine
term and needs no probes. This appendix concerns the curvature-only trace
\(P_H\), estimated with random probes at \(3m\) HVPs per probe; the
concentration result of Theorem \ref{thm:concentration} applies to that
estimator and not to the sample mean of \(\|r_{\rm comp}(\theta_0)\|_2^2\)
over \(\theta_0\sim P_0\).

For probes \(z_r\), define
\begin{equation}
\widehat P(L)=\frac{\eta^4}{R}\sum_{r=1}^R\|V_Hz_r\|_2^2,
\qquad z_r=Lg_r.
\label{eq:hvp-estimator}
\end{equation}
With \(L=I\), this estimates
\(P_{\rm iso}=\eta^4\operatorname{tr}(V_H^2)\), and shaped probes with
\(LL^\top=M_0\) estimate \eqref{eq:predictor}.

\subsection{HVP-only action}

For a probe \(v\), compute
\begin{equation}
q_i=H_iv,\qquad s_i=H_iq_i,\qquad
\bar q=\sum_i\alpha_iq_i.
\end{equation}
Then
\begin{equation}
V_Hv=\sum_i\alpha_is_i-
\sum_i\alpha_iH_i\bar q.
\end{equation}
This requires \(3m\) HVP calls per probe vector \(v\) for \(m\) sources and no explicit Hessian.

\subsection{Gaussian concentration}

\begin{theorem}[Finite-probe concentration]
\label{thm:concentration}
For independent \(g_r\sim\mathcal N(0,I)\), \(z_r=Lg_r\), and
\(A=L^\top V_H^2L\), the estimator in \eqref{eq:hvp-estimator} is unbiased
and, with probability at least \(1-2e^{-t}\),
\begin{equation}
|\widehat P_H-P_H|\le
2\eta^4\|A\|_F\sqrt{t/R}+2\eta^4\|A\|_2t/R.
\label{eq:concentration}
\end{equation}
For \(M_0=I\), relative error at most \(\varepsilon\) with probability
\(1-\delta\) is ensured by
\begin{equation}
R\ge\max\left\{
\frac{16\log(2/\delta)}{\varepsilon^2r_4(V_H)},
\frac{4\log(2/\delta)}{\varepsilon r_s(V_H)}\right\},
\end{equation}
where \(r_4(V_H)=\operatorname{tr}(V_H^2)^2/\operatorname{tr}(V_H^4)\)
and \(r_s(V_H)=\operatorname{tr}(V_H^2)/\|V_H\|_2^2\).
\end{theorem}

\begin{proof}[Proof of Theorem \ref{thm:concentration}]
Since \(\|V_HLg_r\|_2^2=g_r^\top Ag_r\), unbiasedness follows from
\(\E[g_rg_r^\top]=I\). Diagonalize
\(A=U\operatorname{diag}(\lambda_1,\ldots,\lambda_d)U^\top\). Gaussian
rotational invariance gives
\begin{equation}
\frac1R\sum_{r=1}^Rg_r^\top Ag_r-\operatorname{tr}(A)
=\sum_{j=1}^d\sum_{r=1}^R\frac{\lambda_j}{R}(Z_{rj}^2-1),
\end{equation}
where the \(Z_{rj}\) are independent standard Gaussians. Weighted chi-square
tail bounds give
\begin{align}
\Pr\left[X\ge2\|A\|_F\sqrt{t/R}+2\|A\|_2t/R\right]&\le e^{-t},\\
\Pr\left[X\le-2\|A\|_F\sqrt{t/R}\right]&\le e^{-t}.
\end{align}
A union bound and multiplication by \(\eta^4\) prove \eqref{eq:concentration}. For \(M_0=I\), \(A=V_H^2\), so
\begin{equation}
\|A\|_F^2=\operatorname{tr}(V_H^4),
\qquad \|A\|_2=\|V_H\|_2^2.
\end{equation}
Divide by \(P_H=\eta^4\operatorname{tr}(V_H^2)\) (for \(M_0=I\)), take
\(t=\log(2/\delta)\), and make both relative-error terms at most
\(\varepsilon/2\) to obtain the sample complexity.
\end{proof}

\subsection{Rademacher probes}

The estimator is a Hutchinson-type stochastic trace estimator
\citep{hutchinson1989stochastic}. For independent raw Rademacher probes \(\xi_r\in\{-1,+1\}^d\), let
\(B=A-\operatorname{diag}(A)\). The estimator remains unbiased and has exact
variance
\begin{equation}
\operatorname{Var}(\widehat P_H)=
\frac{2\eta^8}{R}\|B\|_F^2.
\end{equation}
Indeed,
\begin{equation}
\xi^\top A\xi-\operatorname{tr}(A)
=2\sum_{j<k}A_{jk}\xi_j\xi_k,
\end{equation}
and distinct Rademacher monomials are orthogonal in \(L_2\). Chebyshev gives
\begin{equation}
\Pr\left(|\widehat P_H-P_H|\ge\eta^4u\right)
\le\frac{2\|B\|_F^2}{Ru^2}.
\end{equation}
If \(B=0\), every raw Rademacher probe is exact.

\subsection{Implementation safeguards}

For the estimator to match the theory: all \(H_i\) are evaluated at the
same parameter; each \(H_i\) is the Hessian of the complete empirical-mean
source loss (summing batch losses with reduction \texttt{sum}/\(N_i\) is
valid, whereas averaging \(H_b(H_bv)\) over minibatches does not equal
\(H_i^2v\)); the first HVP \(H_iv\) is detached before the second, since
automatic differentiation would otherwise introduce third derivatives;
source masses \(\alpha_i\) reflect the complete real partition even when
an HVP subset is used; raw probes estimate \(\operatorname{tr}(V_H^2)\)
and probes divided by \(\sqrt d\) estimate the trace divided by \(d\);
and probes used to tune a calibration are kept separate from hold-out
probes used to report the predictor.

\section{Additional Experimental Results}
\label{app:experiments}

All experiments were conducted on a single GPU with PyTorch. E9, E10 and E11 were run on an NVIDIA RTX A6000 with PyTorch 1.12.1.

\subsection{Shared protocol}

The experiment codes E1--E15 used in the text and in the supplementary code and result files name the subsections below (E5, quadratic endpoint matching, is in Appendix \ref{app:e5}; E15, the CIFAR-100 replication, is in Appendix \ref{app:svhn}; the numbering has gaps at E4 and E6). We use the CIFAR-10 training split with deterministic tensor conversion
and channel normalization. Every two-source partition is a disjoint cover
of all 50,000 training examples. Heterogeneity is controlled through IID
assignment or class-wise Dirichlet allocation \citep{hsu2019measuring}
with concentration \(1.0\), \(0.5\), or \(0.1\); the \(0.5\) level appears only in the
official-schedule independent-versus-joint sweep. Source masses are
computed from complete source sizes before any HVP subset is drawn. We
store partition seeds, label histograms, source sizes, and hashes of
sorted index sets.

The local evaluation network is a three-block ConvNet with GroupNorm and
average pooling. The local-law validation uses GroupNorm(8); the
independent-versus-joint comparison follows the official implementation
with GroupNorm(4). Unless activation smoothness is itself the
intervention, the network uses ReLU. For each sampled initialization
\(\theta_0\), all source gradients and Hessians are evaluated at that same
\(\theta_0\), and the frozen local model uses \(\theta^\circ=\theta_0\).
``The same checkpoint'' therefore means consistency across sources within
one initialization, not one parameter vector shared across initialization
draws. We use exact loss Hessians rather than a generalized Gauss--Newton
approximation. The main endpoint measurements use two real steps against one
synthetic step, the smallest schedule in which time compression is
non-affine and the defect has an exact closed form with no remainder
(\(R_2=0\) in Theorem \ref{thm:defect}); for larger \(T/K\) or \(K\) the
leading term is the same \(V_H\) up to the coefficient \(K\binom p2\eta^2\)
(Appendix \ref{app:hvp}), so the two-step measurement isolates the
mechanism rather than restricting it. The primary response is the RMS
endpoint-map distance over a fixed initialization distribution, which the
theorems bound and which is measurable without training an evaluation
network. Test accuracy is evaluated separately as a downstream
consequence, because decision margins can hide parameter-map
differences.
\label{app:faq-two-step}

The independent uncertainty unit is a generated source partition.
Initialization seeds are repeated measurements, aggregated within
partition before any between-partition statistic, and confidence
intervals use a bootstrap over partitions. Within a partition,
\(\mathcal C_{\rm ind}\), \(\mathcal C_{\rm joint}\) and
\(\mathcal C_{\rm loc}\) are root-mean-square endpoint distances over
the initialization seeds, as Definition \ref{def:operator-error}
specifies, and \(\widehat\Gamma_{\rm RMS}\) is the difference of the two
RMS values. The no-split control, the budget sweep, the SVHN
replication and the short-schedule tables were aggregated as the mean
over seeds of the per-seed endpoint-distance
difference \(\|\widehat A_{\rm ind}-A_R\|-\|\widehat A_{\rm joint}-A_R\|\),
which we denote \(\widehat\Gamma_{\rm mean}\). On the CIFAR-10 main grid
the change from \(\widehat\Gamma_{\rm mean}\) to \(\widehat\Gamma_{\rm RMS}\)
is below \(3\%\) at the label-skewed levels. For IID partitions, the
absolute change is \(7.8\times10^{-6}\). The endpoint measurements show a
systematic seed effect (a variance decomposition attributes \(75\%\) of
the variation to the partition, \(19\%\) to the seed and \(7\%\) to their
interaction), which is why seeds are aggregated within partition; the
level means change by at most \(7\%\) between three and five seeds, and
the main text reports the three-seed values. The evaluation stage
applies differentiable augmentation without a fixed seed, so repeated
evaluation of the same synthetic set moves test accuracy by up to
\(\pm0.35\) points; accuracies are reported to one decimal.

\subsection{Verification of the local prediction on the network (E1)}
\label{app:e1-protocol}

\paragraph{Configuration.}

The original sweep contains 15 partitions and three initialization seeds at
\(\eta=0.4\times2^{-j}\) for \(j=0,\ldots,9\). An extension adds 15 independent
partitions at the shared points \(0.0125\) and \(0.00078125\). Scaling fits
therefore use the original 15-partition sweep, whereas predictor comparisons
at \(\eta=0.00078125\) use all 30 partitions. This split-input analysis avoids
both discarding independent partitions and imputing missing learning rates.

For each initialization we compute
\begin{equation}
P_{\rm comp}=\eta^4\left\|\sum_i\alpha_iH_i(g_i-\bar g)\right\|_2^2.
\end{equation}
The probe diagnostic \(P_{\rm iso}=\eta^4\operatorname{tr}(V_H^2)\) uses
isotropic HVP probes and is reported separately. We compare both quantities
with label divergence and gradient disagreement (Table
\ref{tab:e1-predictors}). The estimator suite contains
15 checks, including explicit-matrix identities, minibatch invariance,
source-order invariance, singular \(M_0\), Gaussian and Rademacher probes, and
spectral-norm estimation on \(H^2\) for nearly tied positive/negative extreme
eigenvalues. All checks pass.

\begin{table}[h]
\caption{E1: which statistic predicts the observed discrepancy.}
\label{tab:e1-predictors}
\centering
\footnotesize
\begin{tabular}{@{}lccc@{}}
\toprule
Predictor & Spearman \(\rho\) & partial \(\rho\) (adj.\ label div.) & partial \(\rho\) (adj.\ grad.\ dis.)\\
\midrule
\(P_{\rm comp}\) & 0.991 [0.960, 0.999] & 0.869 [0.568, 0.989] & 0.701 [0.210, 0.971]\\
\(P_{\rm iso}\)  & 0.955 [0.857, 0.984] & 0.316 [\(-\)0.316, 0.785] & \(-\)0.001 [\(-\)0.262, 0.496]\\
Label divergence & 0.964 [0.907, 0.985] & N/A & N/A\\
Gradient disagreement & 0.983 [0.937, 0.995] & N/A & N/A\\
\bottomrule
\end{tabular}
\par\smallskip\raggedright\footnotesize
\(n=30\) partitions at \(\eta_{\rm ref}\); rank correlation with
\(\mathcal C_{\rm loc}\); brackets are 95\% bootstrap intervals over
partitions. N/A: a baseline cannot be adjusted for itself.
\end{table}

The high marginal correlations partly reflect the three pre-specified
heterogeneity strata. The partial correlations measure residual rank
association after adjusting for one baseline; they are not a complete
within-level or causal decomposition. The weak adjusted result for
\(P_{\rm iso}\) shows that the isotropic curvature-only diagnostic loses
useful task- or initialization-dependent information, without showing
which affine component dominates.

The calibration cells (\(10^4\)-example sources, float32, three
partitions per cell, one initialization seed) use a single-initialization
form of the statistic. For the network initialization \(\theta_0\), the
endpoint defect of Theorem \ref{thm:structural} is
\(\eta^2 r_{\rm comp}(\theta_0)\), where
\(r_{\rm comp}(\theta_0)=\sum_i\alpha_iH_i(g_i-\bar g)=V_H\theta_0-w\)
(Eq.~\ref{eq:pcomp}) is computed exactly by one Hessian-vector product
per source, without random probes; the affine term is included, because
it enters through the source gradients. The calibration ratio is the
measured discrepancy \(\|\Delta\Phi(\theta_0)\|\) divided by
\(\eta^2\|r_{\rm comp}(\theta_0)\|\). It coincides with
\(\mathcal C_{\rm obs}/\sqrt{P_{\rm comp}}\) when the initialization
distribution is concentrated at \(\theta_0\), and it is averaged over the
three partitions of a cell rather than over \(P_0\). Levels are never
pooled, because the ratio can differ between levels.

\paragraph{Analysis.}
\label{app:e1-detail}

\begin{figure}[h]
\centering
\includegraphics[width=0.96\linewidth]{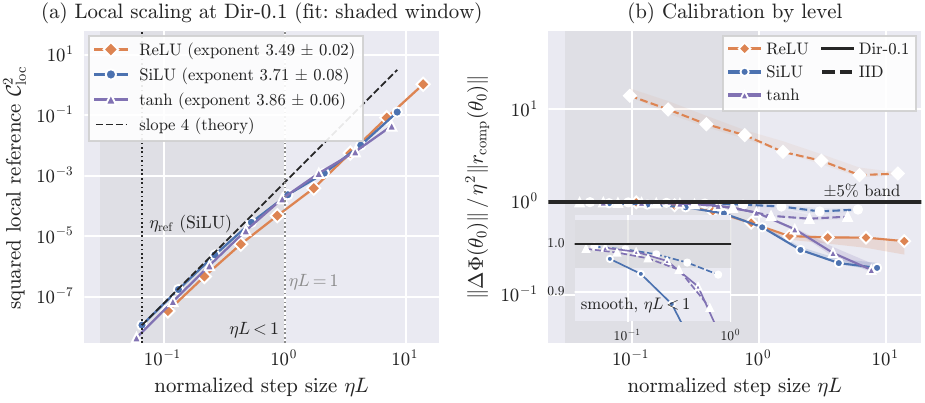}
\caption{Local-law validation with all activations and both levels. (a) Squared discrepancy \(\|\Delta\Phi(\theta_0)\|^2\) against \(\eta L\) at Dir-0.1, with exponents fitted inside the shaded window \(\eta L<1\) and dotted lines at \(\eta_{\rm ref}\) and \(\eta L=1\). (b) Calibration ratio \(\|\Delta\Phi(\theta_0)\|/(\eta^2\|r_{\rm comp}(\theta_0)\|)\) against \(\eta L\) (solid: Dir-0.1, dashed: IID, bands over partitions), with the small-step regime enlarged in the inset. At IID the ReLU structural term is at numerical zero, so its ratio is not meaningful.}
\label{fig:e1-full}
\end{figure}

The local-law validation tests whether the local reference
\(\mathcal C_{\rm loc}\) of Section \ref{sec:experiments}, the RMS
distance between the weighted average of the exact source-wise two-step
endpoints and the real-union two-step endpoint on the actual network
(no distillation involved), follows the quadratic law. Its square
should scale as \(\eta^4\), and its size should match
\(\sqrt{P_{\rm comp}}\) (Eq.~\ref{eq:pcomp}). On the \(10^4\)-example
calibration sources (three partitions per cell, one seed, eight rates from
\(7.8\times10^{-4}\) to \(0.1\), each activation at its own measured
\(L\)), the squared discrepancy inside \(\eta L<1\) scales with exponent
\(3.71\pm0.08\) for SiLU and \(3.86\pm0.06\) for tanh at Dir-0.1, and
\(3.96\pm0.02\) and \(3.93\pm0.04\) under IID (mean \(\pm\) SD of
per-partition fits). ReLU gives \(3.49\pm0.02\) at Dir-0.1 and \(3.01\)
under IID. Over the full sweep the exponents fall to \(3.25\)--\(3.50\) as
the largest rates leave the perturbative window. At \(\eta_{\rm ref}\) the
single-initialization calibration ratios (Appendix \ref{app:e1-protocol})
are \(0.996\)/\(0.969\) for SiLU and \(0.992\)/\(0.991\) for tanh at
IID/Dir-0.1 (SD over partitions at most \(0.004\)), within
approximately \(3\%\) of one, and they drift below one as \(\eta L\) grows
(Figure \ref{fig:e1-full}). The lower ReLU exponent is expected: a
frozen Hessian misses activation crossings, whose gradient discrepancy is
of lower order than the quadratic remainder, so reducing \(\eta\) does not
restore it over the tested step-size range. Its calibration ratio is \(1.00\) at Dir-0.1 but
\(14.9\pm1.7\) at IID, where the structural signal
(\(\mathcal C_{\rm loc}\approx10^{-5}\)) approaches zero faster than the
non-smooth remainder and the ratio is not meaningful; ReLU therefore
serves as a boundary case for the smoothness assumption. The curvature disagreement has
low effective rank (\(r_H=208.5\pm5.2\) in \(320{,}010\) dimensions;
Appendix \ref{app:assumption-support}).

\paragraph{A denser sweep inside the perturbative window.}
On the full sources, a second ReLU sweep uses nine rates from
\(3.9\times10^{-4}\) to \(6.25\times10^{-3}\) in steps of \(\sqrt2\), eight
of which lie in the strict window \(\eta L<1\) for every partition and
level, with five partitions per level and three initializations per
partition (Figure \ref{fig:e1-dense}). Fitting \(\log\mathcal C_{\rm loc}^2\)
against \(\log\eta\) on this fixed subset gives exponents of \(3.878\)
\([3.867,3.890]\) at Dir-1.0 and \(3.754\) \([3.689,3.837]\) at Dir-0.1.
The slope is fitted per initialization, averaged within partition, and
the 95\% interval comes from a bootstrap over the five partitions; on the
common grid of step sizes this equals a single fit pooled over all
initializations. The estimates are within 1\% of the
original three-point estimates, and both intervals exclude four.
The \(3\)--\(6\%\) shortfall is the direction expected when the two
expansion conditions, \(\eta L\ll1\) and a slowly varying Hessian, hold
only marginally (Appendix \ref{app:assumption-support}). The IID exponent
is \(3.00\) \([2.97,3.02]\); its discrepancy is dominated by terms other
than the structural one and is not a test of the \(\eta^4\) law. Dir-0.5
appears only in the distillation grid.

\begin{figure}[h]
\centering
\includegraphics[width=0.96\linewidth]{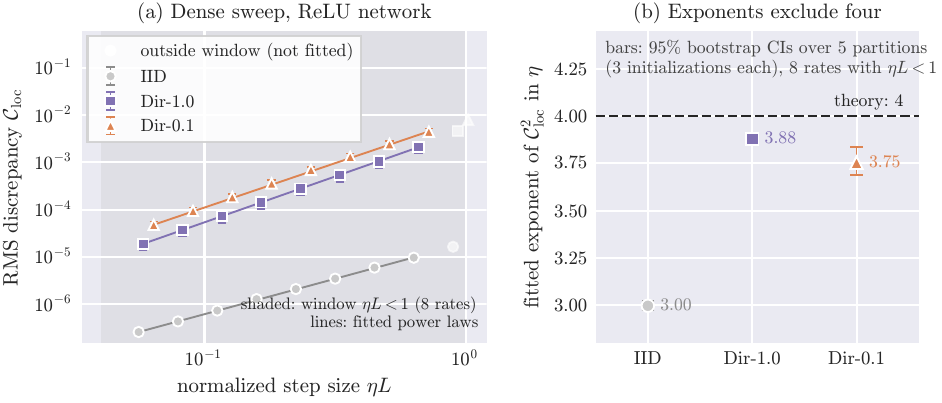}
\caption{Dense E1 sweep on the ReLU network. (a) Partition-level mean RMS
discrepancy against \(\eta L\) at nine step sizes (five partitions per
level with three initializations each, error bars are SDs over
partitions); filled markers are the eight rates inside the strict window
and lines are the fitted power laws. (b) Fitted exponents of
\(\mathcal C_{\rm loc}^2\) in \(\eta\) with 95\% intervals from a bootstrap
over the five partitions; both skewed levels lie 3--6\% below four, and
the IID discrepancy is not governed by the structural term.}
\label{fig:e1-dense}
\end{figure}

\subsection{Endpoint-matching distillates on the network (E10)}
\label{app:e10}

\paragraph{Configuration.}
E10 applies the endpoint-matching objective of E5 to the E1 network, so
that learned synthetic sets, a real network and the local law can be
compared in one setting. The network is the three-block ConvNet with SiLU
activation, the sources are the E1 sources (\(10^4\) examples each) on the
IID and Dir-0.1 partitions (three partitions per level), and the main
setting is \(T=2\), \(K=1\) at \(\eta_{\rm ref}=7.8\times10^{-4}\).
For each source, a synthetic set with one, five or ten images per
represented class (labels fixed, images initialized from random real
images of the class) is optimized to minimize
\(\E_{\theta_0}\|\Phi^{K,\eta}_{S_i}(\theta_0)-\Phi^{T,\eta}_{D_i}(\theta_0)\|^2\)
with Adam (learning rate \(0.03\), chosen once on the IID partition 0 at
five images per class from \(\{0.01,0.03,0.1\}\) and then frozen) for
\(1{,}100\) iterations, the point at which the ten-image pilot reached
within \(3\%\) of its loss after \(3{,}000\) iterations. Under Dir-0.1
each source allocates its budget to its represented classes as in E2. The
joint set is trained with the same objective against
\(\Phi^{T,\eta}_{D_\cup}\) at the same total number of images. Training
draws four initializations per iteration from a fixed bank of \(64\)
(\(32\) for the step-size and \((4,2)\) supplements); evaluation uses
\(32\) fresh initializations never used in training. The bank leaves a
generalization gap over \(\theta_0\) (IID, ten images per class:
matching loss \(0.034\) on the bank against \(0.063\) on the held-out
initializations). To attribute the source-error plateau reported below,
the Dir-0.1 cells at ten images per class were retrained on all three
partitions under the four combinations of bank size (\(64\) or
\(512\) initializations) and iteration count (\(1{,}100\) or
\(3{,}000\)), scored on the same \(32\) held-out initializations and
the same real targets. The supplements vary the step size
(\(4\eta_{\rm ref}\), \(16\eta_{\rm ref}\)), the ratio
(\(T=4\), \(K=2\)) and the activation (tanh in place of SiLU) at five
images per class. For \((4,2)\) two predictions are reported: the
composed-synthetic prediction with coefficient \(K\binom p2=2\)
(Theorem \ref{thm:defect}) and the real four-step local discrepancy,
whose second-order coefficient is \(\binom T2=6\).

\paragraph{Analysis.}
Four findings follow (Table \ref{tab:e10} lists the per-cell values,
Figure \ref{fig:e10} plots the Dir-0.1 rows and Figure \ref{fig:syn-e10}
shows learned sets). First, at \(\eta_{\rm ref}\) the structural
prediction \(\eta^2r_{\rm comp}(\theta_0)\) captures the local endpoint
discrepancy \(\Delta\Phi(\theta_0)\) in magnitude and direction on both
activations: cosine \(1.000\) (SiLU) and \(0.999\) (tanh) with relative
residual \(0.02\)--\(0.04\) on every partition, and for \((4,2)\) the
real four-step discrepancy is approximated by \(6\eta^2r_{\rm comp}\)
with cosine \(0.999\) and norm ratio \(0.985\) (IID) and \(0.905\)
(Dir-0.1). Second,
the learned sets retain a large source error: the relative source error
is \(0.50\) at one image per class and \(0.25\)--\(0.39\) at five and ten
(\(0.21\)--\(0.24\) for tanh), so at Dir-0.1 the aggregate residual
\(\|r_{\rm src}\|_{L^2(P_0)}\) is \(23.6\), \(13.3\) and \(12.7\) times
\(\mathcal C_{\rm struct}\) at the three budgets (the bound
\(\mathcal U_{\rm src}\) is \(39\)--\(50\) times), the composed discrepancy
is nearly orthogonal to the prediction (cosine at most \(0.05\)), and the
inequality of Corollary \ref{cor:detectability}, evaluated initialization
by initialization, is satisfied for none of \(96\) initializations at
\(\eta_{\rm ref}\), for \((4,2)\) and for tanh. The count is the numerical
inequality only, since the corollary's common-optimum assumption is not
established on the network. The joint set has a lower error than the
union in every cell, and the excess shrinks with the budget. Third, the
local prediction degrades as the step size grows: at \(16\eta_{\rm ref}\)
it overshoots \(\Delta\Phi\) by a factor of \(1.9\) and
\(\mathcal C_{\rm ind}\) falls below \(\mathcal C_{\rm struct}\), so the
\(3\) of \(96\) initializations that satisfy the inequality there reflect
the failing prediction rather than a smaller source error. Fourth, a
larger initialization bank and longer training do not remove the source
error (Table \ref{tab:e10-attrib}): \(3{,}000\) instead of \(1{,}100\)
iterations changes the held-out relative source error by at most
\(0.002\), a bank of \(512\) instead of \(64\) initializations lowers it
from \(0.380\) to \(0.370\) while narrowing the gap between bank and
held-out matching loss (\(0.116\) against \(0.144\) to \(0.130\) against
\(0.137\)), \(\mathcal U_{\rm pw}/\mathcal C_{\rm struct}\) moves from
\(38.5\) to \(37.5\), and the inequality holds for none of \(96\)
initializations under every condition. Other optimization choices and the
objective itself were not varied.

\begin{table}[h]
\caption{E10: attribution of the source-error plateau. Dir-0.1, ten images per class per source, \((2,1)\) at \(\eta_{\rm ref}\), SiLU, three partitions, all conditions scored on the same 32 held-out initializations. Bank loss is the relative matching loss of the source sets on their training bank, held-out is the relative source error on the evaluation initializations (mean and SD over partitions), errors are in units of \(10^{-3}\), and \(\widehat\Gamma_{\rm RMS}=\mathcal C_{\rm ind}-\mathcal C_{\rm joint}\) is computed from the unrounded RMS values. These runs were archived with \(\mathcal U_{\rm pw}\), the RMS over initializations of \(\sum_i\alpha_i\|e_i(\theta_0)\|\), in place of \(\mathcal U_{\rm src}\) of Table \ref{tab:e10}; on the cell common to both tables the two differ by \(0.2\%\) (\(3.808\) against \(3.814\)). Ineq.\ counts initializations as in Table \ref{tab:e10}.}
\label{tab:e10-attrib}
\centering
\footnotesize
\begin{tabular}{@{}rrrrrrrrrr@{}}
\toprule
Bank & Iter. & Bank loss & Held-out & \(\mathcal U_{\rm pw}\) & \(\mathcal C_{\rm ind}\) & \(\mathcal C_{\rm joint}\) & \(\widehat\Gamma_{\rm RMS}\) & \(\mathcal U_{\rm pw}/\mathcal C_{\rm struct}\) & ineq.\\
\midrule
64 & 1,100 & 0.116 & \(0.380\pm0.008\) & 3.81 & 1.25 & 0.99 & 0.26 & 38.5 & 0/96\\
64 & 3,000 & 0.114 & \(0.379\pm0.008\) & 3.79 & 1.25 & 0.99 & 0.26 & 38.4 & 0/96\\
512 & 1,100 & 0.132 & \(0.372\pm0.008\) & 3.73 & 1.20 & 0.95 & 0.24 & 37.7 & 0/96\\
512 & 3,000 & 0.130 & \(0.370\pm0.008\) & 3.71 & 1.20 & 0.95 & 0.25 & 37.5 & 0/96\\
\bottomrule
\end{tabular}
\end{table}

\begin{figure}[h]
\centering
\includegraphics[width=\linewidth]{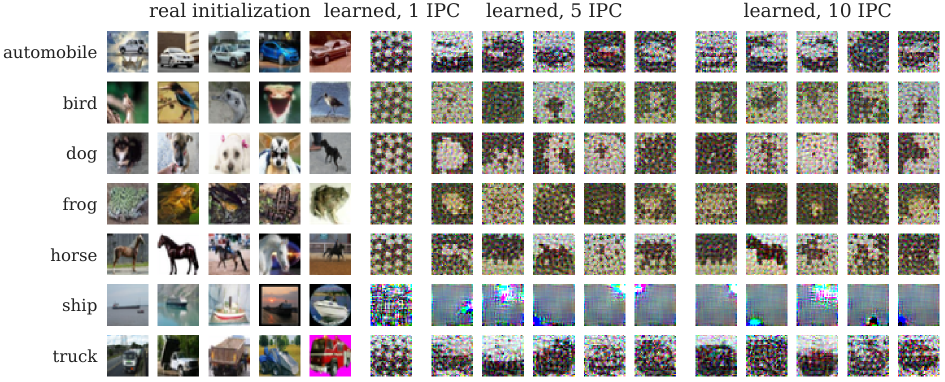}
\caption{Endpoint-matching distillates on the SiLU network (E10; Dir-0.1 partition 0 of the E1 subset, source 0, which holds the seven classes shown). Images are initialized from real images of the class and optimized for \(1{,}100\) Adam iterations on the one-step-versus-two-step endpoint discrepancy; the columns show the initialization and the learned sets at one, five and ten images per class (up to five per class displayed). The learned images move far from their initialization (median relative displacement \(1.1\)), since the objective constrains the induced training step and not the appearance.}
\label{fig:syn-e10}
\end{figure}

\begin{figure}[h]
\centering
\includegraphics[width=\linewidth]{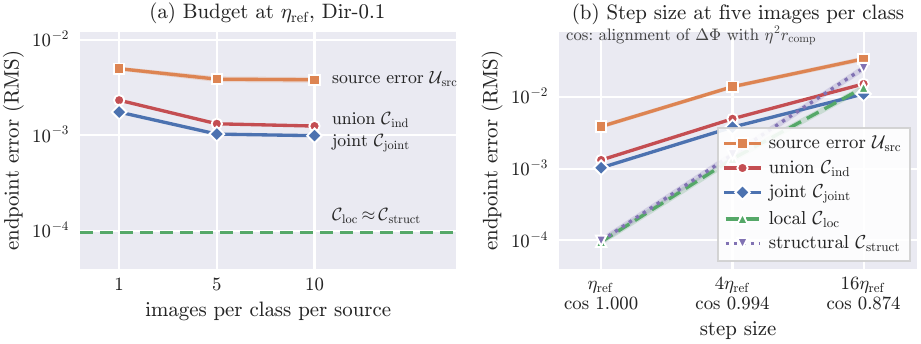}
\caption{At small step sizes the local prediction is accurate, but the endpoint error of learned synthetic sets lies far above the local discrepancy. Endpoint-matching distillates on the SiLU network at Dir-0.1 (three partitions, RMS over \(32\) held-out initializations, mean and SD over partitions). (a) Weighted source error, error of the union of the learned sets and error of the jointly learned set against the image budget at \(\eta_{\rm ref}\); the local reference \(\mathcal C_{\rm loc}\) and the structural composition error \(\mathcal C_{\rm struct}\) closely agree and lie far below. (b) The same quantities against the step size at five images per class, with the cosine between \(\Delta\Phi\) and \(\eta^2r_{\rm comp}\) at each step size.}
\label{fig:e10}
\end{figure}

\begin{table}[h]
\caption{E10: endpoint-matching distillates on the network (SiLU unless stated). Per cell, RMS over 32 held-out initializations; mean over three partitions. Errors in units of \(10^{-3}\). \(\mathcal U_{\rm src}=\sum_i\alpha_i\|\Phi^{K,\eta}_{S_i}-\Phi^{T,\eta}_{D_i}\|_{L^2(P_0)}\) is the weighted source error of Section \ref{sec:e10} at the row's \((T,K)\), rel.\ is the weighted mean over sources of each set's endpoint error divided by its own real \(T\)-step displacement, \(\hat s=K\binom p2\eta^2r_{\rm comp}\) is the structural prediction (coefficient \(1\) for \((2,1)\) and \(2\) for \((4,2)\), Theorem \ref{thm:defect}), and \(d_{\rm ind}\) is the per-initialization error of the union of the learned sets against the real-union endpoint. The last column (ineq.) is evaluated before aggregation: it counts initializations \(\theta_0\) at which \(\|\hat s(\theta_0)\|>\sum_i\alpha_i\|e_i(\theta_0)\|\), with \(e_i(\theta_0)=\Phi^{K,\eta}_{S_i}(\theta_0)-\Phi^{T,\eta}_{D_i}(\theta_0)\).}
\label{tab:e10}
\centering
\footnotesize
\setlength{\tabcolsep}{2.6pt}
\begin{tabular}{@{}lllrrrrrrrrr@{}}
\toprule
Level & Setting & IPC & \(\mathcal U_{\rm src}\) & rel. & \(\mathcal C_{\rm ind}\) & \(\mathcal C_{\rm joint}\) & \(\mathcal C_{\rm loc}\) & \(\mathcal C_{\rm struct}\) & \(\cos(\Delta\Phi,\hat s)\) & \(\cos(d_{\rm ind},\hat s)\) & ineq.\\
\midrule
IID & \((2,1)\), \(\eta_{\rm ref}\) & 1 & 2.03 & 0.51 & 1.84 & 1.51 & 0.00009 & 0.00009 & 1.000 & 0.05 & 0/96\\
IID & \((2,1)\), \(\eta_{\rm ref}\) & 5 & 1.10 & 0.27 & 1.04 & 1.01 & 0.00009 & 0.00009 & 1.000 & 0.04 & 0/96\\
IID & \((2,1)\), \(\eta_{\rm ref}\) & 10 & 1.02 & 0.25 & 0.99 & 0.99 & 0.00009 & 0.00009 & 1.000 & 0.04 & 0/96\\
IID & \((2,1)\), \(4\eta_{\rm ref}\) & 5 & 4.11 & 0.26 & 3.86 & 3.71 & 0.0014 & 0.0014 & 0.997 & 0.03 & 0/96\\
IID & \((2,1)\), \(16\eta_{\rm ref}\) & 5 & 12.9 & 0.23 & 11.6 & 10.6 & 0.021 & 0.022 & 0.955 & \(-0.04\) & 0/96\\
IID & \((4,2)\), \(\eta_{\rm ref}\) & 5 & 2.16 & 0.27 & 2.04 & 1.97 & 0.0005 & 0.0002 & 0.999 & 0.04 & 0/96\\
Dir-0.1 & \((2,1)\), \(\eta_{\rm ref}\) & 1 & 4.98 & 0.50 & 2.33 & 1.75 & 0.096 & 0.099 & 1.000 & 0.01 & 0/96\\
Dir-0.1 & \((2,1)\), \(\eta_{\rm ref}\) & 5 & 3.87 & 0.39 & 1.32 & 1.03 & 0.096 & 0.099 & 1.000 & 0.04 & 0/96\\
Dir-0.1 & \((2,1)\), \(\eta_{\rm ref}\) & 10 & 3.81 & 0.38 & 1.25 & 0.99 & 0.096 & 0.099 & 1.000 & 0.04 & 0/96\\
Dir-0.1 & \((2,1)\), \(4\eta_{\rm ref}\) & 5 & 13.9 & 0.36 & 4.96 & 3.79 & 1.37 & 1.58 & 0.994 & 0.04 & 0/96\\
Dir-0.1 & \((2,1)\), \(16\eta_{\rm ref}\) & 5 & 34.0 & 0.28 & 15.2 & 11.0 & 13.5 & 25.3 & 0.874 & 0.02 & 3/96\\
Dir-0.1 & \((4,2)\), \(\eta_{\rm ref}\) & 5 & 7.50 & 0.39 & 2.61 & 2.01 & 0.54 & 0.20 & 0.998 & 0.04 & 0/96\\
\midrule
IID, tanh & \((2,1)\), \(\eta_{\rm ref}\) & 5 & 1.06 & 0.21 & 0.93 & 0.85 & 0.00013 & 0.00013 & 0.999 & 0.01 & 0/96\\
Dir-0.1, tanh & \((2,1)\), \(\eta_{\rm ref}\) & 5 & 1.98 & 0.24 & 1.24 & 0.90 & 0.054 & 0.054 & 0.999 & 0.01 & 0/96\\
\bottomrule
\end{tabular}
\end{table}

\subsection{Independent versus joint distillation (E2)}
\label{app:e2}

\begin{table}[h]
\caption{Separate-then-union versus joint distillation on CIFAR-10 (distribution matching, official schedule, 10 images per class per source). Operator quantities are RMS endpoint distances to the real-union two-step map, in units of \(10^{-3}\). Accuracies in percent.}
\label{tab:e2-gap}
\centering
\small
\setlength{\tabcolsep}{2.8pt}
\begin{tabular}{@{}l S[table-format=1.3(3)] S[table-format=1.2(2)] S[table-format=1.2(2)] S[table-format=1.2(2)] S[table-format=2.1(1)] S[table-format=2.1(1)] S[table-format=+1.1(1)]@{}}
\toprule
& \multicolumn{4}{c}{Endpoint errors} & \multicolumn{3}{c}{Test accuracy}\\
\cmidrule(r){2-5}\cmidrule(l){6-8}
Source split & {reference \(\mathcal C_{\rm loc}\)} & {separate \(\mathcal C_{\rm ind}\)} & {joint \(\mathcal C_{\rm joint}\)} & {excess \(\widehat\Gamma_{\rm RMS}\)} & {separate} & {joint} & {gain}\\
\midrule
\multicolumn{8}{@{}l}{\emph{Heterogeneity sweep (five partitions per level)}}\\
IID & 0.016(0.001) & 7.41(0.04) & 7.34(0.08) & 0.06(0.10) & 46.3(0.9) & 51.6(0.6) & +5.3(0.9)\\
Dir-1.0 & 4.85(0.99) & 8.72(0.29) & 7.34(0.08) & 1.38(0.25) & 46.2(0.6) & 51.9(0.5) & +5.7(0.6)\\
Dir-0.5 & 5.58(0.76) & 8.71(0.31) & 7.33(0.06) & 1.38(0.27) & 45.8(1.1) & 51.7(0.5) & +5.9(1.2)\\
Dir-0.1 & 8.09(1.17) & 9.85(1.33) & 7.50(0.26) & 2.35(1.44) & 45.8(0.7) & 51.1(1.1) & +5.3(0.7)\\
\addlinespace
\multicolumn{8}{@{}l}{\emph{No-split control (three replicates)}}\\
No split & {0 (exact)} & {--} & {--} & {--} & 46.0(0.4) & 52.6(0.4) & +6.6(0.4)\\
\bottomrule
\end{tabular}
\par\smallskip\raggedright\scriptsize
Mean (SD) over partitions. Within a partition, endpoint errors are RMS over three initialization seeds and accuracies average three evaluation seeds. Gain: joint minus separate accuracy in percentage points. 0 (exact): \(V_H\equiv0\) by construction. The no-split endpoint errors were aggregated as \(\widehat\Gamma_{\rm mean}\) and are reported in Appendix \ref{app:nosplit} and Table \ref{tab:svhn}, hence the dashes.
\end{table}

The independent-versus-joint results of Section \ref{sec:experiments} (Table \ref{tab:e2-gap}; Figure \ref{fig:levels})
use distribution matching (DM) at its official schedule: 20{,}000
iterations, image learning rate 1.0, real batch 256, differentiable
Siamese augmentation enabled, and IPC 10 per independently distilled
source, with snapshots at iterations
\(\{0,1000,2500,5000,10000,15000,20000\}\). The joint set receives the
sum of the independent image budgets, which is IPC \(15\)--\(20\) over the
union depending on how many classes each source represents, and stores
the same semantic mass (Figure \ref{fig:syn-dm} shows the sets on one
Dir-0.1 partition). Its accuracy is therefore comparable to published
IPC-10 numbers only up to that budget difference. Evaluation follows DM's
official protocol (1000 epochs, augmentation enabled), with three
evaluation seeds averaged within each partition.

Gradient-based condensation (DC) is used only as a short-schedule
optimization control (Appendix \ref{app:legacy}). At the official
1000-iteration schedule on the full CIFAR-10 training set our DC runs
reach 41.4\% at IPC 10, below the published 44.9\%, so they are not part
of the official-schedule evidence. The short-schedule DC results are
retained in the tables marked ``short schedule''; their
endpoint-level behaviour matches DM's, with a positive gap at Dir-0.1, and
the canonical and matched-step choices of DC's IPC-dependent loop counts
for the joint set agree at the endpoint level to within 0.3\%. DM
performs one image update per iteration and has no IPC-dependent loop
structure.

Independent and joint endpoints use the same \(\theta_0\) and
\(\eta=0.00625\). The endpoint protocol compares \(\Phi_{S_i}^{1,\eta}\)
with \(\Phi_{D_i}^{2,\eta}\) for every source and
\(\Phi_{S_1\oplus S_2}^{1,\eta}\) with \(\Phi_{D_1\oplus D_2}^{2,\eta}\)
for the union. Hence \(T=2\) governs both real targets, while \(K=1\)
governs only the synthetic maps. Each synthetic set enters with its stored
class masses as loss weights and with a scalar step multiplier fitted at
each evaluation initialization: for source \(i\) the measured map is
\(\Phi^{1,c_i\eta}_{S_i}\) with
\(c_i=\arg\min_c\|\Phi^{1,c\eta}_{S_i}(\theta_0)-\Phi^{2,\eta}_{D_i}(\theta_0)\|\),
which has the closed form \(c_i=\langle d_i,\delta_i\rangle/\|d_i\|^2\) for the
raw synthetic displacement \(d_i\) and the real two-step displacement
\(\delta_i\), and the joint set is calibrated in the same way against the
union target. This is the per-set form of the scalar retiming of Appendix
\ref{app:repair}: it uses the real two-step endpoint of the set's own
target, so \(\mathcal C_{\rm ind}\), \(\mathcal C_{\rm joint}\) and every
quantity derived from them in E2, E7, E9, E11 and E12 are endpoint errors
after per-set scalar calibration, whereas the endpoint-matching sets of
E10 are trained and evaluated at the fixed step \(\eta\). The composed
union is the mass-weighted sum of the calibrated source displacements.
Appendix \ref{app:e9} reports the same decomposition without the
calibration (class-mass weights only, \(c_i=1\)).
The reported constituent error, including
\(\epsilon_{\rm rel}\), uses this one-step versus two-step comparison. The
learning rate supports ranking but is not used for the asymptotic
calibration claim.

\begin{figure}[h]
\centering
\includegraphics[width=0.6\linewidth]{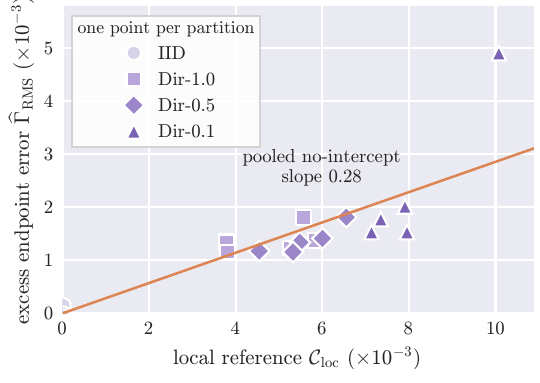}
\caption{Independent versus joint distillation: excess endpoint error
versus the local reference \(\mathcal C_{\rm loc}\) on CIFAR-10.
Each point represents one partition (five per level, three seeds each);
the line is a descriptive pooled no-intercept fit across the skewed
levels (slope 0.28).
Level means appear in Table \ref{tab:e2-gap}.}
\label{fig:levels}
\end{figure}

\begin{figure}[h]
\centering
\includegraphics[width=\linewidth]{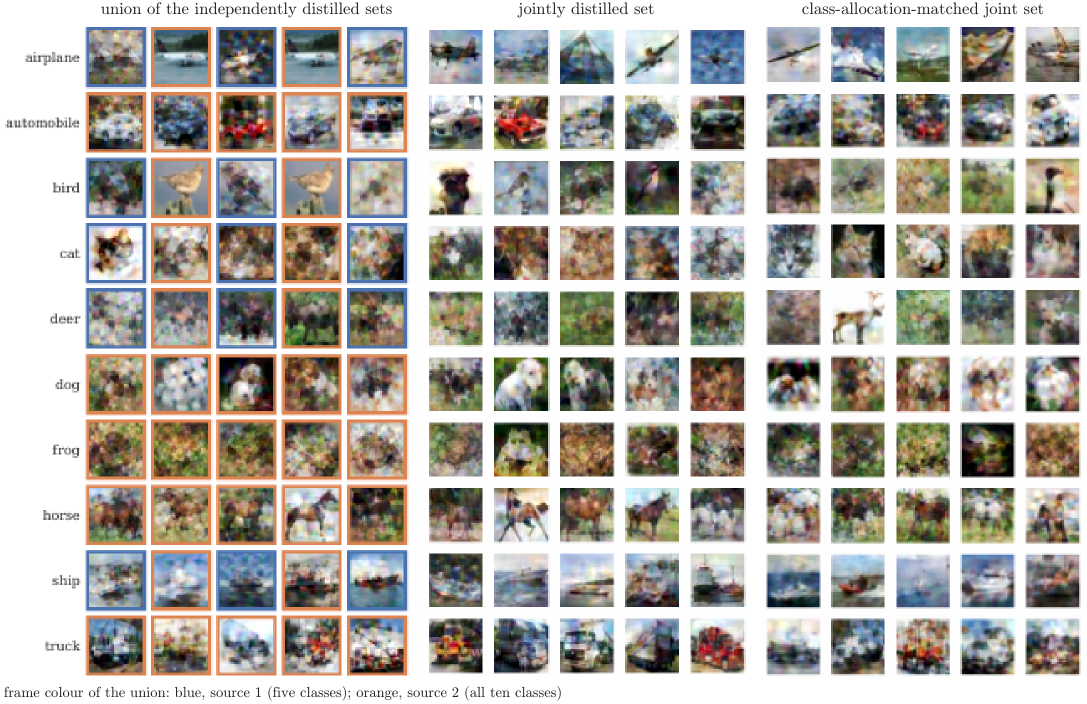}
\caption{Synthetic images on one Dir-0.1 partition of CIFAR-10 (partition 0, DM at the official schedule; up to five images per class per set are shown). Left: the union of the two independently distilled sets (E2), with the frame colour marking the source of each image (blue, source 1, which holds five classes; orange, source 2, which holds all ten); each source stores ten images per class it holds, so the union has twenty images in the shared classes and ten in the others. Middle: the jointly distilled set at the same total budget of \(150\) images, \(15\) per class. Right: the class-allocation-matched joint set of E11 (Appendix \ref{app:balanced}), which places \(20\) images in the five shared classes and \(10\) in the five classes held by one source, the per-class counts of the union.}
\label{fig:syn-dm}
\end{figure}

\paragraph{Paired analysis.}
The official-schedule analysis uses five partitions at each of four levels
(IID, Dir-1.0, Dir-0.5, Dir-0.1), with the same five partition seeds at
every level so that level comparisons can be paired within partition.
Initialization repeats are aggregated within partition before any
between-partition statistic is computed. Results from the earlier short
schedule (500 DM iterations, 50-epoch evaluation) and the feature-skew
family are collected in Appendix \ref{app:legacy}.

\subsection{Exact decomposition of the observed discrepancy (E9)}
\label{app:e9}

\paragraph{Configuration.}
For the E2 sets, E9 records the full endpoint vectors for each
initialization seed (100--102) and decomposes the discrepancy of the
separately distilled union exactly. With the union trained under a single
mass-weighted loss, the one-step update on \(S_1\cup S_2\) is the
weighted sum of the one-step updates on the parts, so
\[
d_{\rm ind}(\theta_0)=\underbrace{\sum_i\alpha_i\Phi^{2,\eta}_{D_i}(\theta_0)-\Phi^{2,\eta}_{D_\cup}(\theta_0)}_{\Delta\Phi(\theta_0)}
+\underbrace{\sum_i\alpha_i\bigl(\Phi^{1,\eta}_{S_i}(\theta_0)-\Phi^{2,\eta}_{D_i}(\theta_0)\bigr)}_{r_{\rm src}(\theta_0)}
\]
holds without any quadratic or common-optimum assumption, with
\(\Phi^{1,\eta}_{S_i}\) the calibrated one-step maps of the E2 protocol
(Appendix \ref{app:e2}). The identity
was verified numerically on every cell (residual at most \(6\times10^{-8}\),
relative \(9\times10^{-6}\), the float32 accumulation error), and the
norms \(\|d_{\rm ind}\|\), \(\|d_{\rm joint}\|\) and
\(\|\Delta\Phi\|\) reproduce \(\mathcal C_{\rm ind}\),
\(\mathcal C_{\rm joint}\) and \(\mathcal C_{\rm loc}\) of Table
\ref{tab:e2-gap} to \(3\times10^{-6}\). The excess of Section \ref{sec:experiments} subtracts the joint error,
\(\widehat\Gamma_{\rm RMS}=\|\Delta\Phi+r_{\rm src}\|_{L^2(P_0)}-\|d_{\rm joint}\|_{L^2(P_0)}\),
so the decomposition bears on its first term. The local approximation
residual \(\Delta\Phi-\eta^2r_{\rm comp}\) separates the network
nonlinearity from the source distillation error. All 108 cells were
processed (CIFAR-10 at four levels and the no-split control, SVHN at two
levels and its no-split control), and the 18 CIFAR-100 cells of E15
(Appendix \ref{app:svhn}) were added with the same script.

\paragraph{Analysis.}
At the three skewed CIFAR-10 levels \(\Delta\Phi\) and \(r_{\rm src}\)
are anti-aligned: the cosine is \(-0.35\), \(-0.48\) and \(-0.46\) at
Dir-1.0, Dir-0.5 and Dir-0.1 (\(-0.78\) on SVHN at Dir-0.1), and none of
the 45 skewed CIFAR-10 cells or the 15 skewed SVHN cells has a positive
cosine, and the measured \(\|d_{\rm ind}\|\) is smaller than
\(\|r_{\rm src}\|\) although the two terms are of the same order. At the E2 step size
(\(\eta L\approx1\)) the prediction \(\eta^2r_{\rm comp}\) has a
magnitude comparable to \(\Delta\Phi\) (ratio \(1.0\)--\(1.5\)) but only
partly its direction (cosine \(0.41\)--\(0.60\), relative residual above one),
and it is nearly orthogonal to \(d_{\rm ind}\) (cosine \(0.01\)--\(0.06\)),
so the fraction of \(\|d_{\rm ind}\|^2\) it explains is negative
(\(-0.25\), \(-0.69\) and \(-1.55\)). Under IID partitioning \(\|\Delta\Phi\|=1.6\times10^{-5}\), so
\(d_{\rm ind}\approx r_{\rm src}\) (cosine \(1.000\)); in the
no-split control \(\Delta\Phi\) is exactly zero and
\(d_{\rm ind}=r_{\rm src}\). Table \ref{tab:e9} lists the per-level
values.

\begin{table}[h]
\caption{E9: exact decomposition of the observed discrepancy on the E2 sets. Norms in units of \(10^{-3}\), RMS over seeds within partition and mean over partitions; cosines are means over seeds and partitions. \(\hat s=\eta^2r_{\rm comp}\).}
\label{tab:e9}
\centering
\footnotesize
\setlength{\tabcolsep}{3pt}
\begin{tabular}{@{}llrrrrrrrr@{}}
\toprule
Dataset & Level & \(\|d_{\rm ind}\|\) & \(\|d_{\rm joint}\|\) & \(\|\Delta\Phi\|\) & \(\|r_{\rm src}\|\) & \(\|\hat s\|\) & \(\cos(\Delta\Phi,r_{\rm src})\) & \(\cos(\Delta\Phi,\hat s)\) & \(\cos(d_{\rm ind},\hat s)\)\\
\midrule
CIFAR-10 & IID & 7.41 & 7.34 & 0.016 & 7.41 & 0.002 & \(-0.01\) & 0.10 & \(-0.04\)\\
CIFAR-10 & Dir-1.0 & 8.72 & 7.34 & 4.85 & 9.15 & 4.84 & \(-0.35\) & 0.41 & 0.06\\
CIFAR-10 & Dir-0.5 & 8.71 & 7.33 & 5.58 & 9.84 & 7.26 & \(-0.48\) & 0.51 & 0.01\\
CIFAR-10 & Dir-0.1 & 9.85 & 7.50 & 8.09 & 10.4 & 12.3 & \(-0.46\) & 0.60 & 0.04\\
CIFAR-10 & no split & 7.43 & 7.34 & 0 & 7.43 & 0 & -- & -- & --\\
SVHN & IID & 6.08 & 6.13 & 0.013 & 6.08 & 0.002 & \(-0.02\) & 0.08 & 0.09\\
SVHN & Dir-0.1 & 7.65 & 6.18 & 10.5 & 12.1 & 12.5 & \(-0.78\) & 0.32 & \(-0.06\)\\
SVHN & no split & 6.21 & 6.12 & 0 & 6.21 & 0 & -- & -- & --\\
CIFAR-100 & IID & 1.95 & 1.86 & 0.018 & 1.95 & 0.002 & \(-0.01\) & 0.11 & 0.00\\
CIFAR-100 & Dir-0.1 & 2.51 & 1.96 & 1.35 & 2.87 & 1.32 & \(-0.48\) & 0.86 & \(-0.02\)\\
\bottomrule
\end{tabular}
\end{table}

\paragraph{Without the scalar calibration.}
Table \ref{tab:e9-noclock} repeats the decomposition with the class-mass
weights only, without the per-set step multiplier of the E2 protocol;
\(\Delta\Phi\) and \(\eta^2r_{\rm comp}\) do not depend on the synthetic
sets and are unchanged. The fitted multipliers are \(1.39\)--\(1.42\) at the
skewed CIFAR-10 levels, \(1.56\) at IID, \(1.25\)--\(1.29\) on SVHN and
\(1.9\) on CIFAR-100, so
without them the source errors are dominated by a scale mismatch common
to all levels: \(\|d_{\rm ind}\|\) rises to \(13\)--\(14\times10^{-3}\) on
CIFAR-10 and varies little with heterogeneity, and the excess shrinks to
between one eighth and one eighteenth of its calibrated value on CIFAR-10,
to a quarter on SVHN and to about one fiftieth on CIFAR-100. The
heterogeneity-dependent part of the decomposition is present in both
readings: \(\Delta\Phi\) and \(r_{\rm src}\) remain anti-aligned at every
skewed level (cosine \(-0.11\) to \(-0.56\)) and the excess remains
positive at every skewed level and indistinguishable from zero at IID.

\begin{table}[h]
\caption{E9 without the per-set scalar calibration. Each entry gives the calibrated value of Table \ref{tab:e9} followed by the uncalibrated value (class-mass weights, \(c_i=1\)). Norms in units of \(10^{-3}\); the excess is \(\widehat\Gamma_{\rm mean}\), the signed per-initialization difference \(\|d_{\rm ind}\|-\|d_{\rm joint}\|\) averaged within partition, mean \(\pm\) SD over partitions, which differs from \(\widehat\Gamma_{\rm RMS}=\mathcal C_{\rm ind}-\mathcal C_{\rm joint}\) of Table \ref{tab:e2-gap} by a few percent.}
\label{tab:e9-noclock}
\centering
\small
\setlength{\tabcolsep}{3.5pt}
\begin{tabular}{@{}llccccc@{}}
\toprule
Dataset & Level & \(\|d_{\rm ind}\|\) & \(\|d_{\rm joint}\|\) & \(\|r_{\rm src}\|\) & \(\cos(\Delta\Phi,r_{\rm src})\) & excess\\
\midrule
CIFAR-10 & IID & 7.41 / 13.4 & 7.34 / 13.4 & 7.41 / 13.4 & \(-0.01\) / \(+0.02\) & \(+0.07\pm0.09\) / \(-0.00\pm0.09\)\\
CIFAR-10 & Dir-1.0 & 8.72 / 13.4 & 7.34 / 13.4 & 9.15 / 13.6 & \(-0.35\) / \(-0.20\) & \(+1.42\pm0.25\) / \(+0.08\pm0.08\)\\
CIFAR-10 & Dir-0.5 & 8.71 / 13.5 & 7.33 / 13.3 & 9.84 / 13.9 & \(-0.48\) / \(-0.27\) & \(+1.42\pm0.27\) / \(+0.17\pm0.11\)\\
CIFAR-10 & Dir-0.1 & 9.85 / 13.7 & 7.50 / 13.4 & 10.4 / 14.4 & \(-0.46\) / \(-0.36\) & \(+2.38\pm1.45\) / \(+0.28\pm0.08\)\\
SVHN & Dir-0.1 & 7.65 / 11.0 & 6.18 / 10.7 & 12.1 / 12.8 & \(-0.78\) / \(-0.56\) & \(+1.47\pm0.95\) / \(+0.37\pm0.16\)\\
CIFAR-100 & Dir-0.1 & 2.51 / 7.83 & 1.96 / 7.82 & 2.87 / 7.85 & \(-0.48\) / \(-0.11\) & \(+0.56\pm0.03\) / \(+0.01\pm0.03\)\\
\bottomrule
\end{tabular}
\end{table}

\subsection{Post-hoc class reweighting (E12)}
\label{app:reweight}

\paragraph{Configuration.}
E12 asks how much of the composed error of the saved E2 sets can be removed
by changing only the class weights of the synthetic loss, with the images
fixed. The one-step synthetic update is linear in the class weights,
\(d_{\rm syn}(w)=-\eta\sum_{(i,c)}w_{i,c}\,g_{i,c}(\theta_0)\), where
\(g_{i,c}(\theta_0)\) is the cross-entropy gradient of the class-\(c\)
images of set \(S_i\) at \(\theta_0\), so the weights that best match the
real-union two-step displacement \(r_{\rm real}(\theta_0)=\Phi^{2,\eta}_{D_\cup}(\theta_0)-\theta_0\)
solve a non-negative least-squares problem in at most twenty variables.
Three weightings are compared for the independent union and, with the same
procedure, for the joint set: the paper's weighting (class-mass weights with
one scalar step multiplier per set fitted at each initialization, as in E2); weights
fitted on a bank of 16 initializations (seeds 300--315, disjoint from the
evaluation seeds) by accumulating the normal equations over the bank; and an
oracle fitted on the evaluation initialization itself, which bounds what any
class reweighting can achieve. All three are evaluated on the paper's
initialization seeds 100--102, the endpoint error is \(\|d_{\rm syn}(w)-r_{\rm real}\|\)
as in E2, and the fraction of squared error removed is
\(1-(\mathcal C_w/\mathcal C_{\rm paper})^2\). Under the bank weights the
accuracy of the Dir-0.1 sets was re-evaluated with the official protocol and a
class-weighted sampler (weights normalized to mean one), against the same
sampler with uniform weights.

\paragraph{Analysis.}
Table \ref{tab:reweight} lists the results. Two quantities move
differently. The squared error of the independent union falls by
\(15\)--\(25\%\) at the skewed levels under weights fitted on held-out
initializations and by \(45\)--\(53\%\) under the oracle on CIFAR-10
(\(41\%\) on SVHN), so at least \(47\%\) of \(\|d_{\rm ind}\|^2\) at
these levels is not removed by any class reweighting, in line with the components that scalar retiming cannot
repair (Appendix \ref{app:repair}); at IID the held-out weights do not
transfer. The mean excess over the joint set falls much more, from
\(1.38\), \(1.38\) and \(2.35\) to \(0.10\), \(0.21\) and \(0.67\)
(\(\times10^{-3}\)) at Dir-1.0, Dir-0.5 and Dir-0.1 under the held-out
weights, partly because the same fitting procedure increases the squared
endpoint error of the joint set on CIFAR-10 by \(12\)--\(14\%\); it remains positive at every skewed
level under all three weightings. At IID it is \(+0.06\pm0.10\) under
the paper's weighting and \(-0.01\pm0.12\) under the bank weights, neither
distinguishable from zero, and \(+0.07\pm0.04\) under the oracle
(\(t=4.16\), \(df=4\), \(p=0.014\), uncorrected two-sided). On accuracy, the held-out weights give no
detectable change on the five Dir-0.1 partitions: \(45.66\pm0.73\) against
\(45.50\pm0.70\) for the union (paired difference \(+0.16\pm0.49\)) and
\(50.36\pm1.01\) against \(50.28\pm0.91\) for the joint set
(\(+0.08\pm0.41\)).

\begin{table}[h]
\caption{E12: endpoint error of the saved DM sets under three class weightings. RMS over the three evaluation initializations within partition, mean \(\pm\) SD over five partitions, in units of \(10^{-3}\). Removed is the fraction of squared error removed relative to the paper's weighting; excess is \(\mathcal C_{\rm ind}-\mathcal C_{\rm joint}\) under the same weighting.}
\label{tab:reweight}
\centering
\small
\setlength{\tabcolsep}{4pt}
\begin{tabular}{@{}lllcccc@{}}
\toprule
Dataset & Level & Weighting & \(\mathcal C_{\rm ind}\) & \(\mathcal C_{\rm joint}\) & removed (ind / joint) & excess\\
\midrule
CIFAR-10 & IID & paper & 7.41 \(\pm\) 0.04 & 7.34 \(\pm\) 0.08 & -- & \(+\)0.06 \(\pm\) 0.10\\
 & & bank & 7.82 \(\pm\) 0.05 & 7.83 \(\pm\) 0.08 & \(-\)0.11 / \(-\)0.14 & \(-\)0.01 \(\pm\) 0.12\\
 & & oracle & 5.83 \(\pm\) 0.05 & 5.76 \(\pm\) 0.04 & 0.38 / 0.38 & \(+\)0.07 \(\pm\) 0.04\\
CIFAR-10 & Dir-1.0 & paper & 8.72 \(\pm\) 0.29 & 7.34 \(\pm\) 0.08 & -- & \(+\)1.38 \(\pm\) 0.25\\
 & & bank & 7.93 \(\pm\) 0.08 & 7.83 \(\pm\) 0.08 & 0.17 / \(-\)0.14 & \(+\)0.10 \(\pm\) 0.08\\
 & & oracle & 6.00 \(\pm\) 0.10 & 5.76 \(\pm\) 0.04 & 0.53 / 0.38 & \(+\)0.24 \(\pm\) 0.10\\
CIFAR-10 & Dir-0.5 & paper & 8.71 \(\pm\) 0.31 & 7.33 \(\pm\) 0.06 & -- & \(+\)1.38 \(\pm\) 0.27\\
 & & bank & 8.03 \(\pm\) 0.28 & 7.81 \(\pm\) 0.07 & 0.15 / \(-\)0.14 & \(+\)0.21 \(\pm\) 0.29\\
 & & oracle & 6.26 \(\pm\) 0.32 & 5.78 \(\pm\) 0.05 & 0.48 / 0.38 & \(+\)0.47 \(\pm\) 0.32\\
CIFAR-10 & Dir-0.1 & paper & 9.85 \(\pm\) 1.33 & 7.50 \(\pm\) 0.26 & -- & \(+\)2.35 \(\pm\) 1.44\\
 & & bank & 8.62 \(\pm\) 0.27 & 7.95 \(\pm\) 0.24 & 0.20 / \(-\)0.12 & \(+\)0.67 \(\pm\) 0.13\\
 & & oracle & 7.15 \(\pm\) 0.36 & 6.10 \(\pm\) 0.36 & 0.45 / 0.34 & \(+\)1.04 \(\pm\) 0.09\\
SVHN & IID & paper & 6.08 \(\pm\) 0.05 & 6.13 \(\pm\) 0.06 & -- & \(-\)0.05 \(\pm\) 0.03\\
 & & bank & 5.83 \(\pm\) 0.07 & 5.91 \(\pm\) 0.05 & 0.08 / 0.07 & \(-\)0.08 \(\pm\) 0.06\\
 & & oracle & 4.98 \(\pm\) 0.07 & 5.09 \(\pm\) 0.09 & 0.33 / 0.31 & \(-\)0.11 \(\pm\) 0.11\\
SVHN & Dir-0.1 & paper & 7.65 \(\pm\) 0.91 & 6.18 \(\pm\) 0.09 & -- & \(+\)1.46 \(\pm\) 0.98\\
 & & bank & 6.54 \(\pm\) 0.24 & 5.96 \(\pm\) 0.11 & 0.25 / 0.07 & \(+\)0.58 \(\pm\) 0.31\\
 & & oracle & 5.80 \(\pm\) 0.24 & 5.15 \(\pm\) 0.09 & 0.41 / 0.31 & \(+\)0.64 \(\pm\) 0.30\\
\bottomrule
\end{tabular}
\par\smallskip\raggedright\footnotesize
The removed fractions are means of the per-partition values; the Dir-0.1 union value under bank weights varies across partitions (SD \(0.18\) on CIFAR-10, \(0.17\) on SVHN).
\end{table}

\subsection{Downstream accuracy and the no-split control (E7)}
\label{app:nosplit}

Each final synthetic set is evaluated with the distilling method's
official protocol. For DM, a ConvNet is trained for 1000 epochs with
differentiable augmentation, and three evaluation seeds are averaged within
each partition. The earlier short schedule used 50 epochs without
augmentation, which under-trained both sets and compressed the gain to
about two points. Image budgets are matched within every partition.
Because some sources omit classes under Dir-0.1, the joint budget varies
between 150 and 200 images across partitions, so absolute accuracy should
be compared only within partition, not across levels. The accuracies are
those of Table \ref{tab:e2-gap}.

The gain from joint distillation is \(+5.3\), \(+5.7\), \(+5.9\), and
\(+5.3\) points at IID, Dir-1.0, Dir-0.5, and Dir-0.1. The
within-partition paired differences from IID are \(+0.44\pm0.94\),
\(+0.58\pm1.67\), and \(+0.06\pm1.23\) (\(t=1.03\), \(0.77\), and \(0.12\)
on four degrees of freedom). The same partition seed gives the largest
gain at three of the four levels, so partition identity carries part of
the variance and paired comparisons are the appropriate ones. Even when
paired, no level effect is detectable. Repeated evaluation of a fixed
synthetic set moves accuracy by up to \(\pm0.35\) points, because the
augmentation stage is not seeded (Appendix \ref{app:experiments}). In the
no-split control (two distillations of the full training set with
different seeds, unioned, against one distillation at the combined
budget), the gain is \(+6.55\pm0.42\) over three replicates, with
per-class effective rank \(8.9\) vs.\ \(10.3\) and cross/within ratio
\(0.941\), the IID values. The mean endpoint-distance difference is
\(\widehat\Gamma_{\rm mean}=(0.08\pm0.21)\times10^{-3}\), negative in one
of the three replicates, and the measured \(\mathcal C_{\rm loc}\) is exactly zero,
as the identity requires.

\paragraph{Warm-started joint distillation (E13).}
\label{app:warm}
To test whether the no-split gain is recovered by continuing distillation
from the union, DM was run on the full training set at the official
settings (image learning rate \(1.0\), differentiable augmentation, real
batch 256) starting from the union of the two independent sets of each
replicate, for 500 and \(2{,}000\) iterations, and, as a control for the
short schedule, from random real images for the same \(2{,}000\) iterations
with the same seeds (Table \ref{tab:warm}). The union and the official joint set were re-evaluated in the same run (\(46.44\) and \(52.59\) against \(46.04\) and \(52.59\) in E7, within the evaluation noise noted above), so the differences in Table \ref{tab:warm} are paired within the run. Warm-starting recovers
\(+0.44\pm0.47\) of the \(6.15\)-point gap after \(2{,}000\) iterations,
whereas the cold start reaches \(+1.83\pm0.65\) in the same number of
iterations with a higher per-class effective rank. Under the same \(2{,}000\)-iteration budget, initialization from real
images therefore yields the larger gain, and the warm-started images change
little over this interval (median relative displacement \(0.035\)). Whether
a longer continuation or other optimization settings would close the gap
from the union was not tested.

\begin{table}[h]
\caption{E13: joint distillation warm-started from the union of the no-split sets (CIFAR-10, DM, 200 images). Accuracy in percent, mean \(\pm\) SD over three replicates; paired differences are within replicate; effective rank is the per-class participation ratio of the frozen embedding.}
\label{tab:warm}
\centering
\small
\begin{tabular}{@{}lcccc@{}}
\toprule
Set & Iterations & Accuracy & vs.\ union & effective rank\\
\midrule
union of the independent sets & 0 & 46.44 \(\pm\) 0.50 & -- & 8.91\\
warm start from the union & 500 & 46.16 \(\pm\) 0.51 & \(-\)0.28 \(\pm\) 0.43 & 9.00\\
warm start from the union & 2,000 & 46.88 \(\pm\) 0.67 & \(+\)0.44 \(\pm\) 0.47 & 9.15\\
cold start from real images & 2,000 & 48.27 \(\pm\) 0.54 & \(+\)1.83 \(\pm\) 0.65 & 9.66\\
joint set (official schedule) & 20,000 & 52.59 \(\pm\) 0.36 & \(+\)6.15 \(\pm\) 0.66 & 10.29\\
\bottomrule
\end{tabular}
\end{table}

\subsection{Class-allocation-matched joint set (E11)}
\label{app:balanced}

\paragraph{Configuration.}
Under Dir-0.1 the canonical joint set of E2 places images in every class,
whereas the composed set stores only the classes each source holds, so
the two sets differ in their per-class allocation as well as in how they
were distilled. E11 removes the first difference. On the three Dir-0.1
partitions of CIFAR-10 the joint set was re-distilled with DM at the
official schedule (\(20{,}000\) iterations, differentiable augmentation,
real initialization) with its per-class counts fixed to those of the
composed set (Figure \ref{fig:syn-dm}, right): on partition 0, \(20\) images in the five classes both
sources hold and \(10\) in the five held by one source (\(150\) in all),
on partition 1, \(20\) in seven classes and \(10\) in three (\(170\)),
and on partition 2, where both sources hold all ten classes, \(20\) in
every class (\(200\)), which is the canonical design and serves as a
replicate with a different distillation draw. The composed set, the
matched joint set and the canonical joint set were evaluated with five
evaluation seeds under the \(1{,}000\)-epoch protocol, and the endpoint
quantities with the three initialization seeds of E2.

\paragraph{Analysis.}
Matching the class allocation leaves the excess endpoint error positive in
all nine (partition, seed) cells, unchanged on partition 0 (\(1.55\)
against \(1.53\times10^{-3}\)) and reduced by about a third on
partition 1 (\(1.06\) against \(1.53\times10^{-3}\); Table
\ref{tab:balanced}), while the accuracy gain on these two partitions
falls from \(4.5\) to \(1.0\) and from \(5.4\) to \(3.4\) points. On
the replicate partition both quantities agree with the canonical set to
within \(0.3\) points and \(10^{-5}\). The class allocation therefore
accounts for part of the accuracy gain at Dir-0.1 and can change the size
of the excess endpoint error without removing it, so the two quantities
respond differently to the same intervention (Appendix
\ref{app:nosplit}). The
canonical values reproduce the rows of Table \ref{tab:e2-gap} to
\(10^{-4}\) relative in the endpoint quantities and to within \(0.5\)
points in the five-seed re-evaluation of accuracy.

\begin{table}[h]
\caption{E11: class-allocation-matched joint set at Dir-0.1 on CIFAR-10. The image count is the budget of both joint sets on that partition. Endpoint quantities are RMS over three initialization seeds in units of \(10^{-3}\), accuracy is the mean and SD over five evaluation seeds in percent, and the gain is joint minus independent.}
\label{tab:balanced}
\centering
\footnotesize
\setlength{\tabcolsep}{2.8pt}
\begin{tabular}{@{}lrrrrrrrrrr@{}}
\toprule
& & \multicolumn{2}{c}{Matched joint} & \multicolumn{2}{c}{Canonical joint} & & \multicolumn{2}{c}{Matched} & \multicolumn{2}{c}{Canonical}\\
\cmidrule(lr){3-4}\cmidrule(lr){5-6}\cmidrule(lr){8-9}\cmidrule(lr){10-11}
Partition (images) & \(\mathcal C_{\rm ind}\) & \(\mathcal C_{\rm joint}\) & \(\widehat\Gamma_{\rm RMS}\) & \(\mathcal C_{\rm joint}\) & \(\widehat\Gamma_{\rm RMS}\) & Acc.\ indep. & Acc.\ joint & Gain & Acc.\ joint & Gain\\
\midrule
0 (150) & 9.48 & 7.92 & 1.55 & 7.95 & 1.53 & \(45.4\pm0.7\) & \(46.4\pm0.6\) & \(+1.0\) & \(49.9\pm0.5\) & \(+4.5\)\\
1 (170) & 8.88 & 7.82 & 1.06 & 7.35 & 1.53 & \(46.1\pm1.0\) & \(49.5\pm0.5\) & \(+3.4\) & \(51.4\pm0.4\) & \(+5.4\)\\
2 (200) & 12.18 & 7.29 & 4.90 & 7.29 & 4.90 & \(46.5\pm0.8\) & \(52.5\pm0.8\) & \(+6.0\) & \(52.8\pm0.7\) & \(+6.3\)\\
\bottomrule
\end{tabular}
\end{table}

\subsection{Image budget and the local reference}
\label{app:regime}

\paragraph{The constituent-error regime.}
Corollary \ref{cor:detectability} guarantees the structural gap only when
it exceeds the weighted constituent error, and the independent-versus-joint
comparison lies on the other side of that threshold. Along the official DM
schedule (Figure \ref{fig:trajectory}) the constituent error falls from
\(0.33\) to \(0.15\) within 5{,}000 iterations and then plateaus, the
composed error tracks it with elasticities of \(0.93\)--\(1.08\) over the
four training intervals, and both settle with
\(\mathcal C_{\rm ind}/\mathcal C_{\rm loc}=1.22\pm0.03\) at the end of the
schedule. Additional iterations therefore do not bring ten images per
class closer to the two-step map of 25{,}000 images; the budget sweep
below varies the budget directly.

\begin{figure}[h]
\centering
\includegraphics[width=\linewidth]{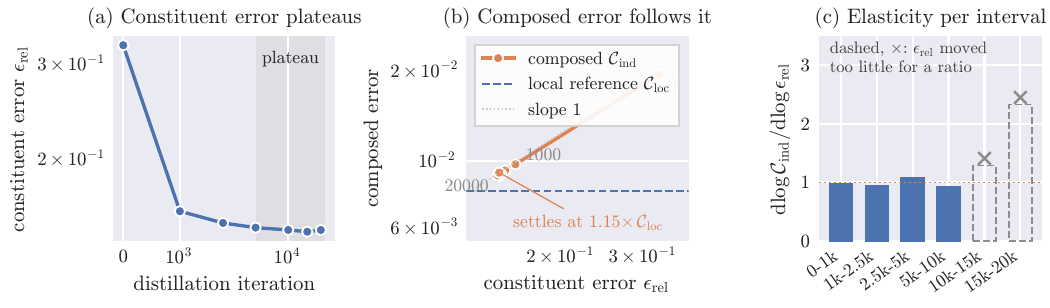}
\caption{Composed error follows constituent error (DM,
Dirichlet-0.1, partition p0). (a) Relative constituent error at the
snapshot checkpoints \{0,1000,2500,5000,10000,15000,20000\} of the
official schedule; it plateaus after 5{,}000 iterations. (b) The composed
error \(\mathcal C_{\rm ind}\) tracks it with unit elasticity and settles
above the local reference \(\mathcal C_{\rm loc}\) (dashed), at
\(1.15\times\) for this partition and about \(1.2\times\) over the five
partitions of the official sweep.}
\label{fig:trajectory}
\end{figure}

\begin{figure}[h]
\centering
\includegraphics[width=0.8\linewidth]{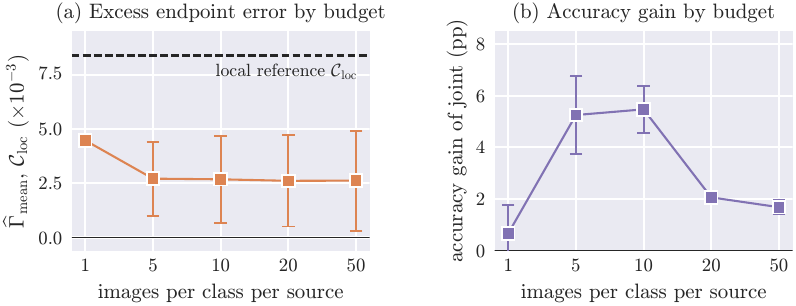}
\caption{Image budget at Dir-0.1 (the same three partitions at every budget). (a) Excess endpoint error \(\widehat\Gamma_{\rm mean}\) of the separately distilled union, aggregated as the mean over seeds of the per-seed endpoint-distance difference as in Table \ref{tab:ipc}, with the local reference \(\mathcal C_{\rm loc}\), which the partitions fix. (b) Accuracy gain of joint over separate-then-union distillation.}
\label{fig:budget}
\end{figure}

\paragraph{Varying the image budget.}
Table \ref{tab:ipc} varies the images per class (IPC) of each source at
Dir-0.1 on the same three partitions (p0--p2) and three initialization
seeds at every budget (Figure \ref{fig:budget}). Three findings follow.
First, the budget moves the constituent error where iterations could not:
\(\epsilon_{\rm rel}\) falls from \(0.42\) at IPC 1 to \(0.15\) at IPC 50,
the composed error follows it with elasticities between \(0.83\) and
\(1.09\), and \(\mathcal C_{\rm ind}/\mathcal C_{\rm loc}\) descends from
\(2.60\) to \(1.02\pm0.08\), while \(\mathcal C_{\rm loc}\), which depends
only on the real partition and the initialization, is identical at every
budget. Second, whether \(\mathcal C_{\rm ind}\) settles at a persistent
plateau near \(\mathcal C_{\rm loc}\) is not resolved. The criteria were
fixed before the two largest budgets were run: a plateau is a ratio
stabilizing at or above one with elasticity falling clearly below one, a
crossing is a partition mean below \(0.95\) with most measurements below
one, and anything between is undetermined. At IPC 50 the partition mean
is \(1.02\pm0.08\), the three partitions give \(0.95\), \(1.02\) and
\(1.11\), four of the nine partition--seed measurements lie below one and
the last elasticity is \(0.95\), so the data fall in the undetermined
region; no general bound requires a plateau, because approximate
constituents can cancel part of the structural term (Appendix
\ref{app:e5}). Third, the excess endpoint error and the accuracy gain
respond differently to the budget: the excess is about \(1.7\times\) its
IPC-10 value at IPC 1 and essentially constant for IPC \(\ge5\), whereas
the gain is near zero at IPC 1, about five points at IPC 5--10 and about
two points at IPC 20--50. Together with the heterogeneity axis, which
changes the excess but not the gain on CIFAR-10, the budget is a second
manipulation on which the two quantities separate; Appendix
\ref{app:coordination} discusses one reading of the inverted-U profile.

\begin{table}[h]
\caption{Image budget at Dir-0.1 (DM, official schedule; partitions
p0--p2 \(\times\) three initialization seeds at every budget).}
\label{tab:ipc}
\centering
\small
\setlength{\tabcolsep}{4pt}
\begin{tabular}{@{}rcccccc@{}}
\toprule
IPC & \(\epsilon_{\rm rel}\) & \(\mathcal C_{\rm ind}\) & \(\mathcal C_{\rm ind}/\mathcal C_{\rm loc}\) & \(\widehat\Gamma_{\rm mean}\) & acc.\ diff.\ (pp) & elasticity to next\\
\midrule
1 & 0.416 \(\pm\) 0.023 & 21.3 & 2.60 \(\pm\) 0.45 & 4.47 \(\pm\) 0.07 & \(+\)0.66 \(\pm\) 1.12 & 0.83\\
5 & 0.206 \(\pm\) 0.021 & 11.9 & 1.44 \(\pm\) 0.10 & 2.71 \(\pm\) 1.70 & \(+\)5.25 \(\pm\) 1.50 & 0.88\\
10 & 0.172 \(\pm\) 0.019 & 10.2 & 1.22 \(\pm\) 0.03 & 2.68 \(\pm\) 2.00 & \(+\)5.47 \(\pm\) 0.90 & 1.09\\
20 & 0.157 \(\pm\) 0.019 & 9.20 & 1.09 \(\pm\) 0.06 & 2.61 \(\pm\) 2.10 & \(+\)2.06 \(\pm\) 0.14 & 0.95\\
50 & 0.147 \(\pm\) 0.017 & 8.63 & 1.02 \(\pm\) 0.08 & 2.62 \(\pm\) 2.30 & \(+\)1.68 \(\pm\) 0.27 & N/A\\
\bottomrule
\end{tabular}
\par\smallskip\raggedright\footnotesize
Each entry is the mean \(\pm\) SD over the three partitions of the
partition-level value (itself the mean over three seeds);
\(\mathcal C_{\rm loc}=8.38\times10^{-3}\) at every IPC.
\(\mathcal C_{\rm ind}\) and \(\widehat\Gamma_{\rm mean}\) in units of
\(10^{-3}\). Elasticity is
\(\mathrm{d}\log\mathcal C_{\rm ind}/\mathrm{d}\log\epsilon_{\rm rel}\)
between consecutive rows, computed from the unrounded partition means
(recomputing from the rounded entries gives 0.83, 0.85, 1.13, 0.97). The joint set's budget is matched to the sum of
the two sources' budgets (IPC 2 to 100 over the union). The IPC-10 row uses
the three sweep partitions and therefore differs slightly from the
five-partition values of Table \ref{tab:e2-gap}.
\end{table}

\subsection{Further datasets (SVHN and CIFAR-100)}
\label{app:svhn}
SVHN \citep{netzer2011svhn} uses the standard training split (73{,}257
images, no \emph{extra} split) with the same tensor conversion and
per-channel normalization as CIFAR-10, five partitions and three
initialization seeds per level at IID and Dir-0.1, and the same DM
schedule, budget and evaluation protocol. CIFAR-100 (E15) uses three
partitions per level, the Dirichlet draw being per class over the 100
classes, and ten images per represented class per source, so the union
stores \(1{,}660\)--\(2{,}000\) images and the joint set the same budget;
the E9 decomposition was applied to every CIFAR-100 cell (Tables
\ref{tab:e9} and \ref{tab:e9-noclock}). Table \ref{tab:svhn} collects the
results. The endpoint layer replicates on both datasets: the IID excess is
small on both (\(-0.04\pm0.04\) on SVHN, \(0.09\pm0.02\) on CIFAR-100,
\(\times10^{-3}\)) and the paired increase from IID to Dir-0.1 is positive
on every partition (five of five on SVHN, three of three on CIFAR-100), and the CIFAR-100 decomposition shows the CIFAR-10 pattern,
with the two terms anti-aligned (cosine \(-0.48\), negative in nine of
nine cells) and the prediction aligned with \(\Delta\Phi\) (cosine
\(0.86\)) but orthogonal to \(d_{\rm ind}\). The accuracy layer depends on
the dataset. On SVHN the gain rises with skew (paired difference
\(+2.24\pm1.21\), five of five, \(t=4.14\)), because the composed set
loses more accuracy under skew (\(70.8\to66.4\), against
\(46.3\to45.8\) on CIFAR-10); on CIFAR-100 it is smaller at Dir-0.1 than
at IID (paired difference \(-1.10\pm1.00\), \(p=0.20\)), a difference
three partitions do not resolve. The SVHN no-split control gives
\(+2.16\pm1.27\) (\(t=2.95\), \(p\approx0.10\)) with
\(\mathcal C_{\rm loc}\) exactly zero, smaller than on CIFAR-10, and
SVHN's frozen-feature effective ranks are lower throughout (\(4.5\)
against \(4.9\) per class, compared with \(8.9\) against \(10.3\) on
CIFAR-10), consistent with the smaller within-class variation of digit
images.

\begin{table}[h]
\caption{Cross-dataset comparison (DM, official schedule, IPC 10
per source; five partitions per split level on CIFAR-10 and SVHN, three on
CIFAR-100, three replicates for the no-split controls).}
\label{tab:svhn}
\centering
\small
\begin{tabular}{@{}llcccc@{}}
\toprule
Dataset & Level & \(n\) & \(\mathcal C_{\rm loc}\) & \(\widehat\Gamma_{\rm mean}\) & accuracy gain (pp)\\
\midrule
CIFAR-10 & IID & 5 & 0.016 & 0.08 \(\pm\) 0.06 & \(+\)5.28 \(\pm\) 0.92\\
CIFAR-10 & Dir-0.1 & 5 & 7.72 & 2.33 \(\pm\) 1.40 & \(+\)5.35 \(\pm\) 0.75\\
SVHN & IID & 5 & 0.013 & \(-\)0.04 \(\pm\) 0.04 & \(+\)3.24 \(\pm\) 1.10\\
SVHN & Dir-0.1 & 5 & 10.5 & 1.47 \(\pm\) 0.95 & \(+\)5.48 \(\pm\) 1.45\\
CIFAR-10 & no split & 3 & 0 (exact) & 0.08 \(\pm\) 0.21 & \(+\)6.55 \(\pm\) 0.42\\
SVHN & no split & 3 & 0 (exact) & 0.09 \(\pm\) 0.16 & \(+\)2.16 \(\pm\) 1.27\\
CIFAR-100 & IID & 3 & 0.018 & 0.09 \(\pm\) 0.02 & \(+\)4.98 \(\pm\) 0.14\\
CIFAR-100 & Dir-0.1 & 3 & 1.35 & 0.56 \(\pm\) 0.03 & \(+\)3.88 \(\pm\) 1.14\\
\midrule
\multicolumn{6}{@{}l}{Paired difference Dir-0.1 \(-\) IID within partition (\(t\) on 4 degrees of freedom, 2 for CIFAR-100)}\\
CIFAR-10 & & 5 & & \(+\)2.25 \(\pm\) 1.39 (\(t=3.62\)) & \(+\)0.06 \(\pm\) 1.23 (\(t=0.12\))\\
SVHN & & 5 & & \(+\)1.51 \(\pm\) 0.94 (\(t=3.58\)) & \(+\)2.24 \(\pm\) 1.21 (\(t=4.14\))\\
CIFAR-100 & & 3 & & \(+\)0.47 \(\pm\) 0.02 (\(t=37\)) & \(-\)1.10 \(\pm\) 1.00 (\(t=-1.91\))\\
\bottomrule
\end{tabular}
\par\smallskip\raggedright\footnotesize
\(\mathcal C_{\rm loc}\) and \(\widehat\Gamma_{\rm mean}\) in units of
\(10^{-3}\); mean \(\pm\) SD over partitions, each partition averaging its
initialization seeds (five seeds for the CIFAR-10 excess error here, three
for SVHN and CIFAR-100). Two-sided \(p\): \(0.022\), \(0.91\) (CIFAR-10),
\(0.023\), \(0.014\) (SVHN) and \(0.001\), \(0.20\) (CIFAR-100).
\end{table}

\subsection{Cross-architecture evaluation (E14)}
\label{app:crossarch}

\paragraph{Configuration.}
The saved E2 sets (four levels, five partitions), the E7 no-split sets and
the E11 class-allocation-matched sets were re-evaluated on ResNet-18 (the
instance-normalized variant of the dataset-condensation code) with the
DM evaluation protocol otherwise unchanged (1000 epochs, learning rate
\(0.01\), batch 256, differentiable augmentation, three evaluation seeds).
No set was re-distilled.

\paragraph{Analysis.}
Table \ref{tab:crossarch} compares the two evaluation networks. On
ResNet-18 the joint set remains more accurate than the union on every
CIFAR-10 partition, but the gain is smaller than on the ConvNet
(\(+1.8\) to \(+3.4\) against \(+5.3\) to \(+5.9\) points) and its
dependence on skew differs. With the paired design of Appendix
\ref{app:nosplit} (the same partition seeds at every level), the
within-partition difference of the ResNet-18 gain between Dir-0.1 and IID
is \(+1.64\pm1.39\) points, positive on five of five partitions
(\(t=2.63\) on four degrees of freedom, \(p=0.06\)); between Dir-0.5 and
IID it is \(+1.50\pm1.50\) (five of five, \(p=0.09\)) and between
Dir-1.0 and IID \(+0.24\pm0.74\) (three of five). The corresponding
ConvNet differences are \(+0.06\pm1.23\), \(+0.58\pm1.67\) and
\(+0.44\pm0.94\). The no-split gain on ResNet-18 is \(+1.13\pm0.79\)
(positive in all three replicates, \(p=0.13\) on two degrees of freedom),
and the class-allocation-matched joint set gains \(+1.37\) against
\(+2.81\) for the canonical joint set on the same three Dir-0.1
partitions. On SVHN the ResNet-18 gain is \(+0.39\pm1.89\) at IID and
\(+4.12\pm2.01\) at Dir-0.1, a paired difference of \(+3.72\pm3.41\)
(four of five partitions, \(p=0.07\)). Source heterogeneity is therefore not
necessary for the gain on either network; its size depends on the
evaluation network, and on ResNet-18 the gain at the highest skew exceeds
the IID gain on every partition, a pattern that the ConvNet shows on SVHN
but not on CIFAR-10.

\begin{table}[h]
\caption{E14: accuracy of the same synthetic sets on the ConvNet and on ResNet-18 (DM, official evaluation protocol, three evaluation seeds). Percent, mean \(\pm\) SD over partitions; gain is joint minus union within partition.}
\label{tab:crossarch}
\centering
\small
\setlength{\tabcolsep}{4pt}
\begin{tabular}{@{}llcccccc@{}}
\toprule
 & & & \multicolumn{2}{c}{ConvNet} & \multicolumn{3}{c}{ResNet-18}\\
\cmidrule(lr){4-5}\cmidrule(lr){6-8}
Dataset & Level & \(n\) & joint & gain & union & joint & gain\\
\midrule
CIFAR-10 & IID & 5 & 51.6 \(\pm\) 0.6 & \(+\)5.28 \(\pm\) 0.92 & 40.3 \(\pm\) 0.6 & 42.1 \(\pm\) 0.5 & \(+\)1.81 \(\pm\) 0.85\\
CIFAR-10 & Dir-1.0 & 5 & 51.9 \(\pm\) 0.5 & \(+\)5.72 \(\pm\) 0.64 & 40.2 \(\pm\) 0.6 & 42.3 \(\pm\) 0.6 & \(+\)2.05 \(\pm\) 0.77\\
CIFAR-10 & Dir-0.5 & 5 & 51.7 \(\pm\) 0.5 & \(+\)5.86 \(\pm\) 1.24 & 39.2 \(\pm\) 1.1 & 42.6 \(\pm\) 0.4 & \(+\)3.31 \(\pm\) 0.84\\
CIFAR-10 & Dir-0.1 & 5 & 51.1 \(\pm\) 1.1 & \(+\)5.35 \(\pm\) 0.75 & 37.8 \(\pm\) 1.1 & 41.2 \(\pm\) 0.8 & \(+\)3.44 \(\pm\) 0.95\\
CIFAR-10 & no split & 3 & 52.6 \(\pm\) 0.4 & \(+\)6.55 \(\pm\) 0.42 & 41.1 \(\pm\) 0.8 & 42.2 \(\pm\) 0.1 & \(+\)1.13 \(\pm\) 0.79\\
SVHN & IID & 5 & 73.7 \(\pm\) 0.9 & \(+\)3.24 \(\pm\) 1.10 & 67.0 \(\pm\) 1.2 & 67.4 \(\pm\) 1.1 & \(+\)0.39 \(\pm\) 1.89\\
SVHN & Dir-0.1 & 5 & 72.8 \(\pm\) 0.7 & \(+\)5.48 \(\pm\) 1.45 & 61.9 \(\pm\) 2.4 & 66.0 \(\pm\) 0.7 & \(+\)4.12 \(\pm\) 2.01\\
\bottomrule
\end{tabular}
\par\smallskip\raggedright\footnotesize
ConvNet accuracies are those of Tables \ref{tab:e2-gap} and \ref{tab:svhn} (three evaluation seeds); ResNet-18 accuracies are from this experiment (three evaluation seeds). The one-sample \(t\) of the ResNet-18 gain against zero is \(4.8\), \(6.0\), \(8.9\), \(8.1\) at the four CIFAR-10 levels and \(0.5\), \(4.6\) on SVHN.
\end{table}

\subsection{Feature-space diagnostics of independently and jointly distilled sets (E8)}
\label{app:coordination}

For each synthetic set (Table \ref{tab:coordination}) we embed every
image with a frozen, randomly initialized ConvNet and compute, per class,
the participation ratio of the Gram spectrum of the embedded images, that
is, the number of directions the class's synthetic images effectively
occupy, averaged over classes. For the two independently distilled sets we
also compute the ratio of the mean cross-set distance to the mean
within-set distance. A ratio below one is consistent with overlap between
the two sets, and a ratio above one with the sets occupying different
regions. These are descriptive statistics of the images and are unrelated
to \(r_H\), which is a property of the real sources' Hessians.

\begin{table}[h]
\caption{Effective rank per class and cross/within distance ratio in a frozen feature space. Mean \(\pm\) SD over partitions (\(n\) as listed; two IID partitions, five Dir-0.1 partitions, three no-split replicates); image budgets matched within partition. Dir-1.0 and Dir-0.5 were not measured.}
\label{tab:coordination}
\centering
\small
\begin{tabular}{@{}llccc@{}}
\toprule
Level & \(n\) & eff.\ rank, separate (unioned) & eff.\ rank, joint & cross/within\\
\midrule
IID & 2 & 8.91 \(\pm\) 0.11 & 10.27 \(\pm\) 0.02 & 0.942 \(\pm\) 0.000\\
Dirichlet-0.1 & 5 & 7.39 \(\pm\) 0.42 & 9.79 \(\pm\) 0.40 & 1.288 \(\pm\) 0.097\\
No split (E7) & 3 & 8.91 \(\pm\) 0.05 & 10.29 \(\pm\) 0.11 & 0.941 \(\pm\) 0.001\\
\bottomrule
\end{tabular}
\end{table}

Under IID the two independent sets have a cross/within ratio below one,
consistent with overlapping coverage in the frozen feature space; under
label skew the ratio is above one, consistent with each set covering its
own skewed subset. In both cases the joint set occupies more directions
per class at the same budget, a statement about the chosen embedding
rather than about coverage of the real data. The pattern accompanies the
accuracy gain and is present when the curvature-variance term is
identically zero, as in the no-split control. Class omission under
Dir-0.1 cannot account for the gain under IID and with no split, where
every class is present in both sources at identical budgets; its
contribution at Dir-0.1 is isolated by the class-allocation-matched
control of Appendix \ref{app:balanced}. One reading of the inverted-U
budget profile of Appendix \ref{app:regime} is that redundancy between
the two sets matters most at intermediate budgets, where each set is
large enough to overlap the other and small enough for the overlap to
cost accuracy; the statistics here are consistent with that reading but
do not measure class coverage directly.

\subsection{Assumption diagnostics (E3)}
\label{app:assumption-support}

\begin{table}[!htb]
\caption{Assumption diagnostics.}
\label{tab:assumption-support}
\centering
\small
\begin{tabular}{@{}llrrr@{}}
\toprule
Network & Level & \(\eta L\) & \(r_H\) & \(\delta_H\)\\
\midrule
\multicolumn{5}{@{}l}{\emph{Local-law network at \(\eta_{\rm ref}=7.8\times10^{-4}\)}}\\
ReLU, 2{,}000-sample subset & IID / Dir-1.0 / Dir-0.1 & 0.112 / 0.116 / 0.126 & --- & 0.29 \(\pm\) 0.08\\
ReLU, \(10^4\) subset & IID / Dir-0.1 & 0.096 / 0.108 & --- & ---\\
SiLU, \(10^4\) subset & IID / Dir-0.1 & 0.047 / 0.066 & --- & 0.14 \(\pm\) 0.02\\
tanh, \(10^4\) subset & IID / Dir-0.1 & 0.039 / 0.060 & --- & 0.20 \(\pm\) 0.02\\
\addlinespace
\multicolumn{5}{@{}l}{\emph{Distillation network (ReLU) at \(\eta=6.25\times10^{-3}\)}}\\
& IID & 0.80 \(\pm\) 0.00 & N/A & ---\\
& Dir-1.0 & 0.85 \(\pm\) 0.02 & 204.0 \(\pm\) 8.4 & ---\\
& Dir-0.5 & 0.86 \(\pm\) 0.02 & 201.0 \(\pm\) 4.0 & ---\\
& Dir-0.1 & 1.03 \(\pm\) 0.06 & 208.5 \(\pm\) 5.2 & 0.28 \(\pm\) 0.08\\
\bottomrule
\end{tabular}
\par\smallskip\raggedright\footnotesize
\(L=\max_i\|H_i\|_2\) at the shared initialization; \(r_H\) is the
participation ratio \(\operatorname{tr}(V_H)^2/\operatorname{tr}(V_H^2)\)
from Hutchinson probes \citep{hutchinson1989stochastic}, out of
\(320{,}010\) parameters; \(\delta_H=\|H(\theta_2)-H(\theta_0)\|_2/\|H(\theta_0)\|_2\)
along the two-step real-union trajectory at Dir-0.1, mean \(\pm\) SD over
five partitions. Local-law rows are partition means; distillation rows
are mean \(\pm\) SD over five partitions (three seeds at Dir-0.1, one
otherwise). Dashes: not measured. N/A: under IID \(V_H\) is at numerical
zero and its participation ratio is not meaningful. \(L\) differs by
activation (ReLU \(123\)--\(139\), SiLU \(60\)--\(85\), tanh \(50\)--\(76\)),
so the activations sit at different \(\eta L\) at a given \(\eta\).
\end{table}

The main-text assumptions are stated in Section \ref{sec:commutation}.
Table \ref{tab:assumption-support} reports the diagnostics available for
them: the normalized step size \(\eta L\) (which enters the positive-branch
condition and the Theorem~\ref{thm:defect} bracket) for the local-law
network at its calibration rate and for the distillation network at its
operating rate, and the effective rank of \(V_H\) (the \(\sqrt{r_H}\)
factor in the same bracket) for the distillation network. It also reports
the relative Hessian change \(\delta_H\) along the two-step trajectory,
the empirical proxy for the slowly-varying-Hessian condition of
\eqref{eq:nonlinear-remainder}. Two of the three diagnostics sit at their
boundaries. \(\eta L\approx1\) at the distillation rate, and \(\delta_H\)
lies between \(0.14\) (SiLU) and \(0.29\) (ReLU), meaning that the Hessian
moves by 14--29\% within two steps. ReLU has the largest and most variable
\(\delta_H\) although its activation has zero second derivative almost
everywhere. The loss Hessian still varies through the products of layer
parameters, the normalization layers, and the cross-entropy loss, and it
also changes non-smoothly when the trajectory crosses activation
boundaries, so the measured \(\delta_H\) should not be attributed to the
crossings alone. The two ReLU constructions (local-law and official
networks) agree (\(0.29\) vs.\ \(0.28\)), so the activation comparison
transfers to the network used for distillation. Only the low effective
rank holds with a wide margin, consistent with the concentrated Hessian
spectra reported for trained and initialized networks
\citep{sagun2017empirical,ghorbani2019investigation}. The exponent measured in the local-law
validation, \(3\)--\(6\%\) below four (Appendix \ref{app:e1-detail}), is
consistent with both boundary conditions. \(r_H\) on the local-law network
was not measured.

\subsection{Earlier short-schedule and preliminary diagnostics}
\label{app:legacy}

The results in this subsection were obtained with the earlier DM schedule
(500 iterations, 50-epoch evaluation) or are single-condition
preliminaries. They support the main results but do not carry them; the
structural quantities \(P_{\rm comp}\) and \(\mathcal C_{\rm loc}\) in
them depend only on the real partitions and are unchanged by the schedule.

\paragraph{Type of heterogeneity: the feature-skew family.}
A label-balanced feature-skew family complements the label-skewed
partitions. A frozen, untrained ConvNet embeds every image, and within
each class the examples are sorted by a projection whose direction is
derived from the partition seed and split at the median, so each source
receives half of every class (\(\alpha_1=\alpha_2=1/2\), label total
variation zero) and only the within-class inputs differ. Random
directions give eight partitions; a single PC1 split serves as a
higher-contrast diagnostic and is not treated as an independent sample.
The structural term of this family is \(60\times\) below Dirichlet-1.0 in
\(P_{\rm comp}\) (about \(8\times\) in RMS), and under the short schedule
its excess endpoint error is significantly negative (Table
\ref{tab:e2-feature}), so it is a second level at which the structural
signal is small relative to the observed residuals.

\begin{table}[h]
\caption{Type of heterogeneity changes the empirical outcome.}
\label{tab:e2-feature}
\centering
\small
\setlength{\tabcolsep}{4pt}
\begin{tabular}{@{}lrrrrrrr@{}}
\toprule
Shift & label TV & \(n\) & \(P_{\rm comp}\) & \(\sqrt{P_{\rm comp}}\) & \(\mathcal C_{\rm loc}\) & \(\mathcal C_{\rm loc}/\sqrt{P_{\rm comp}}\) & \(\widehat\Gamma_{\rm mean}\)\(^{\dagger}\)\\
\midrule
IID & 0.000 & 5 & \(4.13\times10^{-12}\) & 0.002 & 0.016 & (8.1) & \(-\)0.377\\
Feature skew\(^{\rm f}\) & 0.000 & 8 & \(3.95\times10^{-7}\) & 0.63 & 0.652 & (1.04) & \(-\)0.478\\
Dirichlet-1.0 & 0.227 & 5 & \(2.30\times10^{-5}\) & 4.61 & 4.54 & 1.01 & \(+\)0.928\\
Dirichlet-0.5 & 0.300 & 5 & \(5.25\times10^{-5}\) & 7.12 & 5.48 & 0.78 & not run\\
Dirichlet-0.1 & 0.392 & 5 & \(1.42\times10^{-4}\) & 11.7 & 7.72 & 0.67 & \(+\)1.96\\
\bottomrule
\end{tabular}
\par\smallskip\raggedright\footnotesize
\(P_{\rm comp}\) and \(\mathcal C_{\rm loc}\) from one protocol (E2
network, \(\eta=6.25\times10^{-3}\), full sources, five partitions
\(\times\) five seeds); \(\sqrt{P_{\rm comp}}\), \(\mathcal C_{\rm loc}\),
and \(\widehat\Gamma_{\rm mean}\) in units of \(10^{-3}\). Each column is
the mean of the per-measurement values over partitions and seeds, so the
ratio column is the mean of the per-measurement ratios and is not the
quotient of the two column means, and \(\sqrt{P_{\rm comp}}\) is not the
square root of the \(P_{\rm comp}\) column. In the quadratic model
\(\mathcal C_{\rm loc}=\mathcal C_{\rm struct}=\sqrt{P_{\rm comp}}\). On
the network, the ratio \(\mathcal C_{\rm loc}/\sqrt{P_{\rm comp}}\) is the
calibration of the quadratic prediction against the measured local
reference at the E2 rate, where \(\eta L\approx1\). It is \(1.01\) at
Dir-1.0 and drifts to \(0.67\) at Dir-0.1, the same downward drift with
step size that the local-law validation shows above \(\eta_{\rm ref}\)
(Figure \ref{fig:e1-full}b). At IID both quantities are at numerical
zero, and the ratio (in parentheses) is not meaningful. The same \(P_{\rm comp}\) values measured at \(\eta_{\rm ref}\)
and rescaled by \((\eta/\eta_{\rm ref})^4\) agree with these to within
2.4\%. \(^{\rm f}\)Feature skew is from the earlier short-schedule run
(\(P_{\rm comp}\) at the E2 rate, eight partitions), so its ratio, in
parentheses, is not directly comparable with the others. \(^{\dagger}\)\(\widehat\Gamma_{\rm
mean}\) is from the earlier 500-iteration schedule (Table \ref{tab:e2-gap}
reports \(\widehat\Gamma_{\rm RMS}\) at the official schedule); a negative value means the composed set's map
was closer to the real-union map than the joint set's.
\end{table}

For random-projection feature skew, \(P_{\rm comp}\) is positive in all
eight partitions, with mean \(3.95\times10^{-7}\) at the E2 rate and 95\%
CI \([2.78,5.16]\times10^{-7}\). Nevertheless,
\(\widehat\Gamma_{\rm mean}=-4.78\times10^{-4}\) with CI
\([-5.92,-3.65]\times10^{-4}\), and all eight gaps are negative. The PC1
diagnostic increases \(P_{\rm comp}\) to \(4.70\times10^{-6}\) and
\(\mathcal C_{\rm loc}\) to \(2.14\times10^{-3}\), but its single measured
gap remains negative at \(-6.59\times10^{-4}\). This single-partition
probe argues against an explanation by lack of signal. In this family
the observed composed error is \(13.4\) times the local
reference and its cosine with the quadratic prediction is \(0.011\), so
the measured endpoint is dominated by a residual nearly orthogonal to the
structural term.

\paragraph{Semantic-mass intervention.}

Figure \ref{fig:mass-intervention} compares stored semantic mass with
uniform synthetic class weights on the original three partitions per
heterogeneous level, with clock calibration in all cells (IID is exactly
one by construction). Improved constituent fidelity does not necessarily
transfer to the union.

\begin{figure}[h]
\centering
\includegraphics[width=0.85\linewidth]{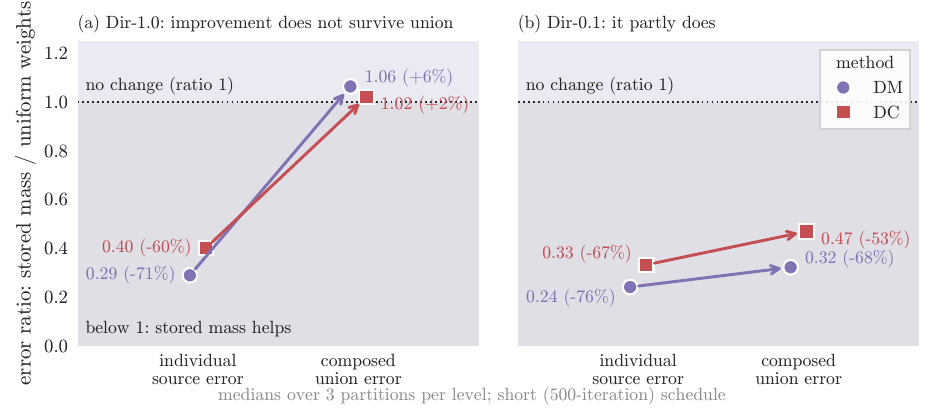}
\caption{Better constituents do not always yield a better union.
Ratio of the error with stored semantic mass to the error with uniform
weights, for the individual sources and for their composed union
(percentages: change relative to uniform weights). (a) At Dir-1.0 the
individual error falls by 60--71\% while the composed error is unchanged.
(b) At Dir-0.1 the improvement partly survives union.}
\label{fig:mass-intervention}
\end{figure}

\paragraph{Five sources.}
\label{app:m5}
One condition with \(m=5\) sources was run at Dir-0.1 with the same DM
schedule and IPC 10 per source, on five partitions with three
initialization seeds each; the joint set received the sum of the five
budgets (Table \ref{tab:m5}). The local reference, the quadratic
prediction, the excess endpoint error and the accuracy gain all grow from
two to five sources, so the mechanism is not restricted to two
constituents, although \(V_H\) is not monotone in the number of sources in
general and this is an observation rather than a prediction. With IPC
fixed per source the composed set's image count roughly doubles, so
\(\mathcal C_{\rm ind}\) and its ratio to \(\mathcal C_{\rm loc}\)
confound source count with budget and are not interpreted.

\begin{table}[h]
\caption{Two versus five sources at Dir-0.1 (DM, official schedule, IPC 10 per source). Mean \(\pm\) SD over \(n=5\) partitions; within a partition the endpoint quantities are RMS over the three initialization seeds 100--102 and the accuracy gain is their mean, the convention of Table \ref{tab:e2-gap}, whose Dir-0.1 row the \(m=2\) column reproduces. The first block depends only on the real partition and the network; the second describes the distilled sets, whose total budget is not matched between \(m=2\) and \(m=5\), so \(\mathcal C_{\rm ind}\) and its ratio to \(\mathcal C_{\rm loc}\) confound source count with budget and are not interpreted.}
\label{tab:m5}
\centering
\small
\begin{tabular}{@{}lccc@{}}
\toprule
 & \(m=2\) & \(m=5\) & ratio\\
\midrule
\multicolumn{4}{@{}l}{\emph{Real partition and network only}}\\
\(\mathcal C_{\rm loc}\) (\(\times10^{-3}\)) & 8.09 \(\pm\) 1.17 & 14.1 \(\pm\) 2.3 & 1.75\\
\(\sqrt{P_{\rm comp}}\) (\(\times10^{-3}\)) & 12.3 \(\pm\) 2.1 & 25.8 \(\pm\) 2.2 & 2.09\\
\addlinespace
\multicolumn{4}{@{}l}{\emph{Distilled sets (IPC 10 per source; total budget grows with \(m\))}}\\
\(\widehat\Gamma_{\rm RMS}\) (\(\times10^{-3}\)) & 2.35 \(\pm\) 1.44 & 5.65 \(\pm\) 1.56 & 2.41\\
accuracy gain (pp) & \(+\)5.35 \(\pm\) 0.75 & \(+\)7.34 \(\pm\) 1.34 & 1.37\\
composed-set images & 176 \(\pm\) 19 & 376 \(\pm\) 32 & 2.14\\
\(\mathcal C_{\rm ind}\) (\(\times10^{-3}\)) & 9.85 \(\pm\) 1.33 & 12.2 \(\pm\) 1.5 & 1.24\\
\bottomrule
\end{tabular}
\end{table}

\FloatBarrier
\section{Limitations and Future Work}
\label{app:limitations}

\paragraph{Scope of the theory.}
The exact results concern full-batch, fixed-step gradient descent on
quadratic objectives (A1). They apply to a neural network only locally,
through a frozen quadratic model that requires a slowly varying Hessian
along the training trajectory (A4); Appendix \ref{app:assumption-support}
reports the measured Hessian variation. The remainder constant \(\nu\)
cannot be estimated reliably for a deep network, so the calibration of
\(P_{\rm comp}\) is checked empirically rather than certified. The
arbitrary-ratio expansion (Theorem \ref{thm:defect}) additionally assumes
a common optimum and \(\eta L<1\), whereas the distillation experiments
run at \(\eta L\approx1\) and rely only on the exact two-to-one identity.
For two-to-one compression, Lemma \ref{lem:optimizers} shows that
heavy-ball momentum and independent minibatches (in expectation) leave the
discrepancy \(\eta^2(V_H\theta_0-w)\) unchanged; general ratios under
momentum, adaptive preconditioners such as Adam, whose effective step
depends on the gradient history, and long horizons, over which the
curvature is path dependent, lie outside the model.

\paragraph{Scope of the evidence.}
The main evidence is the official-schedule comparison on two-source
CIFAR-10 partitions with its controls and the SVHN and CIFAR-100
replications (Appendices \ref{app:e2}--\ref{app:svhn}); the short-schedule
runs, the feature-skew family and the five-source condition are
preliminary (Appendix \ref{app:legacy}). One distillation method was used
at the official schedule, and the accuracy gain was evaluated on two
networks (Appendix \ref{app:crossarch}). Several controls rest on few
units: three SVHN no-split replicates, three CIFAR-100 partitions per
level, two partitions in the class-allocation-matched control and in the
IID feature-space diagnostics.

\paragraph{Open questions.}
Three questions remain open. First, how the theory extends to longer
horizons and other optimizers, where the linear superposition of source
statistics and the frozen-quadratic approximation both fail. Second,
whether a distillation method can bring the source error on a network
into the regime of the law: the inequality
\(\mathcal C_{\rm struct}>\mathcal U_{\rm src}\) of Corollary
\ref{cor:detectability} was not satisfied at the reference settings with
either method (Appendices \ref{app:e10} and \ref{app:e9}), the few
initializations that satisfy it at \(16\eta_{\rm ref}\) do so where the
local prediction itself degrades, and the budget sweep does not settle
whether the composed error stabilizes near \(\mathcal C_{\rm loc}\)
(Appendix \ref{app:regime}). Third, what produces the accuracy advantage
of joint distillation: its dependence on heterogeneity differs between
datasets and networks (Appendices \ref{app:svhn} and
\ref{app:crossarch}), warm-starting joint distillation from the union recovers only a small
part of it and class reweighting gives no detectable accuracy change
(Appendices \ref{app:warm} and \ref{app:reweight}), and the
feature-space diagnostics of Appendix \ref{app:coordination} are
consistent with a coordination deficit between sets optimized without
knowledge of each other without identifying the cause.

\paragraph{Future work.}
Four directions follow: the endpoint-level dissociation at ImageNet
scale; a composition-aware objective that adds the penalty
\(\mathcal R_{\rm comp}=\E_{\theta_0}\|\TrainMap^{K,\eta}_{\oplus_i\alpha_iS_i}(\theta_0)
-\TrainMap^{T,\eta}_{\oplus_i\alpha_iD_i}(\theta_0)\|_2^2\) to source-wise
fidelity through a surrogate in the synthetic samples; a budget-matched
study of the number of sources (Appendix \ref{app:m5}); and a coordination
mechanism that does not pool data.

\end{document}